\documentclass{article}

\usepackage{mlmath}
\usepackage{PRIMEarxiv}

\usepackage[utf8]{inputenc} 
\usepackage[T1]{fontenc}    
\usepackage{hyperref}       
\usepackage{url}            
\usepackage{booktabs,multirow}       
\usepackage{amsfonts}       
\usepackage{nicefrac}       
\usepackage{microtype}      
\usepackage{xcolor}         

\usepackage{times}
\usepackage{graphicx}
\usepackage[caption=false,font=normalsize,labelfont=sf,textfont=sf]{subfig}
\usepackage{float}
\usepackage{caption}
\usepackage{amsmath,amsfonts,amssymb,mathtools,amsthm}
\usepackage{makecell}

\usepackage[T1]{fontenc}

\newtheorem{theorem}{Theorem}
\newtheorem*{theorem*}{Theorem}
\newtheorem{lemma}{Lemma}
\newtheorem*{lemma*}{Lemma}

\newtheorem*{property*}{Property}

\newtheorem{definition}{Definition}

\newtheorem*{assumption*}{Assumption}
\newtheorem*{prop*}{Proposition}
\newtheorem{setting}{Setting}
\newtheorem*{setting*}{Setting}

\title{Implicit Regularization of Dropout}
\begin{document}

\author{
Zhongwang Zhang\textsuperscript{\rm 1},  
Zhi-Qin John Xu\textsuperscript{\rm 1,2}\thanks{Corresponding author: xuzhiqin@sjtu.edu.cn.}
\\
\textsuperscript{\rm 1}  School of Mathematical Sciences, Institute of Natural Sciences, MOE-LSC, Shanghai Jiao Tong University \\ 
\textsuperscript{\rm 2} Qing Yuan Research Institute, Shanghai Jiao Tong University \\
}

\maketitle

\begin{abstract}
It is important to understand how dropout, a popular regularization method, aids in achieving a good generalization solution during neural network training. In this work, we present a theoretical derivation of an implicit regularization of dropout, which is validated by a series of experiments. Additionally, we numerically study two implications of the implicit regularization, which intuitively rationalizes why dropout helps generalization. Firstly, we find that input weights of hidden neurons tend to condense on isolated orientations trained with dropout. Condensation is a feature in the non-linear learning process, which makes the network less complex. Secondly, we experimentally find that the training with dropout leads to the neural network with a flatter minimum compared with standard gradient descent training, and the implicit regularization is the key to finding flat solutions. Although our theory mainly focuses on dropout used in the last hidden layer, our experiments apply to general dropout in training neural networks. This work points out a distinct characteristic  of dropout compared with stochastic gradient descent and serves as an important basis for fully understanding dropout.
\end{abstract}

\section{Introduction}\label{sec..intro}
Dropout is used with gradient-descent-based algorithms for training neural networks (NNs) \cite{hinton2012improving,srivastava2014dropout}, which can improve the generalization in deep learning \cite{tan2019efficientnet,helmbold2015inductive}. For example, common neural network frameworks such as PyTorch default to utilizing dropout during transformer training.
Dropout works by multiplying the output of each neuron by a random variable with probability $p$ being $1/p$ and $1-p$ being zero during training.
Note that every time the concerning quantity is calculated, the variable is randomly sampled at each feedforward operation. 

The effect of dropout is equivalent to adding a specific noise to the gradient descent training. Theoretically, based on the method of the modified gradient flow \cite{hairer2003geometric}, we derive implicit regularization terms of the dropout training for networks with dropout on the last hidden layer. The implicit regularization of dropout can lead to two important implications, condensed weights and flat solutions, verified by a series of experiments\footnote{Code can be found at: \url{https://github.com/sjtuzzw/torch_code_frame}} under general settings.

Firstly, we study weight feature learning in dropout training. Previous works \cite{luo2021phase,zhou2021towards,zhou2022empirical} find that, in the nonlinear training regime, input weights of hidden neurons (the input weight of a hidden neuron is a vector consisting of the weight from its input layer to the hidden layer and its bias term) are clustered into several groups under gradient flow training. The weights in each group have similar orientations, which is called \emph{condensation}. By analyzing the implicit regularization terms, we theoretically find that dropout tends to find solutions with weight condensation. To verify the effect of dropout on condensation, we conduct experiments in the linear regime, such as neural tangent kernel initialization \cite{jacot2018neural}, where the weights are in proximity to the random initial values and condensation does not occur in common gradient descent training. We find that even in the linear regime, with dropout, weights show clear condensation in experiments, and for simplicity, we only show the output here (Fig. \ref{pic:condense_toy}(a)). As condensation reduces the complexity of the NN, dropout may help the generalization by constraining the model's complexity.

Secondly, we study the flatness of the solution in dropout training. We theoretically
show that the implicit regularization terms of dropout lead to a flat minimum. We experimentally verify the effect of the implicit regularization terms on flatness (Fig. \ref{pic:condense_toy}(b)). As suggested by many existing works \cite{keskar2016large,neyshabur2017exploring,zhu2018anisotropic}, flatter minima have a higher probability of better generalization and stability.

\begin{figure}[h]
	\centering
 	\subfloat[condensation]{\includegraphics[width=0.4\textwidth]{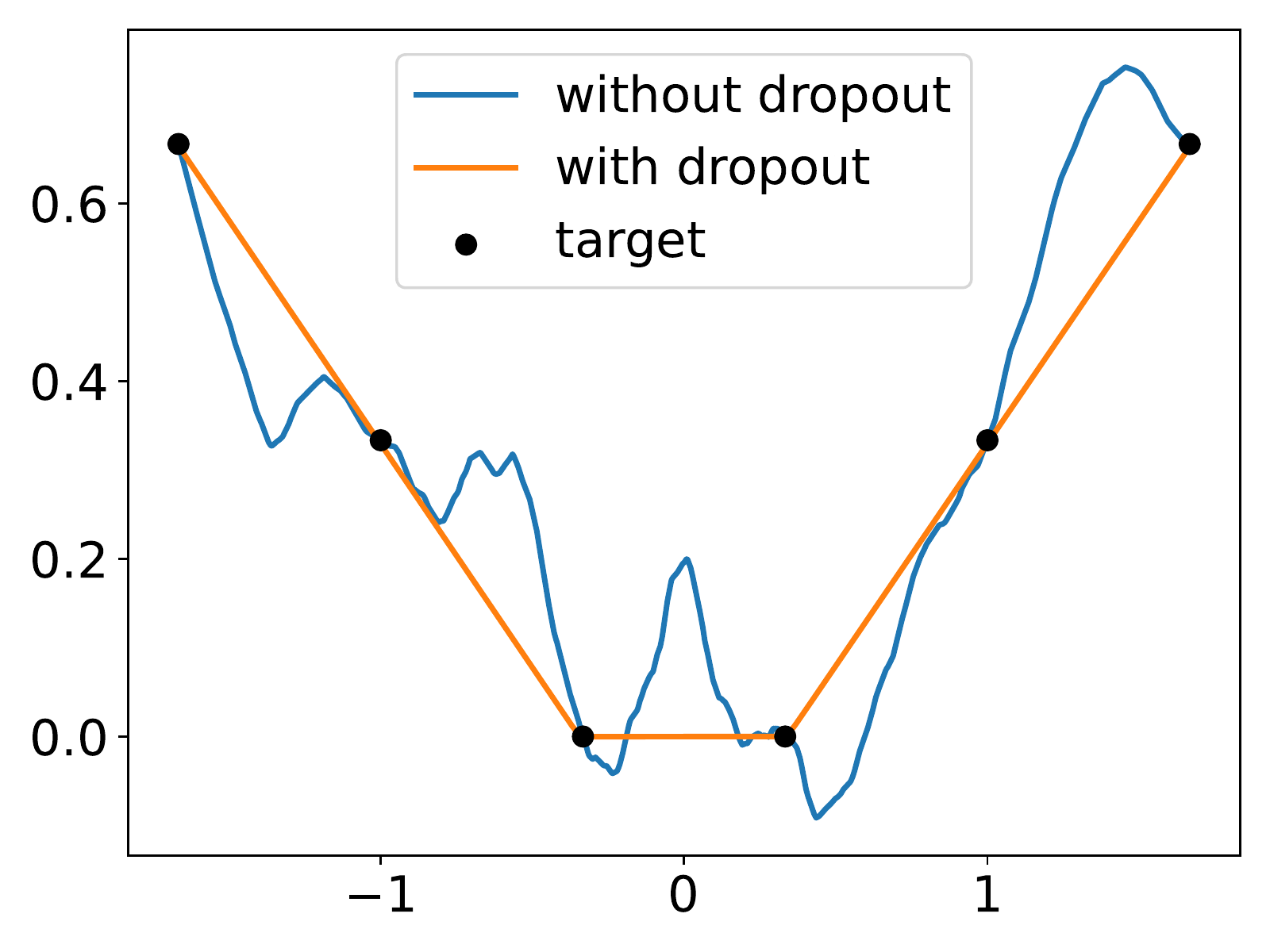}}	
  \subfloat[flatness]{\includegraphics[width=0.4\textwidth]{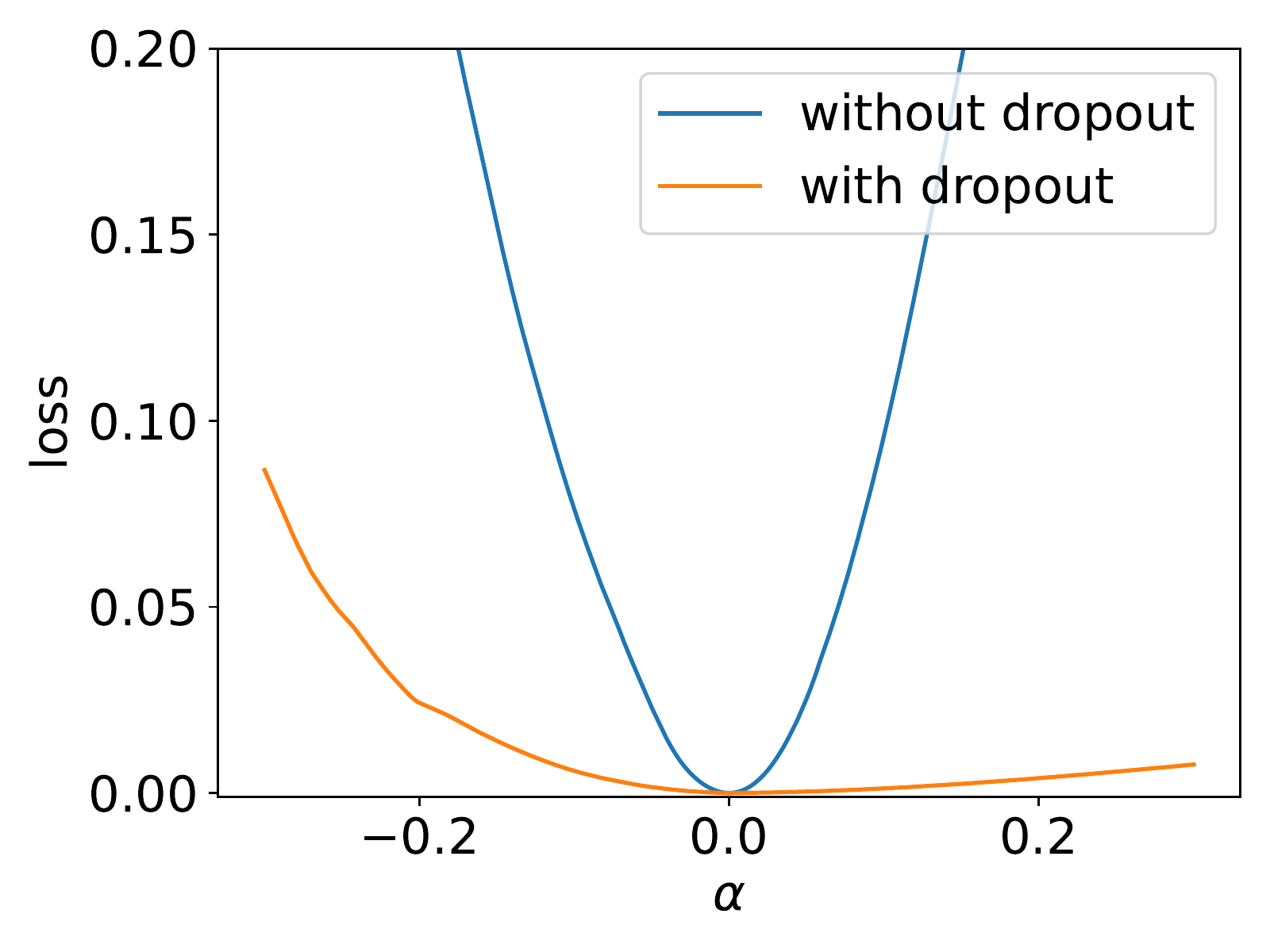}}	
	
  \caption{The experimental results of training two-layer ReLU NNs trained with and without dropout. The width of the hidden layers is $1000$, and the learning rate for all experiments is $1\times10^{-3}$. (a) The output of NNs with or without dropout. The black points represent the target points. (b) The loss value obtained by perturbing the network with or without dropout in a given random direction. $\alpha$ is the step size moving in the above direction.\label{pic:condense_toy}}
\end{figure} 

This work provides a comprehensive investigation into the implicit regularization of dropout and its associated implications. Although our theoretical analysis mainly focuses on the dropout used in the last hidden layer, our experimental results extend to the general use of dropout in training NNs. Our results show that dropout has a distinct implicit regularization for facilitating weight condensation and finding flat minima, which may jointly improve the generalization performance of NNs. 

\section{Related Works}
Dropout is proposed as a simple approach to prevent overfitting in the training of NNs, thus improving the generalization of the network \cite{hinton2012improving,srivastava2014dropout}. Many works aim to find an explicit form of dropout regularization. 
A previous work \cite{mcallester2013pac} presents PAC-Bayesian bounds, and others \cite{wan2013regularization, mou2018dropout} derive Rademacher generalization bounds. These results show that the reduction of complexity brought by dropout is $O(p)$, where $p$ is the probability of keeping an element in dropout. 
All of the above works need specific settings, such as norm assumptions and logistic loss, and they only give a rough estimate of the generalization error bound, which usually consider the worst case. \cite{wager2013dropout, mianjy2018implicit} study the implicit bias of dropout for linear models. However, it is not clear what is the characteristic of the dropout training process and how to bridge the training with the generalization in non-linear neural networks. In this work, we show the implicit regularization of dropout, which may be a key factor in enabling dropout to find solutions with better generalization.


The modified gradient flow is defined as the gradient flow which is close to discrete iterates of the original training path up to some high-order learning rate term~\cite{hairer2003geometric}. \cite{barrett2020implicit} derive the modified gradient flow of discrete full-batch gradient descent training as $\hat{R}_{S,GD}(\vtheta)=R_S({\vtheta})+(\varepsilon/4)\norm{\nabla R_S({\vtheta})}^2+O(\varepsilon^2)$, where $R_S({\vtheta})$ is the training loss on dataset $S$, $\varepsilon$ is the learning rate and $\norm{\cdot}$ denotes the $l_2$-norm. In a similar vein, \cite{smith2020origin} derive the modified gradient flow of stochastic gradient descent training as $\hat{R}_{S,SGD}(\vtheta)=R_S({\vtheta})+(\varepsilon/4)\norm{\nabla R_S({\vtheta})}^2+(\varepsilon/4m)\sum_{i=0}^{m-1} \norm{\nabla R_{S,i}({\vtheta})-\nabla R_S({\vtheta})}^2+O(\varepsilon^2)$, where $R_{S,i}({\vtheta})$ is the $i$th batch loss and the last term is also called ``non-uniform'' term \cite{wu2018sgd}. Our work shows that there exist several distinct features between dropout and SGD. Specifically, in the limit of the vanishing learning rate, the modified gradient flow of dropout still has an additional implicit regularization term, whereas that of SGD converges to the full-batch gradient flow \cite{yaida2018fluctuation}.

The parameter initialization of the network determines the final fitting result of the network. \cite{luo2021phase,zhou2022empirical} mainly identify the linear regime and the condensed regime for two-layer and three-layer wide ReLU NNs. In the linear regime, the training dynamics of NNs are approximately linear and similar to a random feature model with an exponential loss decay. In the condensed regime, active neurons are condensed at several discrete orientations, which may be an underlying reason why NNs outperform traditional algorithms.

\cite{zhang2021embedding,zhang2022embedding} show that NNs of different widths often exhibit similar condensation behavior, e.g., stagnating at a similar loss with almost the same output function. Based on this observation, they propose the embedding principle that the loss landscape of an NN contains all critical points of all narrower NNs. The embedding principle provides a basis for understanding why condensation occurs from the perspective of loss landscape. 

Several works study the mechanism of condensation at the initial training stage, such as for ReLU network \cite{maennel2018gradient,pellegrini2020analytic} and network with continuously differentiable activation functions \cite{zhou2021towards}.  However, studying condensation throughout the whole training process is generally challenging, with dropout training being an exception. The regularization terms we derive in this work show that the dropout training tends to condense in the whole training process.

\section{Preliminary}

\subsection{Deep Neural Networks}
Consider a $L$-layer ($L\geq 2$) fully-connected neural network (FNN). We regard the input as the $0$th layer and the output as the $L$th layer. Let $m_l$ represent the number of neurons in the $l$th layer. In particular, $m_0=d$ and $m_L=d'$. For any $i,k\in \sN$ and $i<k$, we denote $[i:k]=\{i,i+1,\ldots,k\}$. In particular, we denote $[k]:=\{1,2,\ldots,k\}$.

Given weights $W^{[l]}\in \sR^{m_l\times m_{l-1}}$ and biases $b^{[l]}\in\sR^{m_{l}}$ for $l\in[L]$, we define the collection of parameters $\vtheta$ as a $2L$-tuple (an ordered list of $2L$ elements) whose elements are matrices or vectors
\begin{equation*}
    \vtheta=\Big(\vtheta|_1,\cdots,\vtheta|_L\Big)=\Big(\mW^{[1]},\vb^{[1]},\ldots,\mW^{[L]},\vb^{[L]}\Big),
\end{equation*}
where the $l$th layer parameters of $\vtheta$ is the ordered pair $\vtheta|_{l}=\Big(\mW^{[l]},\vb^{[l]}\Big),\quad l\in[L]$.
We may misuse notation and identify $\vtheta$ with its vectorization $\mathrm{vec}(\vtheta)\in \sR^M$, where $M=\sum_{l=0}^{L-1}(m_l+1) m_{l+1}$. 

Given $\vtheta\in \sR^M$, the FNN function $\vf_{\vtheta}(\cdot)$ is defined recursively. First, we denote $\vf^{[0]}_{\vtheta}(\vx)=\vx$ for all $\vx\in\sR^d$. Then, for $l\in[L-1]$, $\vf^{[l]}_{\vtheta}$ is defined recursively as 
$\vf^{[l]}_{\vtheta}(\vx)=\sigma (\mW^{[l]} \vf^{[l-1]}_{\vtheta}(\vx)+\vb^{[l]})$, where $\sigma$ is a non-linear activation function.
Finally, we denote
\begin{equation*}
    \vf_{\vtheta}(\vx)=\vf(\vx,\vtheta)=\vf^{[L]}_{\vtheta}(\vx)=\mW^{[L]} \vf^{[L-1]}_{\vtheta}(\vx)+\vb^{[L]}.
\end{equation*}
For notational simplicity, we denote

\begin{equation*}
    \vf_{\vtheta}^{j}(\vx_i)=\mW^{[L]}_{j} f^{[L-1]}_{\vtheta, j}(\vx_i), 
\end{equation*}
where $\vf_{\vtheta}^{j}(\vx_i), \mW^{[L]}_{j} \in  \sR^{m_L}$ is the $j$th column of $\mW^{[L]}$, and $f^{[L-1]}_{\vtheta, j}(\vx_i)$ is the $j$th element of vector $\vf^{[L-1]}_{\vtheta}(\vx_i)$. In this work, we denote the $l_2$-norm as $\norm{\cdot}$ for convenience.

\subsection{Loss Function}
The training data set is denoted as  $S=\{(\vx_i,\vy_i)\}_{i=1}^n$, where $\vx_i\in\sR^d$ and $\vy_i\in \sR^{d'}$. For simplicity, we assume an unknown function $\vy$ satisfying $\vy(\vx_i)=\vy_i$ for $i\in[n]$. The empirical risk reads as
\begin{equation}
    \RS(\vtheta)=\frac{1}{n}\sum_{i=1}^n\ell(\vf(\vx_i,\vtheta),\vy(\vx_i)), \label{eq:RS}
\end{equation}
where the loss function $\ell(\cdot,\cdot)$ is differentiable and the derivative of $\ell$ with respect to its first argument is denoted by $\nabla\ell(\vy,\vy^*)$. The error with respect to data sample $(\vx_{i},\vy_{i})$ is defined as 
\begin{equation*}
    \ve(\vf_{\vtheta}(\vx_{i}),\vy_{i})=\vf_{\vtheta}(\vx_{i})- \vy_{i}.
\end{equation*}
For notation simplicity, we denote $\ve(\vf_{\vtheta}(\vx_{i}),\vy_{i})=\ve_{\vtheta,i}$. 

\subsection{Dropout}
For $\vf_{\vtheta}^{[l]}(\vx) \in \mathbb{R}^{m_{l}}$, we randomly sample a scaling vector $\veta \in \mathbb{R}^{m_{l}}$ with coordinates of $\veta$ are sampled i.i.d that, 
\begin{equation*}
    (\veta)_{k}= \begin{cases}\frac{1-p}{p} & \text { with probability } p \\ -1 & \text { with probability } 1-p, \end{cases}
\end{equation*}
where $ p \in (0,1] $, $k \in [m_l]$ indices the coordinate of $\vf_{\vtheta}^{[l]}(\vx)$. It is important to note that $\veta$ is a zero-mean random variable. We then apply dropout by computing
\begin{equation*}
\vf_{\vtheta, \veta}^{[l]}(\vx)=(\vone+\veta) \odot \vf_{\vtheta}^{[l]}(\vx),
\end{equation*}
and use $\vf_{\vtheta, \veta}^{[l]}(\vx)$ instead of $\vf_{\vtheta}^{[l]}(\vx)$. Here we use $\odot$ for the Hadamard product of two matrices of the same dimension. To simplify notation, we let $\veta$ denote the collection of such vectors over all layers. We denote the output of model $\vf_{\vtheta} (\vx)$ on input $\vx$ using dropout noise $\veta$ as $\vf_{\vtheta, \veta}^\mathrm{drop} (\vx)$. The empirical risk associated with the network with dropout layer $\vf_{\vtheta, \veta}^\mathrm{drop}$ is denoted by $\RS ^\mathrm{drop}\left(\vtheta, \veta\right)$, given by
\begin{equation}
    \RS ^\mathrm{drop}\left(\vtheta, \veta\right)= \frac{1}{n}\sum_{i=1}^n\ell(\vf_{\vtheta, \veta}^\mathrm{drop}(\vx_i),\vy(\vx_i)). \label{eq:RSdrop}
\end{equation}

\section{Modified Gradient Flow}\label{sec:main}
In this section, we theoretically analyze the implicit regularization effect of dropout. We derive the modified gradient flow of dropout in the sense of expectation. We first summarize the settings and provide the necessary definitions used for our theoretical results below. {\emph{Note that, the settings of our experiments are much more general.}}

\begin{setting}[\textbf{dropout structure}]\label{set:1}
    Consider an $L$-layer ($L\geq 2$) FNN with only one dropout layer after the $(L-1)$th layer of the network,  
    \begin{equation*}
        \vf_{\vtheta, \veta}^\mathrm{drop} (\vx)=\mW^{[L]}(\vone+\veta) \odot \vf^{[L-1]}_{\vtheta}(\vx)+\vb^{[L]}.
    \end{equation*}
\end{setting}

\begin{setting}[\textbf{loss function}]\label{set:2}
    Take the mean squared error (MSE) as our loss function, 
    \begin{equation*}
        \RS(\vtheta)=\frac{1}{2n}\sum_{i=1}^n(\vf(\vx_i,\vtheta)-\vy_{i})^2.
    \end{equation*}
\end{setting}

\begin{setting}[\textbf{network structure}]\label{set:3}
    For convenience, we set the model output dimension to one, i.e. $m_{L}=1$.
\end{setting}


In the following, we introduce two key terms that play an important role in our theoretical results:

\begin{equation} 
	R_1(\vtheta):=\frac{1-p}{2np}\sum_{i=1}^{n} \sum_{j=1}^{m_{L-1}}\Vert \mW^{[L]}_{j}f^{[L-1]}_{\vtheta, j}(\vx_i) \Vert^2,\label{R1} 
\end{equation}

\begin{equation}  
	R_2(\vtheta):=\frac{\varepsilon}{4}\Exp_{\veta} \Vert \nabla_{\vtheta} \RS ^\mathrm{drop}\left(\vtheta, \veta\right) \Vert ^2, \label{R2} 
\end{equation}
where $\mW^{[L]}_{j} \in  \sR^{m_L}$ is the $j$-th column of $\mW^{[L]}$, $f^{[L-1]}_{\vtheta, j}(\vx_i)$ is the $j$-th element of $\vf^{[L-1]}_{\vtheta}(\vx_i)$, $\Exp_{\veta}$ is the expectation with respect to $\veta$, and $\varepsilon$ is the learning rate. 

Based on the above settings, we obtain a modified equation based on dropout gradient flow. 

\begin{lemma}[\textbf{the expectation of dropout loss}] \label{lem:1}
    Given an $L$-layer FNN with dropout $\vf_{\vtheta, \veta}^\mathrm{drop} (\vx)$, under Setting 1--3, we have the expectation of dropout MSE: 
    \begin{align*}
     \Exp_{\veta}&(R_S^\mathrm{drop}\left(\vtheta, \veta\right))= R_S\left(\vtheta\right)+ R_1(\vtheta).
    \end{align*}

\end{lemma} 

Based on the above lemma, we proceed to study the discrete iterate training of gradient descent with dropout, resulting in the derivation of the modified gradient flow of dropout training. 

\textbf{Modified gradient flow of dropout.}\label{thm:1} Under Setting 1--3, the mean iterate of $\vtheta$, with a learning rate $\varepsilon \ll 1$, stays close to the path of gradient flow on a modified loss $\dot{\vtheta}=-\nabla_{\vtheta}\tilde{R}_S^\mathrm{drop}(\vtheta,\veta)$, where the modified loss $\tilde{R}_S^{\mathrm{drop}}(\vtheta,\veta)$ satisfies:
\begin{equation}
    \Exp_{\veta}\tilde{R}_S^{\mathrm{drop}}(\vtheta,\veta)\approx\RS \left(\vtheta\right)+ R_1(\vtheta)+R_2(\vtheta).\label{equ:equal}
\end{equation}

Contrary to SGD \cite{smith2020origin}, the $R_1(\vtheta)$ term is independent of the learning rate $\varepsilon$, thus the implicit regularization of dropout still affects the gradient flow even as the learning rate $\varepsilon$ approaches zero. In Section~\ref{sec:conde}, we show the $R_1(\vtheta)$ term makes the network tend to find solutions with lower complexity, that is, solutions with weight condensation, which is also illustrated and supported numerically. In Section~\ref{sec:flat}, we show the $R_1(\vtheta)$ term plays a more important role in improving the generalization and flatness of the model than the $R_2(\vtheta)$ term, which explicitly aims to find a flatter solution.




\section{Numerical Verification of Implicit Regularization Terms}
In this section, we numerically verify the validity of two implicit regularization terms, i.e., $R_1(\vtheta)$ defined in Equation~(\ref{R1}) and $R_2(\vtheta)$ defined in Equation~(\ref{R2}), under more general settings than out theoretical results. The detailed experimental settings can be found in Appendix \ref{app:1}.

\subsection{Validation of the Effect of $\texorpdfstring{R_1(\vtheta)}{E=energy, m=mass, c=speed\ of\ light}$}

\begin{figure}[h]
	\centering
	\subfloat[]{\includegraphics[width=0.4\textwidth]{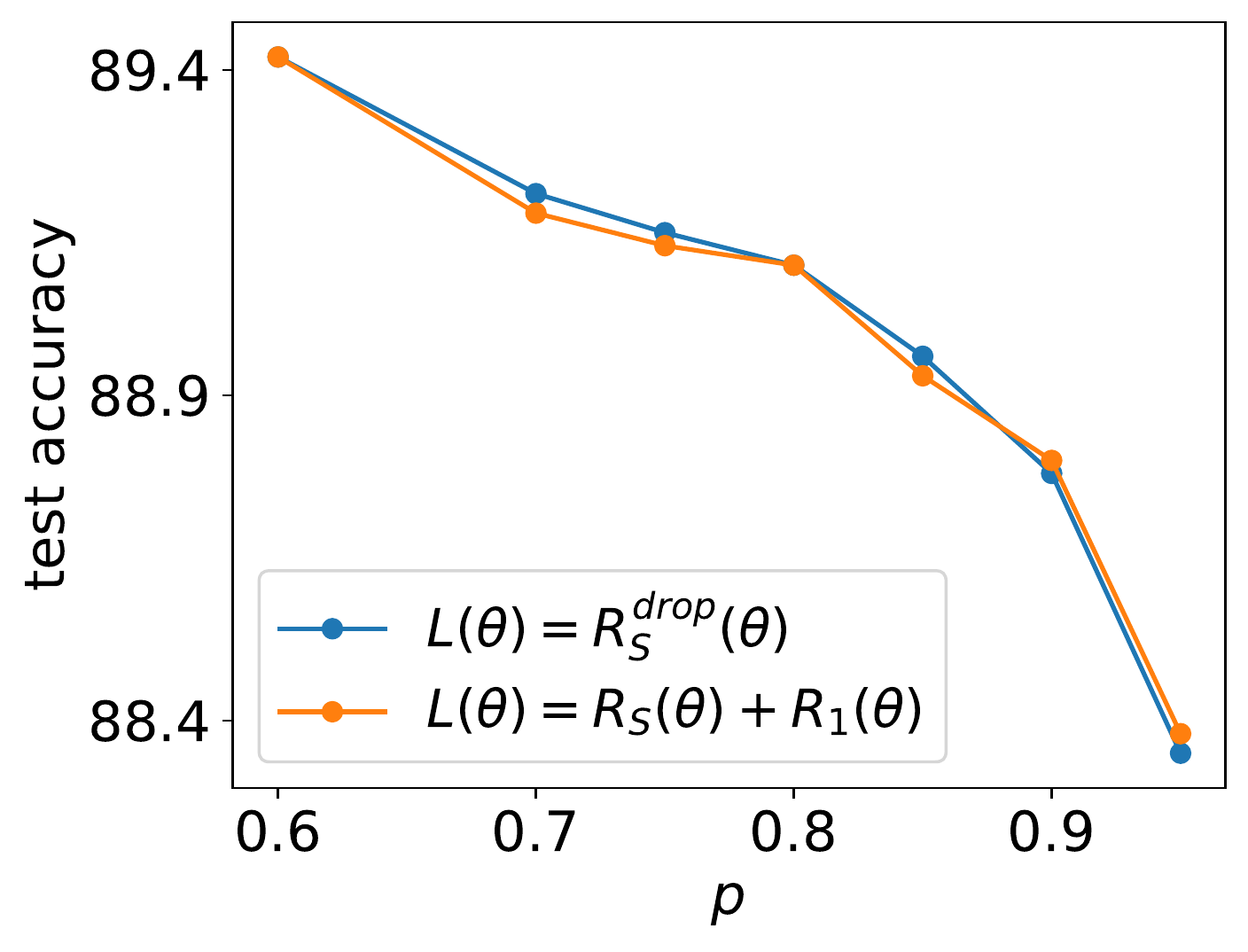}}
	\subfloat[]{\includegraphics[width=0.4\textwidth]{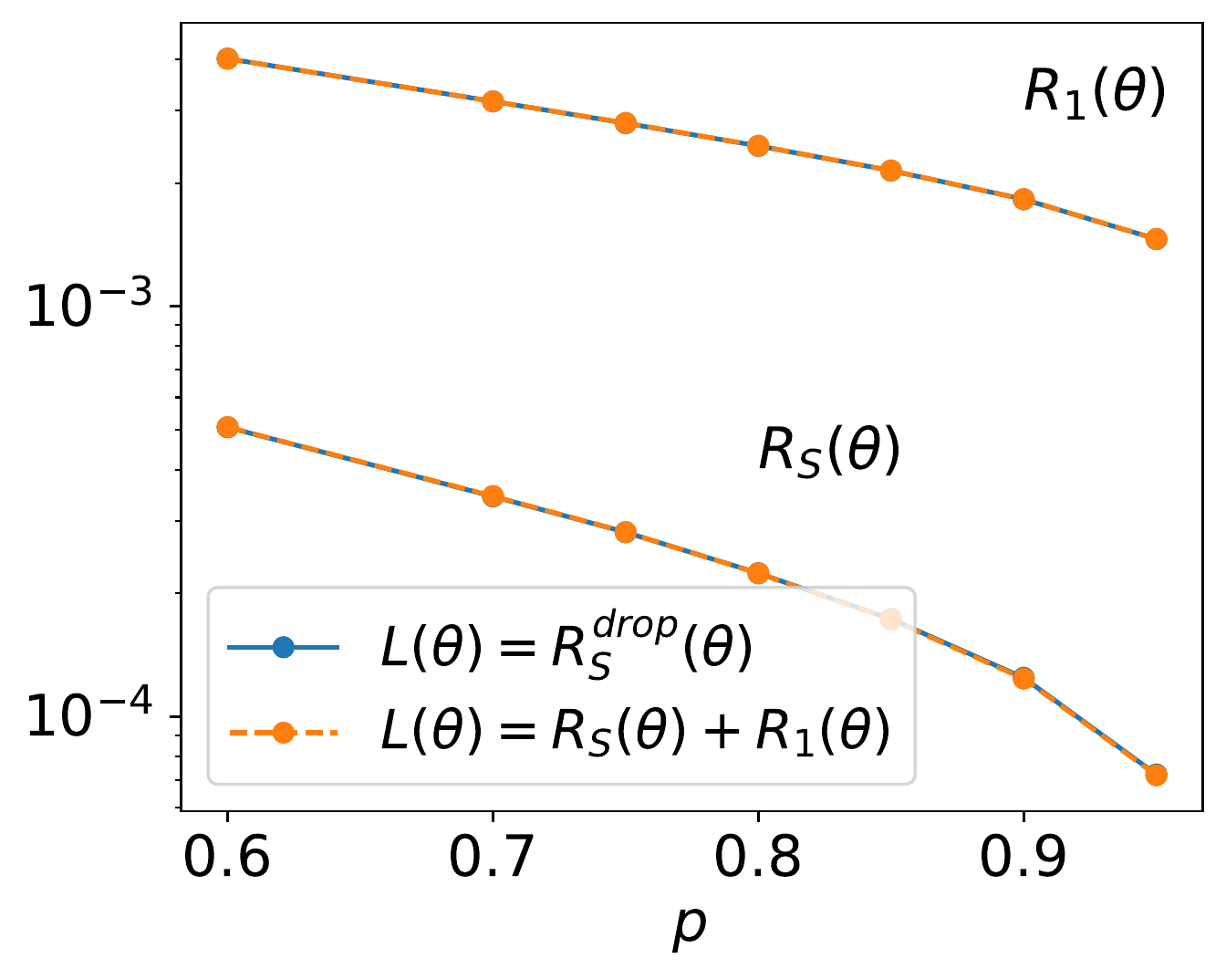}}
    \caption{Two-layer NNs of width 1000 for the classification of the first 1000 images of MNIST dataset, utilizing two distinct loss functions: $\RS^\mathrm{drop}(\vtheta,\veta)$ and $\RS \left(\vtheta\right)+R_1\left(\vtheta\right)$. To study the impact of different dropout rates on the performance of the networks, we conduct experiments with varying dropout rates while maintaining a constant learning rate of $\varepsilon=5\times 10^{-3}$. (a) The test accuracy. (b) The value of $\RS^\mathrm{drop}(\vtheta,\veta)$ and $R_1\left(\vtheta\right)$.  \label{fig:R1}}
\end{figure} 

As $R_1(\vtheta)$ is independent of the learning rate and $R_2(\vtheta)$ vanishes in the limit of zero learning rate, we select a small learning rate to verify the validity of $R_1(\vtheta)$. According to Equation (\ref{equ:equal}), the modified equation of dropout training dynamics can be approximated by $\RS \left(\vtheta\right)+R_1\left(\vtheta\right)$ when the learning rate $\varepsilon$ is sufficiently small. Therefore, we verify the validity of $R_1(\vtheta)$ through the similarity of the NN trained by the two loss functions, i.e., $\RS^\mathrm{drop}(\vtheta,\veta)$ and $\RS \left(\vtheta\right)+R_1\left(\vtheta\right)$ under a small learning rate. 

Fig. \ref{fig:R1}(a) presents the test accuracy of two losses trained under different dropout rates. For the network trained with $\RS \left(\vtheta\right)+R_1\left(\vtheta\right)$, there is no dropout layer, and the dropout rate affects the weight of $R_1\left(\vtheta\right)$ in the loss function. For different dropout rates, the networks obtained by the two losses above exhibit similar test accuracy. It is worth mentioning that for the network trained with $\RS \left(\vtheta\right)$, the obtained accuracy is only $79\%$, which is significantly lower than the accuracy of the network trained through the two loss functions above (over $88\%$ in Fig. \ref{fig:R1}(a)). In Fig. \ref{fig:R1}(b), we show the values of $\RS \left(\vtheta\right)$ and $R_1\left(\vtheta\right)$ for the two networks at different dropout rates. Note that for the network obtained by $\RS^\mathrm{drop}(\vtheta,\veta)$ training, we can calculate the two terms through the network's parameters. It can be seen that for different dropout rates, the values of $\RS \left(\vtheta\right)$ and $R_1\left(\vtheta\right)$ of the two networks are almost indistinguishable.

\begin{figure}[h]
	\centering
	\subfloat[]{\includegraphics[width=0.4\textwidth]{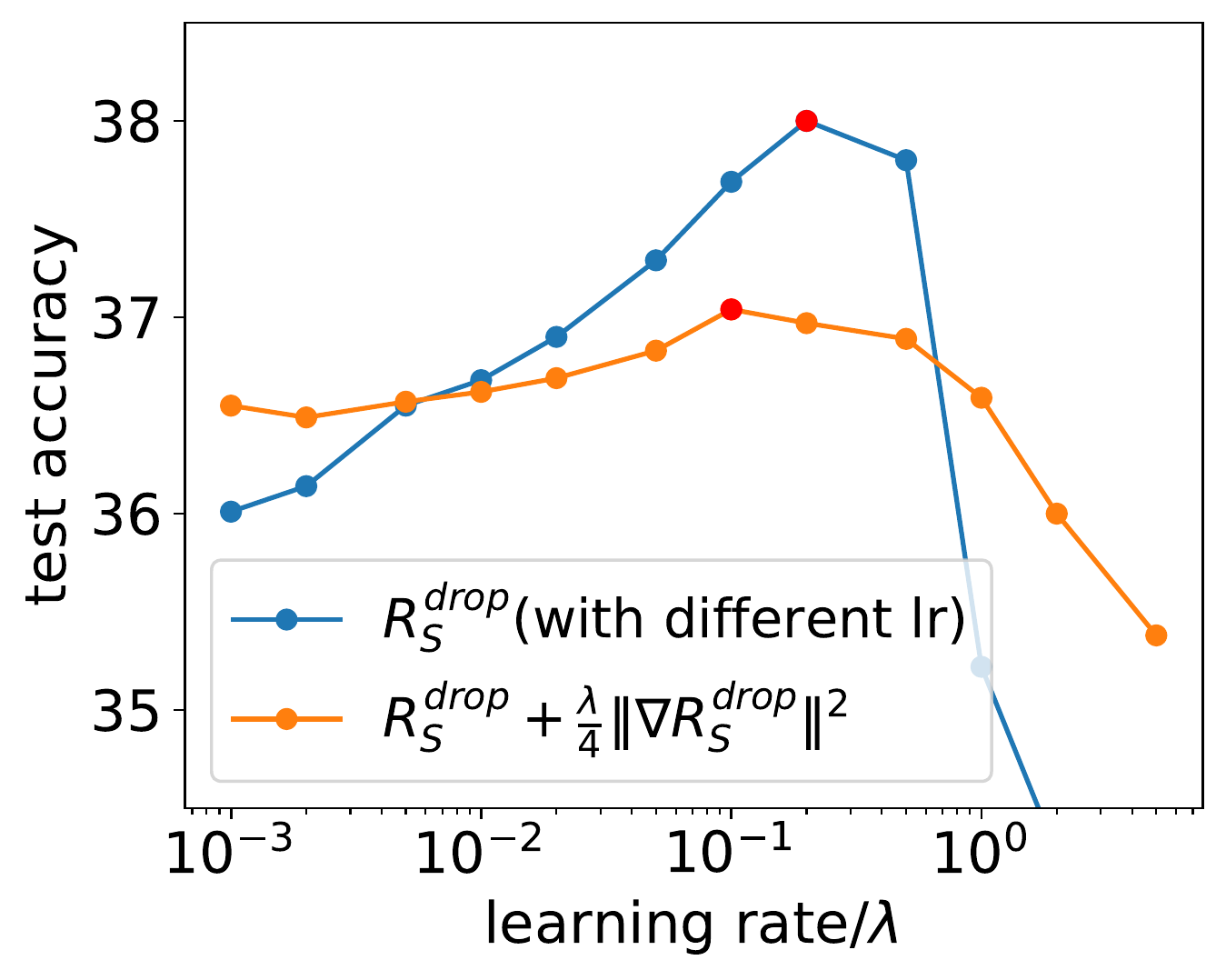}}
	\subfloat[]{\includegraphics[width=0.4\textwidth]{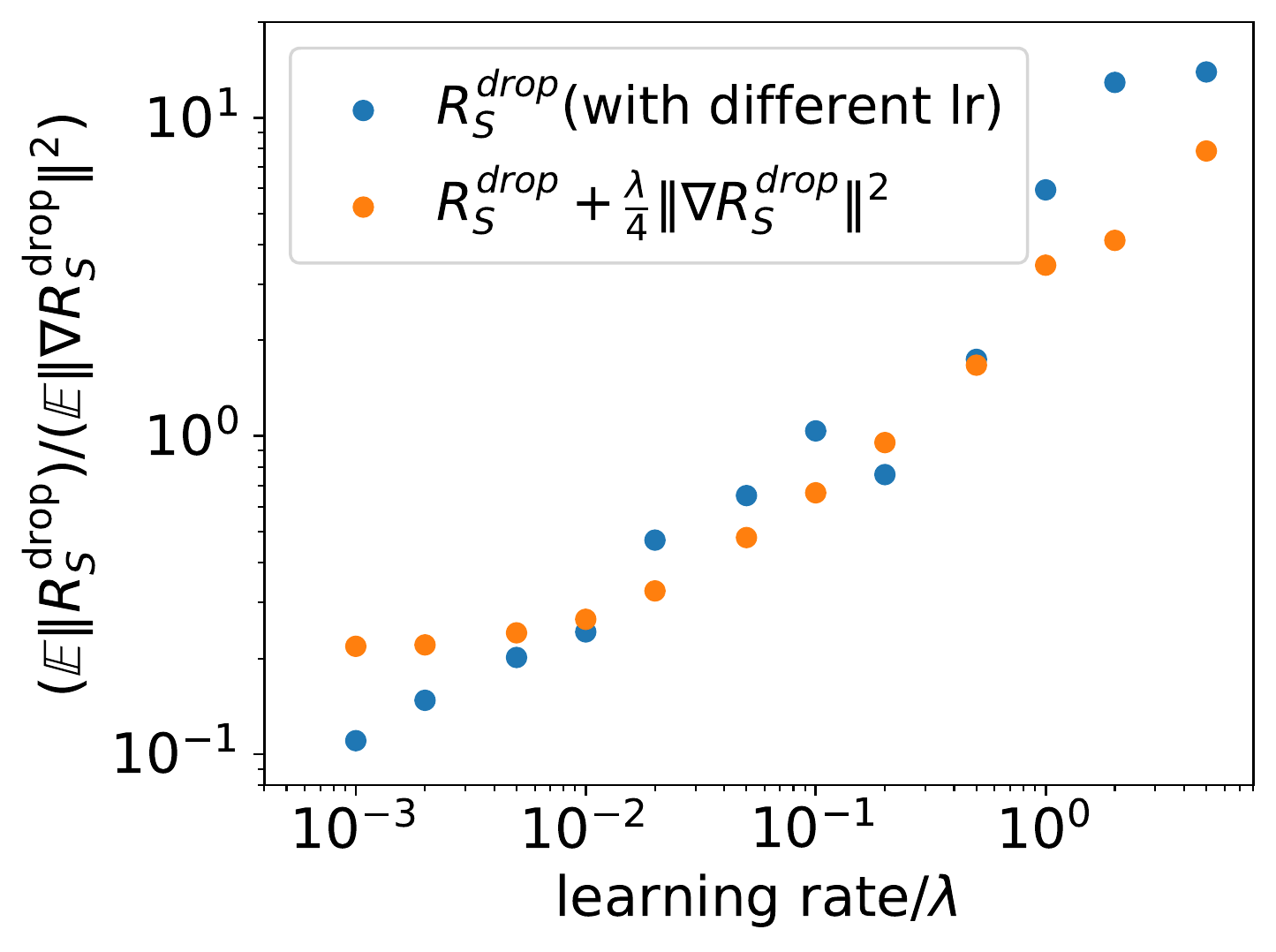}}
    \caption{Classify the first 1000 images of CIFAR-10 by training VGG-9 under a specific loss function by GD. For loss function $\RS ^\mathrm{drop}\left(\vtheta, \veta\right)$, we train the NNs with various learning rates $\varepsilon$. For loss function $\RS ^\mathrm{drop}\left(\vtheta, \veta\right)+(\lambda/4)\Vert \nabla_{\vtheta} \RS ^\mathrm{drop}\left(\vtheta, \veta\right) \Vert ^2$, we train the NNs with various regularization coefficient $\lambda$,  while keeping the learning rate fixed at a small value of $\varepsilon=5\times 10^{-3}$.
    (a) The test accuracy of the network under different learning rates and regularization coefficients. The red dots indicate the location of the maximum test accuracy of the NNs obtained by training with both two loss functions. (b) The $(\Exp_{\veta}\left\|R_{S}^\mathrm{drop}\left(\vtheta, \veta\right)\right\|)/(\Exp_{\veta}\left\|\nabla_{\vtheta} R_{S}^\mathrm{drop}\left(\vtheta, \veta\right)\right\|^{2})$ value of the resulting model in (a) under different learning rates $\varepsilon$ or regularization coefficients $\lambda$.	\label{fig:R2}}
\end{figure} 

\subsection{Validation of the Effect of $\texorpdfstring{R_2(\vtheta)}{E=energy, m=mass, c=speed\ of\ light}$}
 
As shown in Theorem \ref{thm:1}, the modified loss $\tilde{R}_S^{\mathrm{drop}}(\vtheta,\veta)$ satisfies the equation:
\begin{equation*}
    \Exp_{\veta}\tilde{R}_S^{\mathrm{drop}}(\vtheta,\veta)=\Exp_{\veta}\left( \RS ^\mathrm{drop}\left(\vtheta, \veta\right)+\frac{\varepsilon}{4}\Vert \nabla_{\vtheta} \RS ^\mathrm{drop}\left(\vtheta, \veta\right) \Vert ^2\right).
\end{equation*}
In order to validate the effect of $R_2(\vtheta)$ in the training process, we verify the equivalence of the following two training methods: (i) training networks with a dropout layer by MSE $\RS ^\mathrm{drop}\left(\vtheta, \veta\right)$ with different learning rates $\varepsilon$; (ii) training networks with a dropout layer by MSE with an explicit regularization: 
$$\RS ^\mathrm{regu}\left(\vtheta, \veta\right):=\RS ^\mathrm{drop}\left(\vtheta, \veta\right)+(\lambda/4)\Vert \nabla_{\vtheta} \RS ^\mathrm{drop}\left(\vtheta, \veta\right) \Vert ^2 $$ 
with different values of $\lambda$ and a fixed learning rate much smaller than $\varepsilon$. The exact form of $R_2(\vtheta)$ has an expectation with respect to $\veta$, but in this subsection, we ignore this expectation in experiments for convenience. 

As shown in Fig. \ref{fig:R2}, we train the NNs by the MSE $\RS ^\mathrm{drop}\left(\vtheta, \veta\right)$ with different learning rates (blue), and the regularized MSE $\RS ^\mathrm{regu}\left(\vtheta, \veta\right)$ with a fixed small learning rate and different values of $\lambda$ (orange). In Fig. \ref{fig:R2}(a), the learning rate $\varepsilon$ and the regularization coefficient $\lambda$ are close when they reach their corresponding maximum test accuracy (red point). In addition, as shown in Fig. \ref{fig:R2}(b), we study the value of $\Exp_{\veta}\left\|R_{S}^\mathrm{drop}\left(\vtheta, \veta\right)\right\|/\Exp_{\veta}\left\|\nabla_{\vtheta} R_{S}^\mathrm{drop}\left(\vtheta, \veta\right)\right\|^{2}$ under different learning rates (blue) and regularization coefficients (orange). In practical experiments, we take 3000 different dropout noises $\veta$ to approximate the expectation after the training process. The results indicate that the same learning rate $\varepsilon$ and regularization coefficient $\lambda$ result in similar ratios.

Due to the computational cost of full-batch GD, we only use a few training samples in the above experiments. We conduct similar experiments with dropout under different learning rates and regularization coefficients using SGD as detailed in Appendix \ref{app:R2_sgd}.

\section{Dropout Facilitates Condensation} \label{sec:conde}
A condensed network, which refers to a network with neurons having aligned input weights, is equivalent to another network with a reduced width \cite{zhou2021towards,luo2021phase}. Therefore, the effective complexity of the network is smaller than its superficial appearance. Such low effective complexity may be an underlying reason for good generalization. In addition, the embedding principle  \cite{zhang2021embedding,zhang2022embedding,bai2022embedding} shows that although the condensed network is equivalent to a smaller one in the sense of approximation, it has more degeneracy and more descent directions that may lead to a simpler training process. 

In this section, we experimentally and theoretically study the effect of dropout on the condensation phenomenon.

\subsection{Experimental Results}

To empirically validate the effect of dropout on condensation, we examine ReLU and tanh activations in one-dimensional and high-dimensional fitting problems, as well as image classification problems. Due to space limitations, some experimental results and detailed experimental settings are left in Appendices \ref{app:1}, \ref{app:exp}.

\subsubsection{Network with One-dimensional Input}
We train a tanh NN with 1000 hidden neurons for the one-dimensional fitting problem to fit the data shown in Fig. \ref{pic:condense_tanh} with MSE. Additional experimental verifications on ReLU NNs are provided in Appendix \ref{app:exp}. The experiments performed with and without dropout under the same initialization can both well fit the training data. In order to clearly study the effect of dropout on condensation, we take the parameter initialization distribution in the linear regime \cite{luo2021phase}, where condensation does not occur without additional constraints. The dropout layer is used after the hidden layer of the two-layer network (top row) and used between the hidden layers and after the last hidden layer of the three-layer network (bottom row). Upon close inspection of the fitting process, we find that the output of NNs trained without dropout in Fig. \ref{pic:condense_tanh}(a, e) has much more oscillation than the output of NNs trained with dropout in Fig. \ref{pic:condense_tanh}(b, f). To better understand the underlying effect of dropout, we study the feature of parameters.

\begin{figure*}[h]
	\centering
	\subfloat[$p=1$, output]{\includegraphics[width=0.24\textwidth]{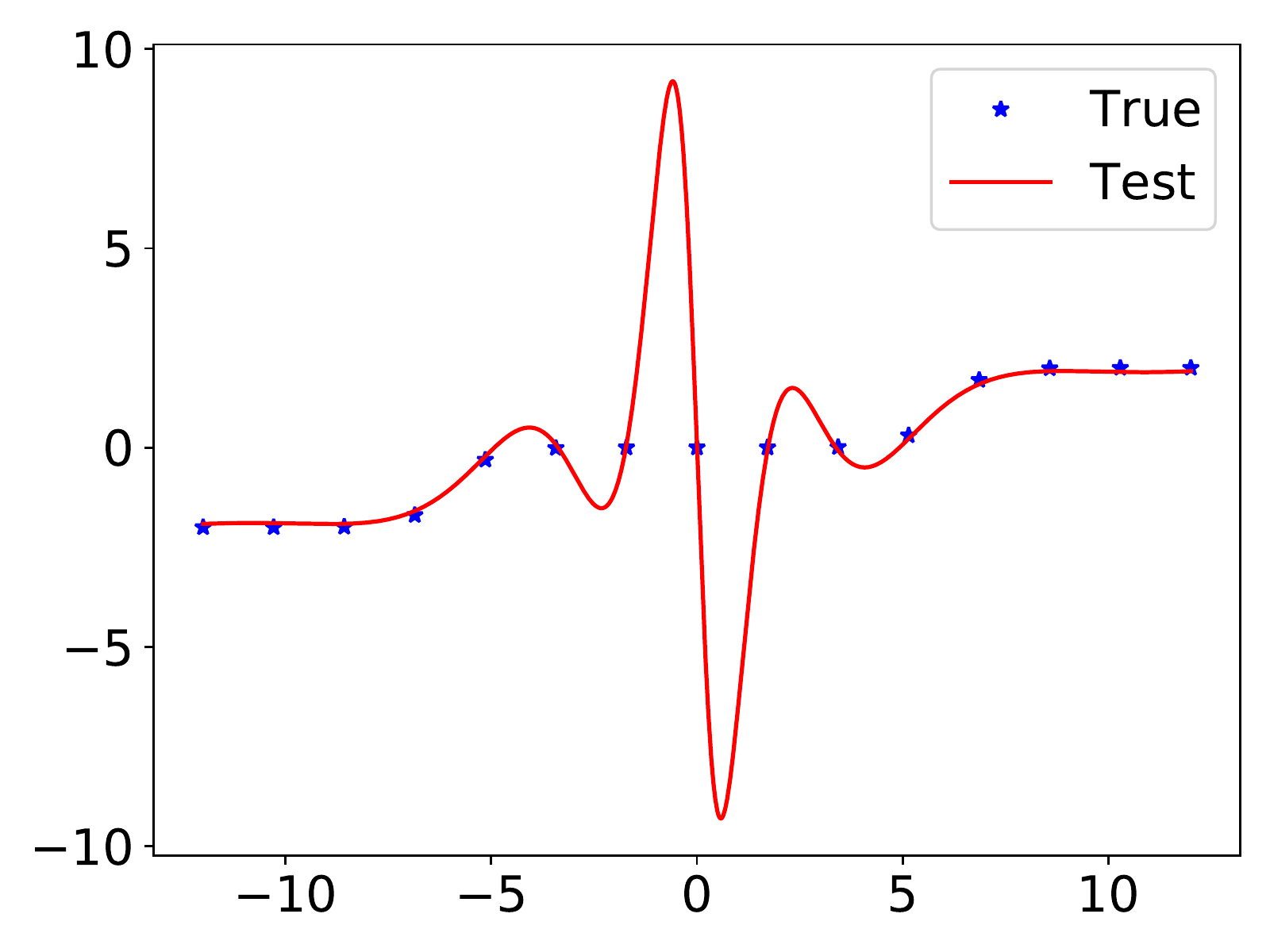}}
	\subfloat[$p=0.9$, output]{\includegraphics[width=0.24\textwidth]{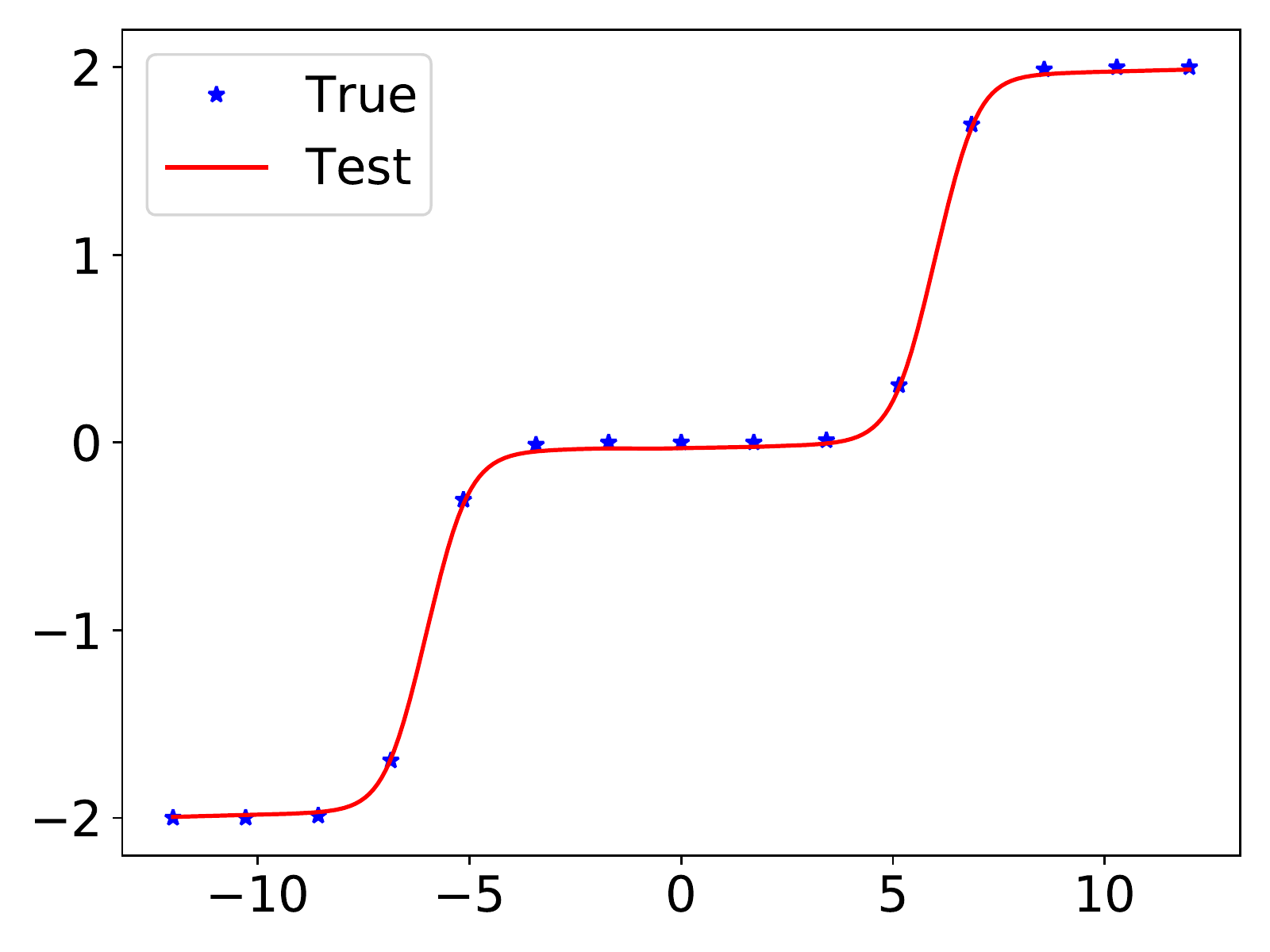}}
	\subfloat[$p=1$, feature]{\includegraphics[width=0.24\textwidth]{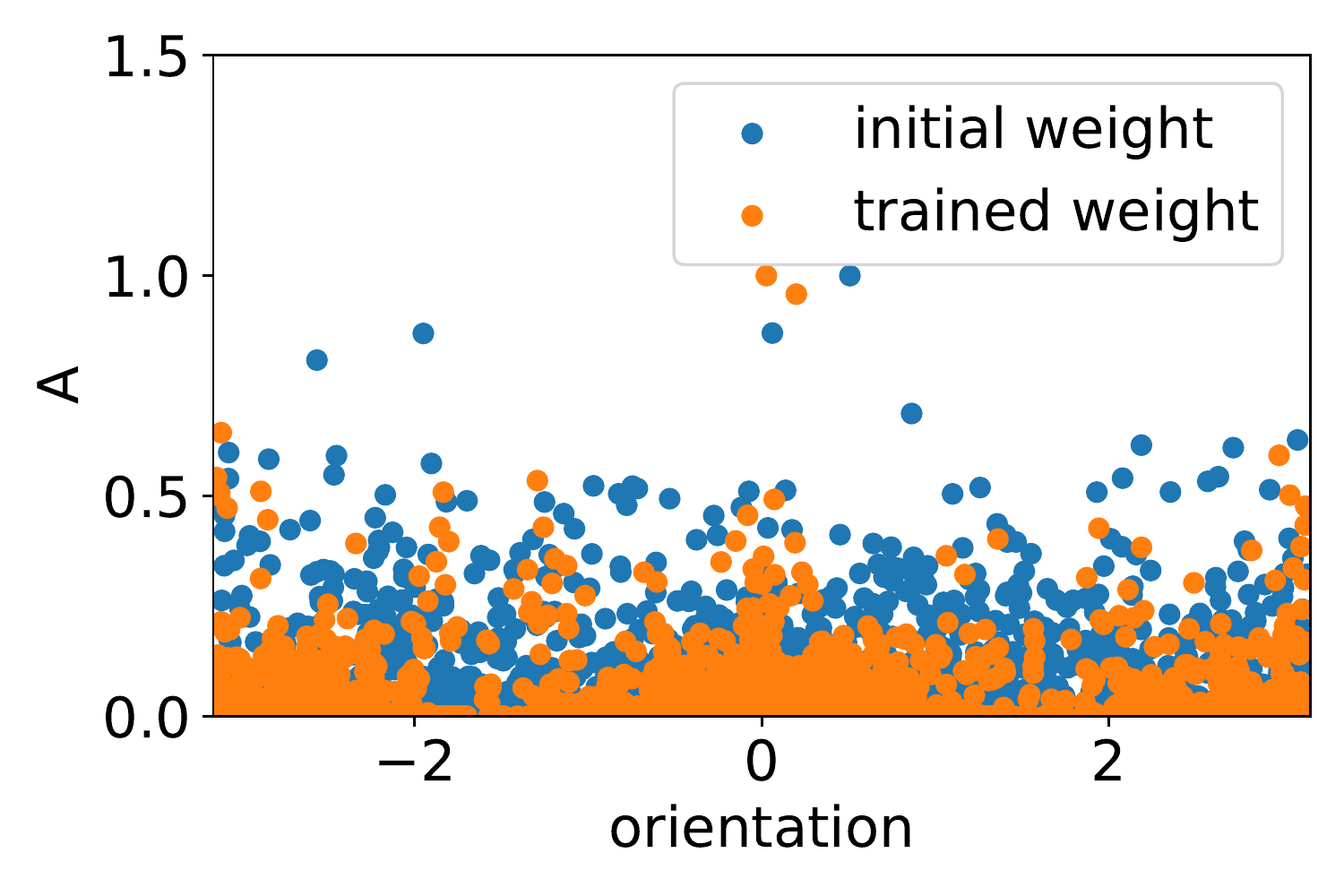}}
	\subfloat[$p=0.9$, feature]{\includegraphics[width=0.24\textwidth]{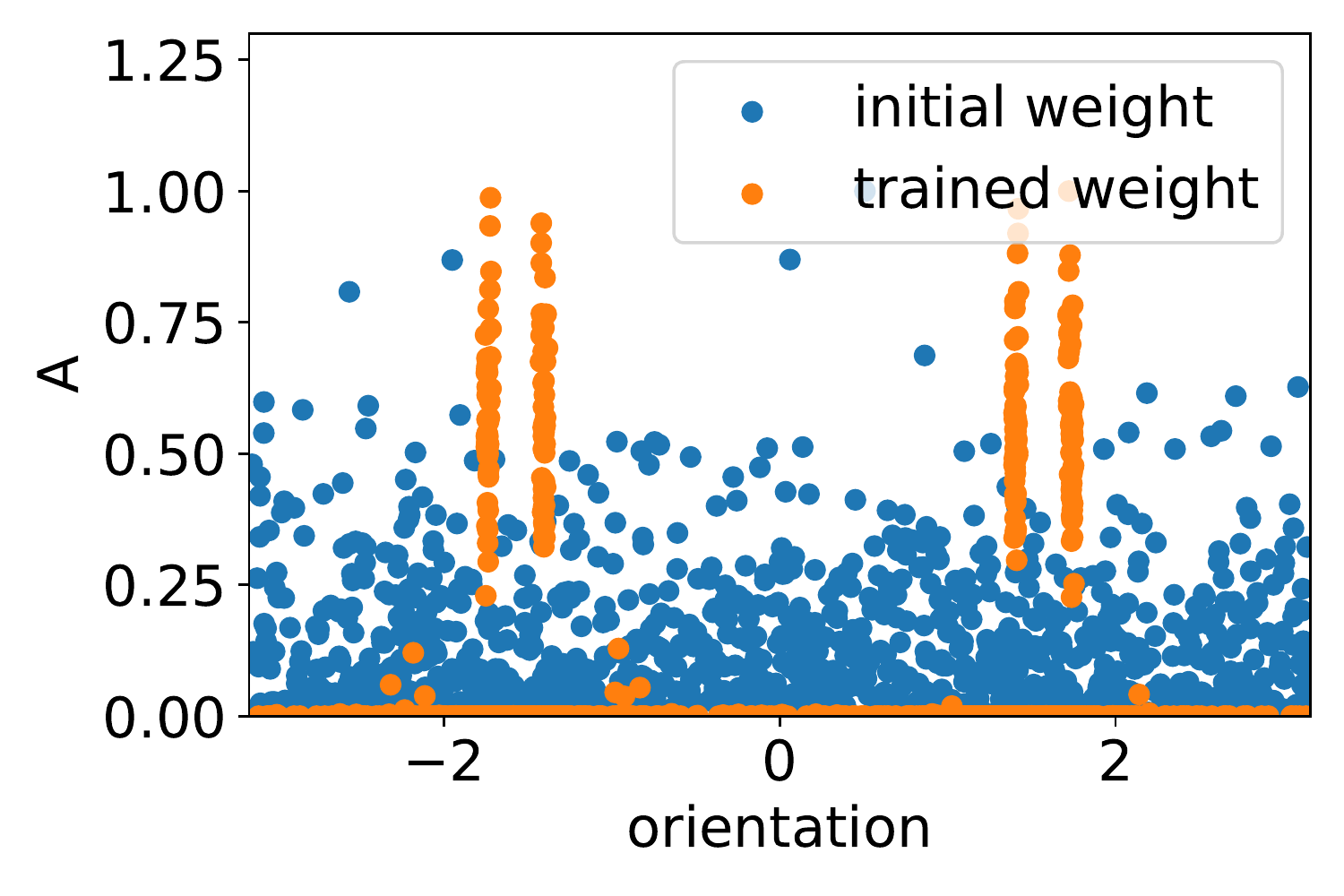}}\\
	\subfloat[$p=1$, output]{\includegraphics[width=0.24\textwidth]{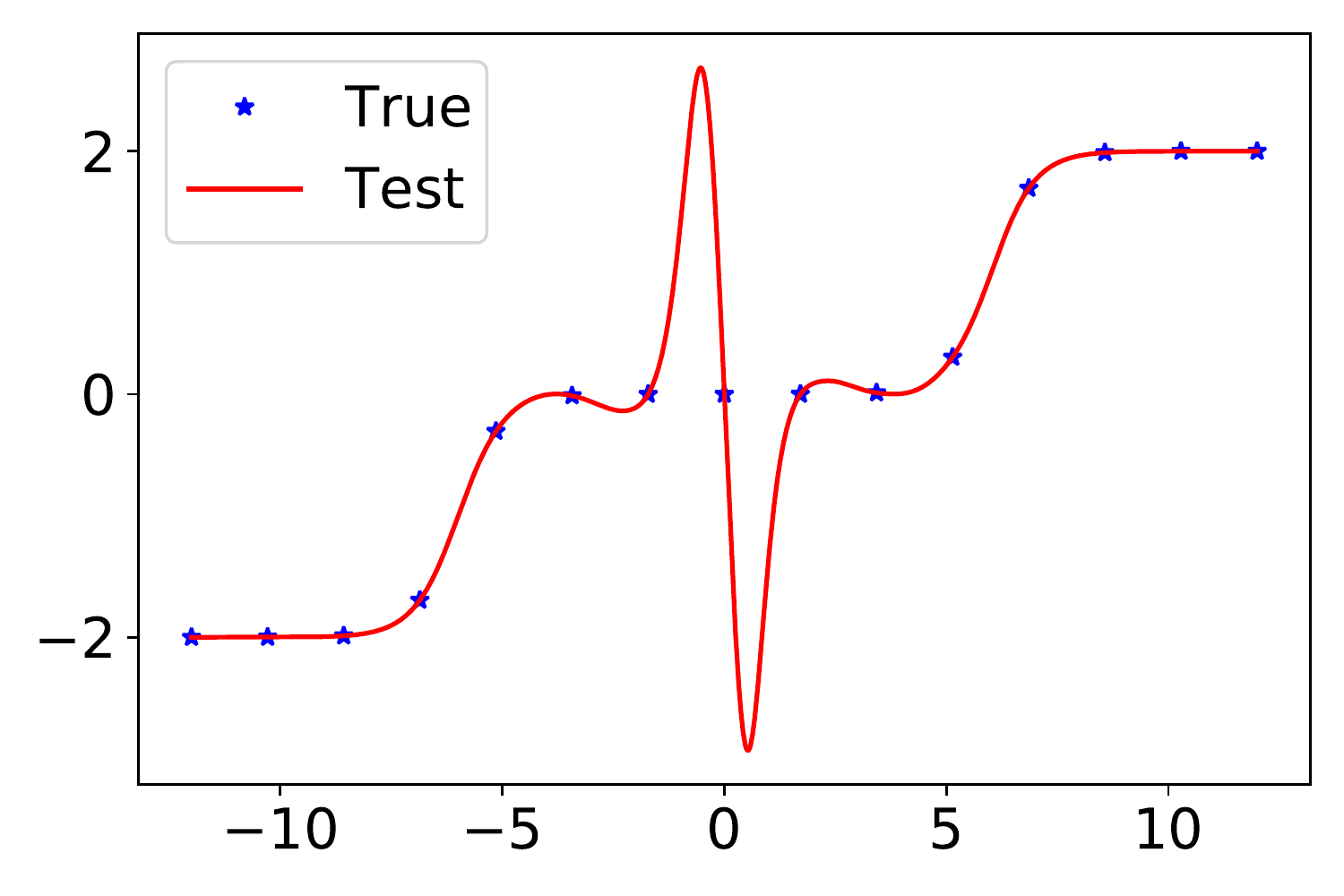}}
	\subfloat[$p=0.9$, output]{\includegraphics[width=0.24\textwidth]{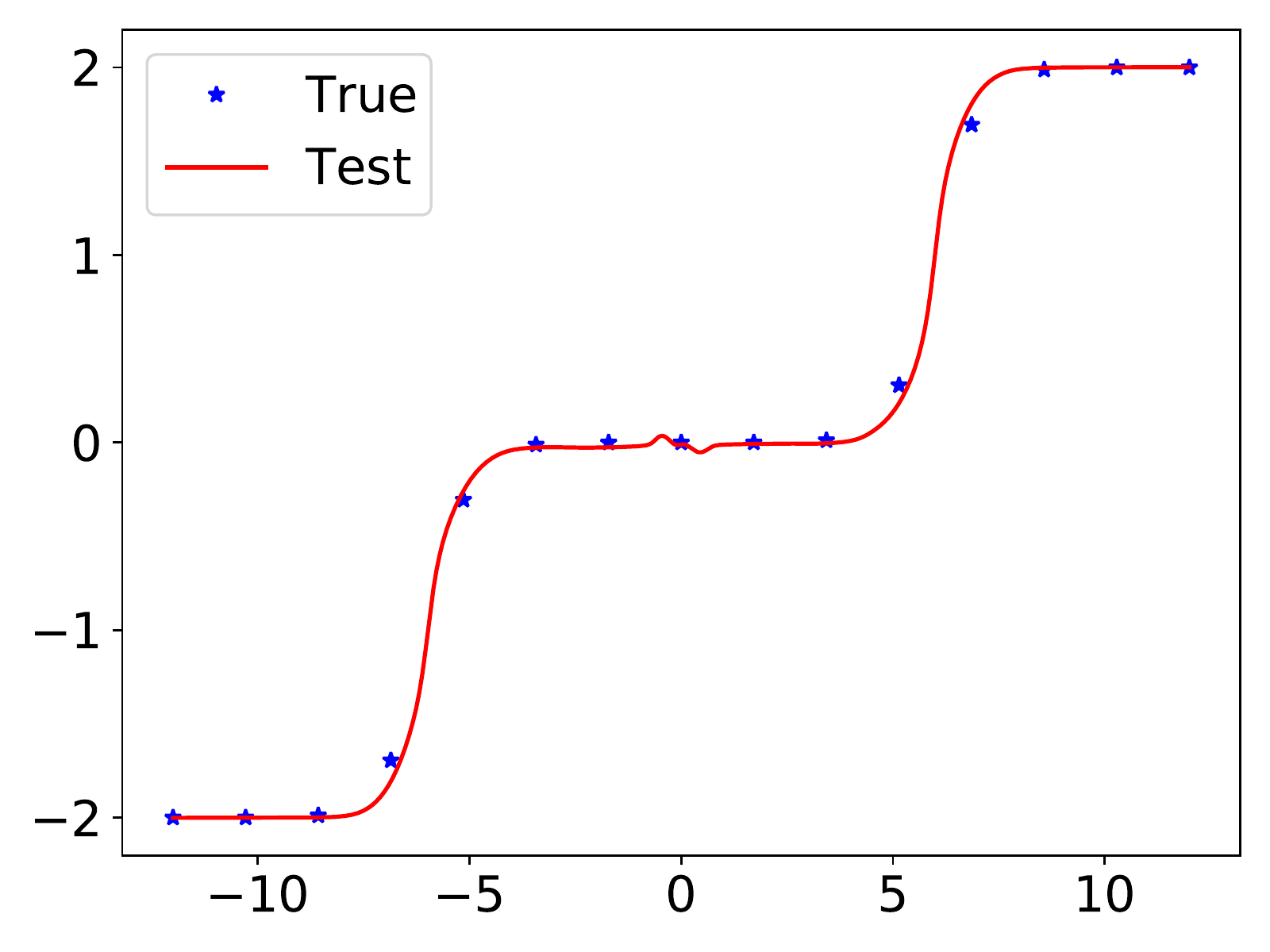}}
	\subfloat[$p=1$, feature]{\includegraphics[width=0.24\textwidth]{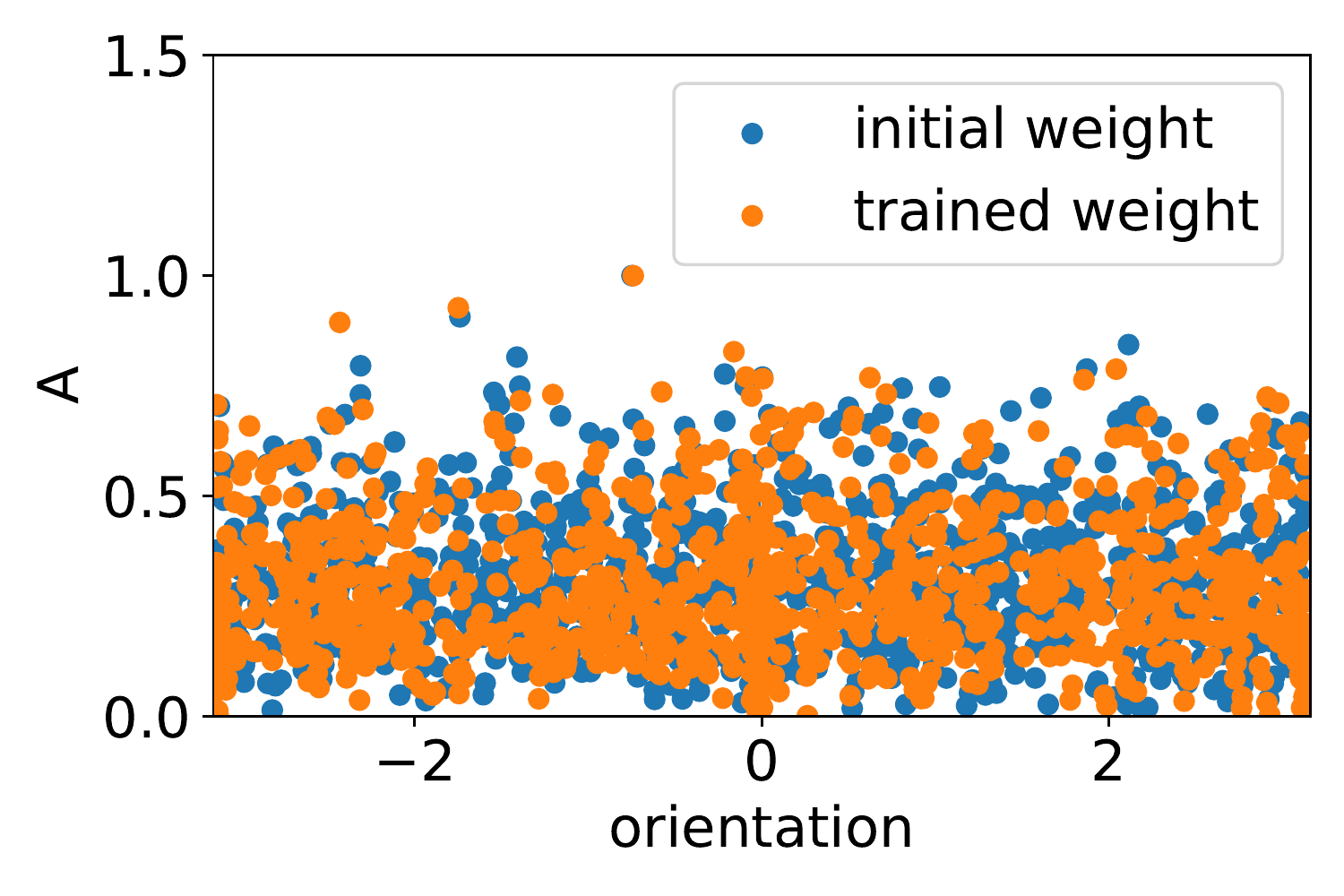}}
	\subfloat[$p=0.9$, feature]{\includegraphics[width=0.24\textwidth]{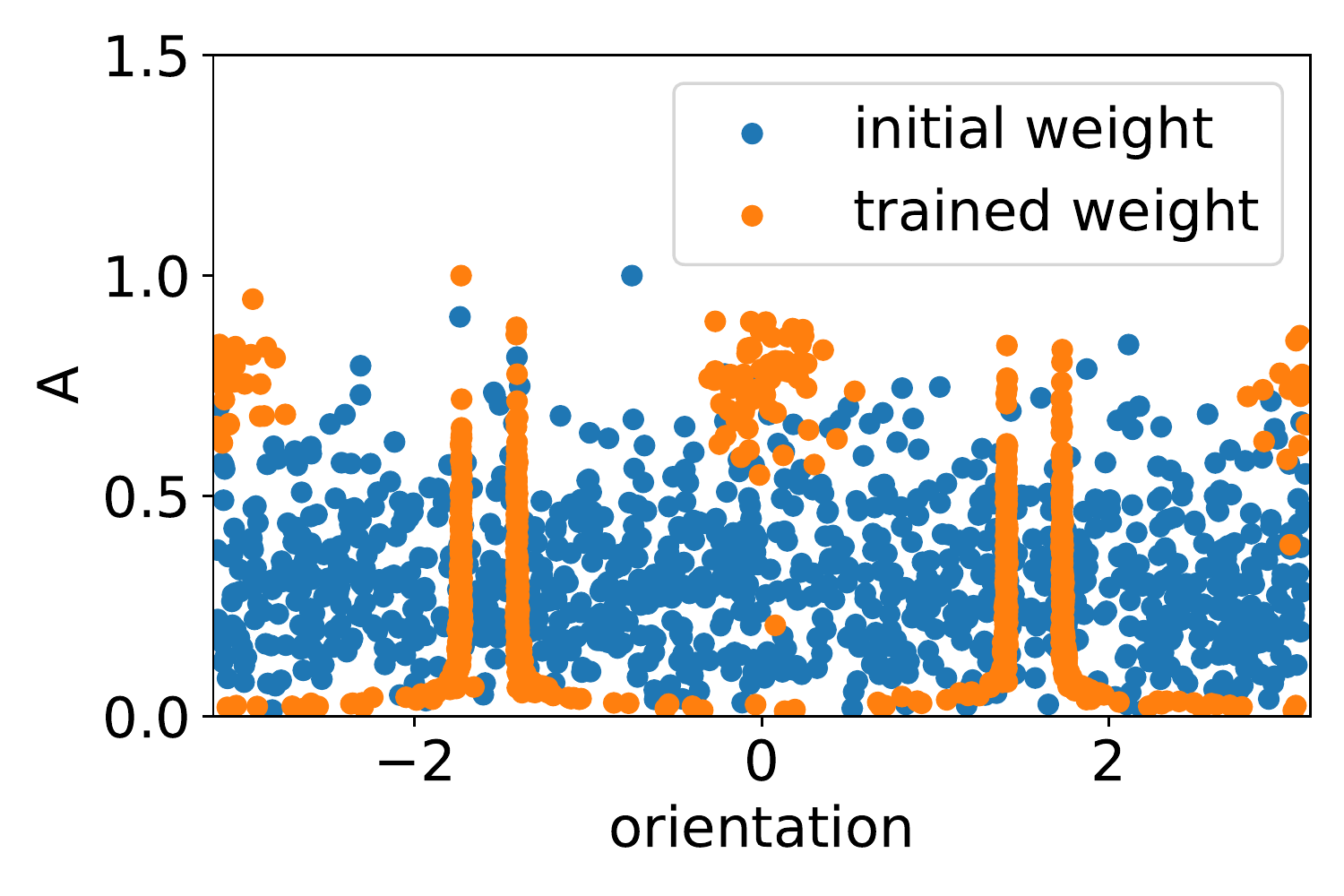}}

  \caption{ tanh NNs outputs and features under different dropout rates. The width of the hidden layers is $1000$, and the learning rate for different experiments is $1\times10^{-3}$. In (c,d,g,h), blue dots and orange dots are for the weight feature distribution at the initial and final training stages, respectively. The top row is the result of two-layer networks, with the dropout layer after the hidden layer. The bottom row is the result of three-layer networks, with the dropout layer between the two hidden layers and after the last hidden layer. Refer to Appendix \ref{app:relu_further} for further experiments on ReLU NNs. \label{pic:condense_tanh}}
\end{figure*} 


The parameter pair $(a_j,\vw_j)$ of each neuron can be separated into a unit orientation feature $\hat{\vw}_j=\vw_j/\norm{\vw_j}_{2}$ and an amplitude $A_j=|a_j|\norm{\vw_j}_{2}$ indicating its contribution\footnote{Due to the homogeneity of ReLU neurons, this amplitude can accurately describe the contribution of ReLU neurons. For tanh neurons, the amplitude has a certain positive correlation with the contribution of each neuron.} to the output, i.e., $(\hat{\vw}_j,A_j)$. For a one-dimensional input, $\vw_j$ is two-dimensional due to the incorporation of bias. Therefore, we use the angle to the $x$-axis in $[-\pi,\pi)$ to indicate the orientation of each $\hat{\vw}_j$. For simplicity, for the three-layer network with one-dimensional input, we only consider the input weight of the first hidden layer. The scatter plots of $\{(\hat{\vw}_j,|a_j|)\}_{j=1}^{m}$ and $\{(\hat{\vw}_j,\norm{\vw_j}_2)\}_{j=1}^{m}$ of tanh activation are presented in Appendix \ref{app:tanh_further} to eliminate the impact of the non-homogeneity of tanh activation.

The scatter plots of $\{(\hat{\vw}_j,A_j)\}_{j=1}^{m}$ of the NNs are shown in Fig.~\ref{pic:condense_tanh}(c,d,g,h). For convenience, we normalize the feature distribution of each model parameter such that the maximum amplitude of neurons in each model is $1$. Compared with the initial weight distribution (blue), the weight trained without dropout (orange) is close to its initial value. However, for the NNs trained with dropout, the parameters after training are significantly different from the initialization, and the non-zero parameters tend to condense on several discrete orientations, showing a condensation tendency.

In addition, we study the stability of the model trained with the loss function $R_S(\vtheta)$ under the two loss functions $R_S^{\mathrm{drop}}(\vtheta)$ and $R_S(\vtheta)+R_1(\vtheta)$. As shown in the left panel of Fig. \ref{pic:add_dropout_tanh}, we use $R_S(\vtheta)$ as the loss function to train the model before the dashed line where  when $R_S(\vtheta)$ is small, and we then replace the loss function by $R_S^{\mathrm{drop}}(\vtheta)$ or $R_S(\vtheta)+R_1(\vtheta)$. The outputs and features of the models trained with these three loss functions are shown in the middle and right panels of Fig. \ref{pic:add_dropout_tanh}, respectively. The results reveal that dropout ($R_1(\vtheta)$ term) aids the training process in escaping from the minima obtained by $R_S(\vtheta)$ training and finding a condensed solution.

\begin{figure}[h]
	\centering
	\includegraphics[width=0.9\textwidth]{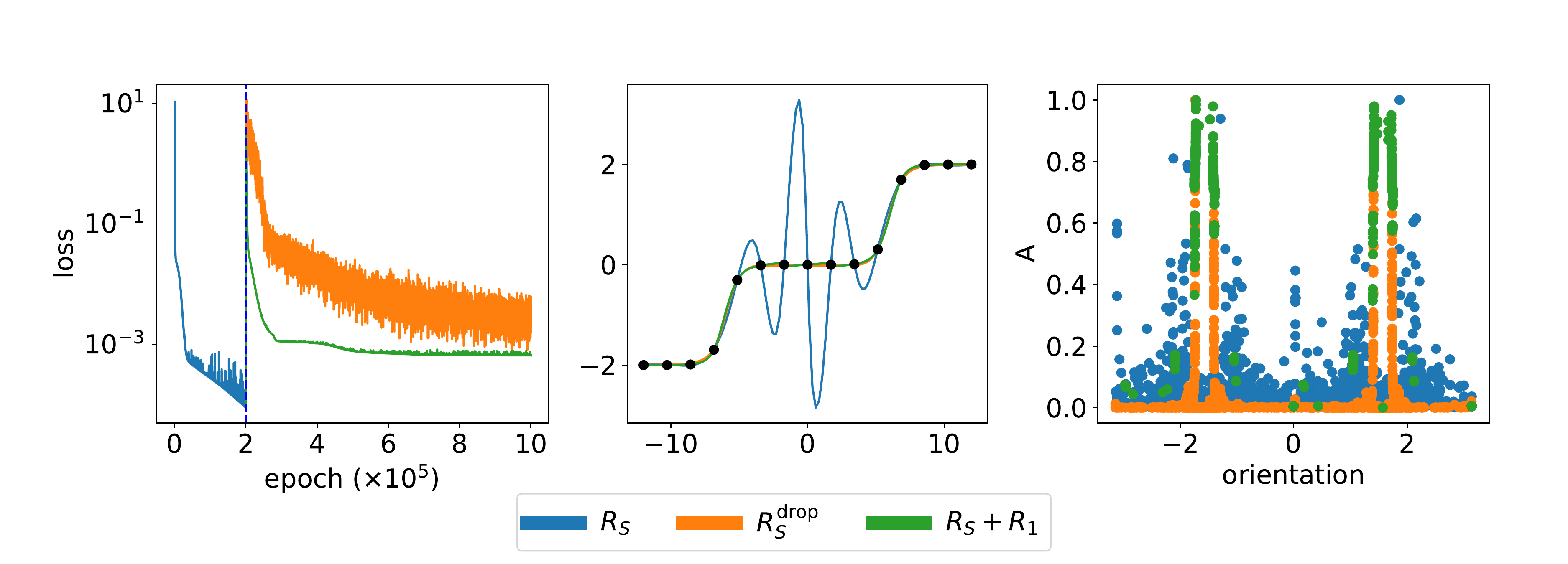}

  \caption{Compared dynamics are initialized at model found by $R_S(\vtheta)$, marked by the vertical dashed line in iteration 200000 with two-layer tanh NN. Left: The loss trajectory under different loss functions. MIddle: The output of the model trained by $R_S(\vtheta)$ (blue) and the model trained by $R_S^{\mathrm{drop}}(\vtheta)$ (orange) and $R_S(\vtheta)+R_1(\vtheta)$ (green) initialized at model found by $R_S(\vtheta)$. The black points are the target points. Right: The feature of the model trained by $R_S(\vtheta)$ (blue) and the model trained by $R_S^{\mathrm{drop}}(\vtheta)$ (orange) and $R_S(\vtheta)+R_1(\vtheta)$ (green) initialized at model found by $R_S(\vtheta)$.\label{pic:add_dropout_tanh}}
\end{figure}

One may wonder if any noise injected into the training process could lead to condensation. We also perform similar experiments for SGD. As shown in Fig. \ref{pic:sgd_condense_tanh}, no significant condensation occurs even in the presence of noise during training. Therefore, the experiments in this section reveal the special characteristic of dropout that facilitate condensation.

\begin{figure}[h]
	\centering
	\subfloat[batch size $=2$, output]{\includegraphics[width=0.3\textwidth]{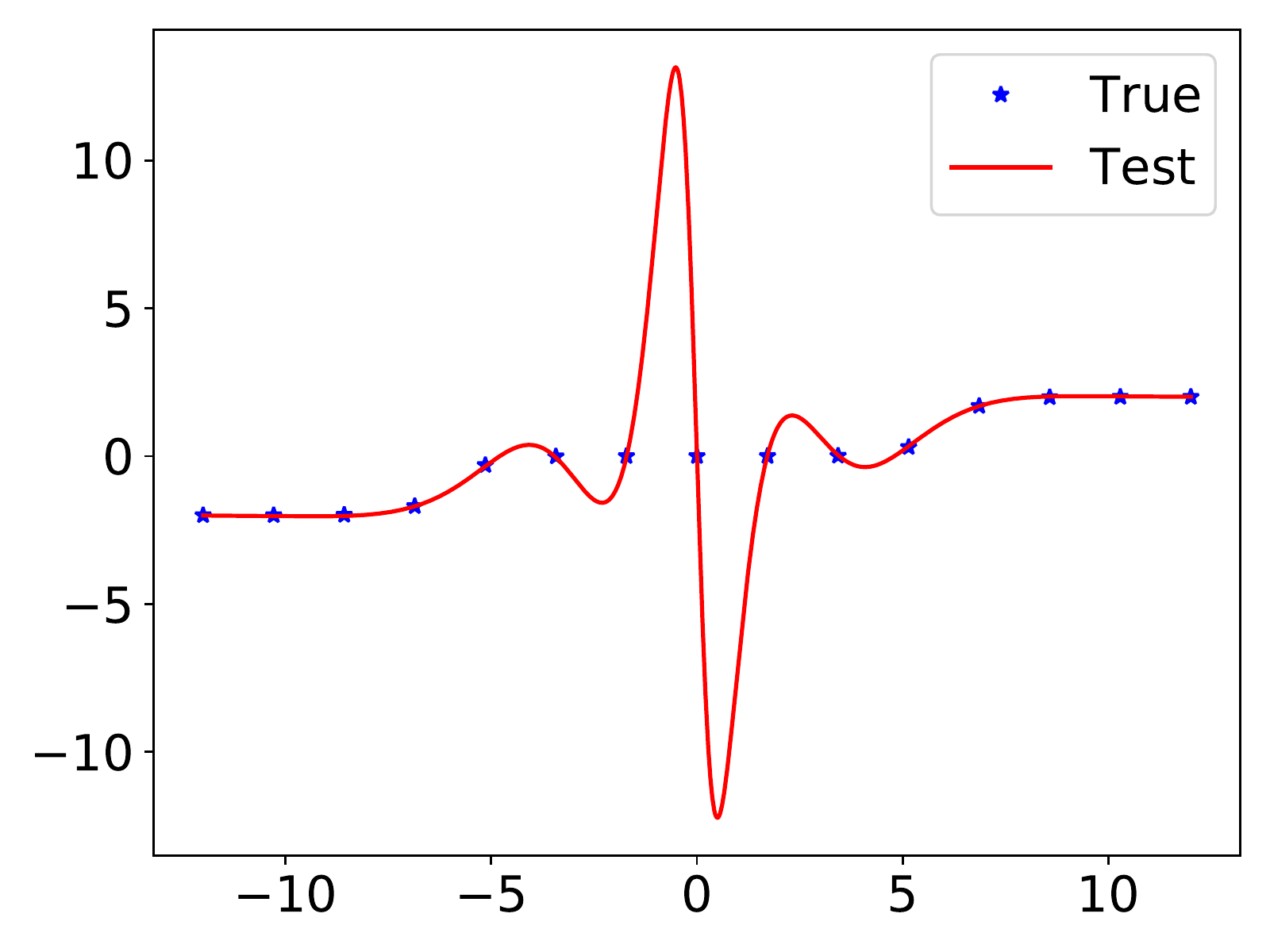}}
	\subfloat[batch size $=2$, feature]{\includegraphics[width=0.3\textwidth]{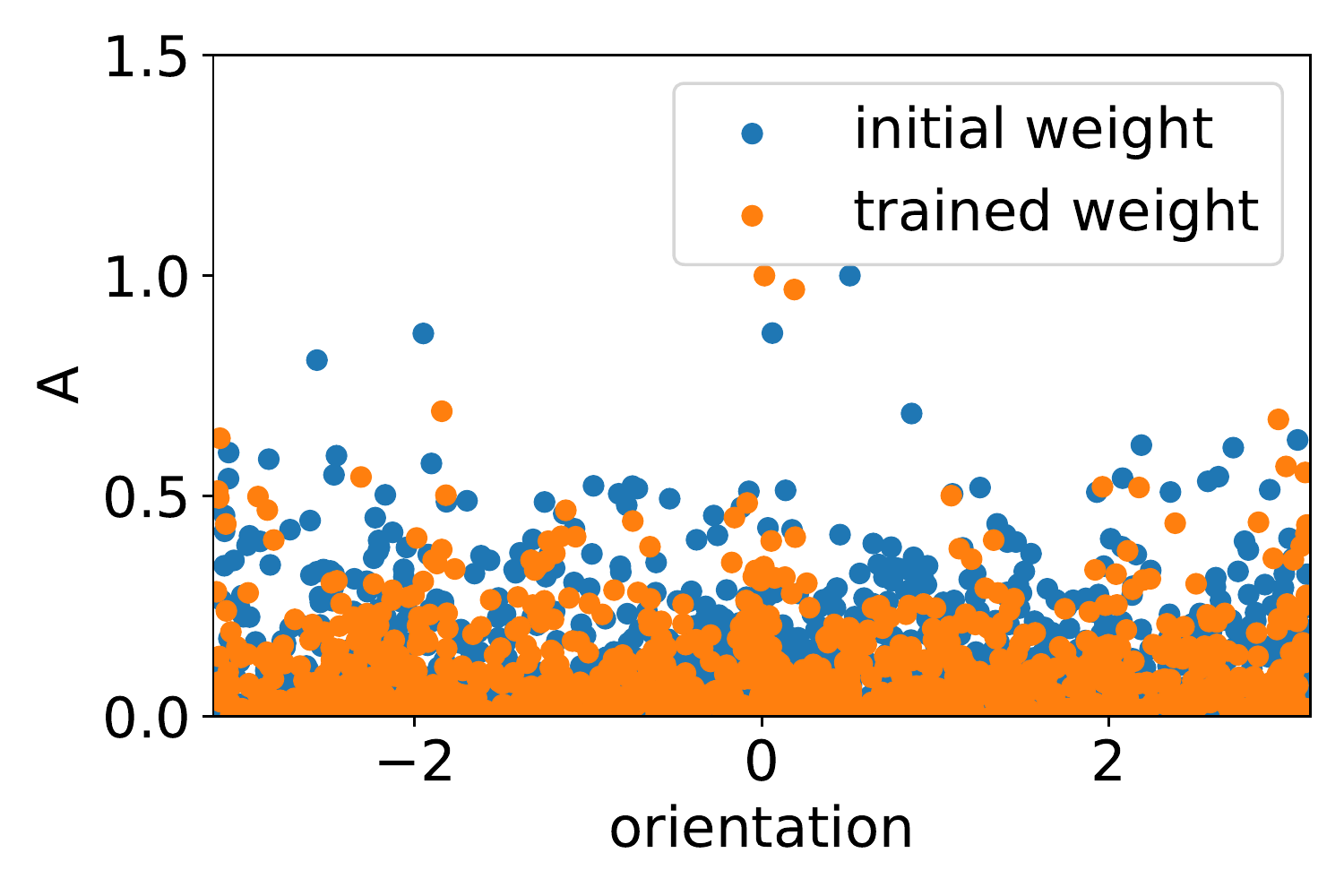}}

  \caption{Two-layer tanh NN output and feature with a batch size of 2. The width of the hidden layer is $1000$, and the learning rate is $1\times10^{-3}$. In (b), blue dots and orange dots are for the weight feature distribution at the initial and final training stages, respectively. \label{pic:sgd_condense_tanh}}
\end{figure} 
\subsubsection{Network with High-dimensional Input}
We conducted further investigation into the effect of dropout on high-dimensional two-layer tanh NNs under the teacher-student setting. Specifically, we utilize a two-layer tanh NN with only one hidden neuron and 10-dimensional input as the target function. The orientation similarity of two neurons is calculated by taking the inner product of their normalized weights. As shown in Fig. \ref{fig:condense_nd}(a,b), for the NN with dropout, the neurons of the network have only two orientations, indicating the occurrence of condensation, while the NN without dropout does not exhibit such a phenomenon.

\begin{figure*}[h]
	\centering
	\subfloat[$p=0.5$, feature]{\includegraphics[width=0.24\textwidth]{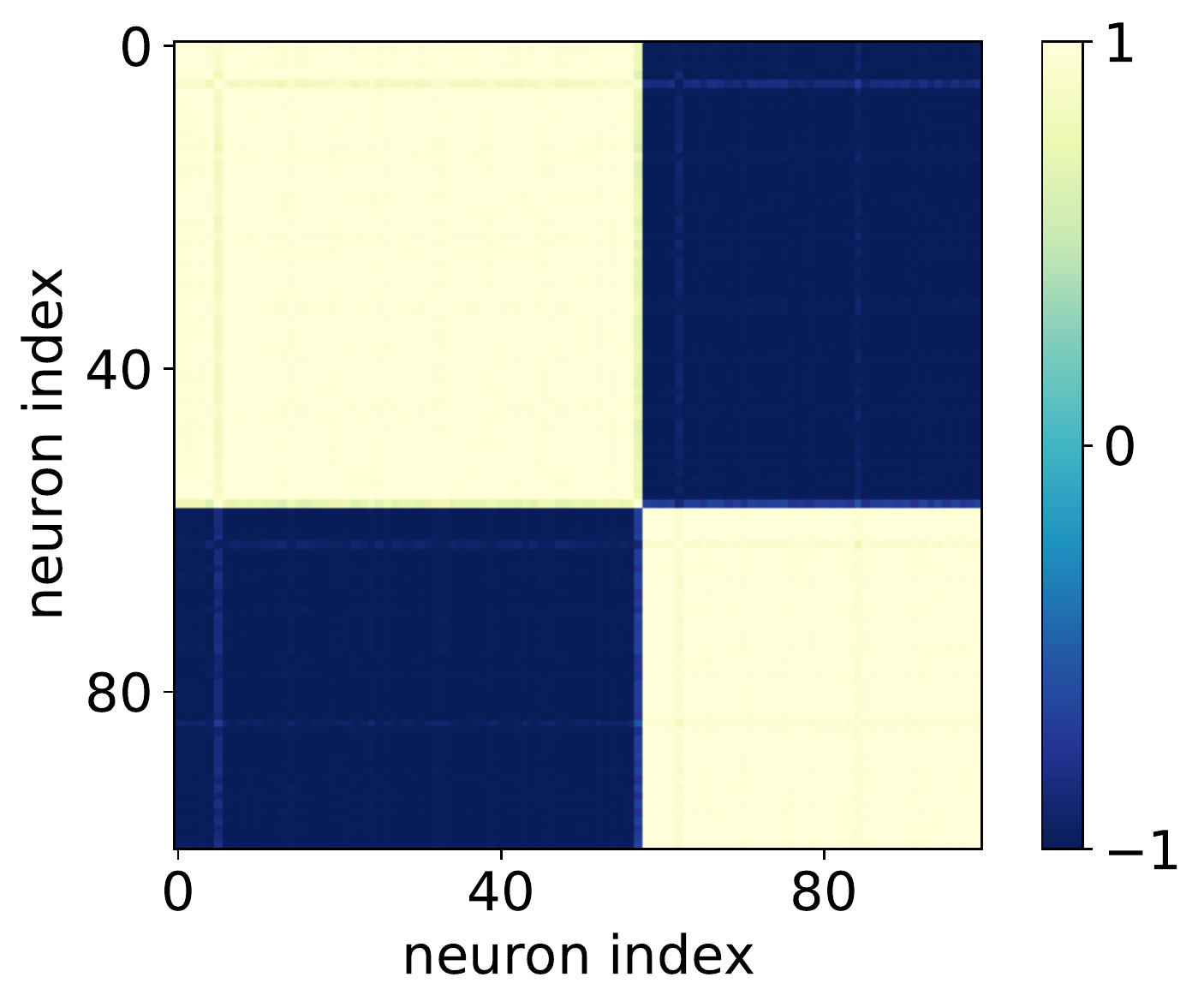}}
	\subfloat[$p=1$, feature]{\includegraphics[width=0.24\textwidth]{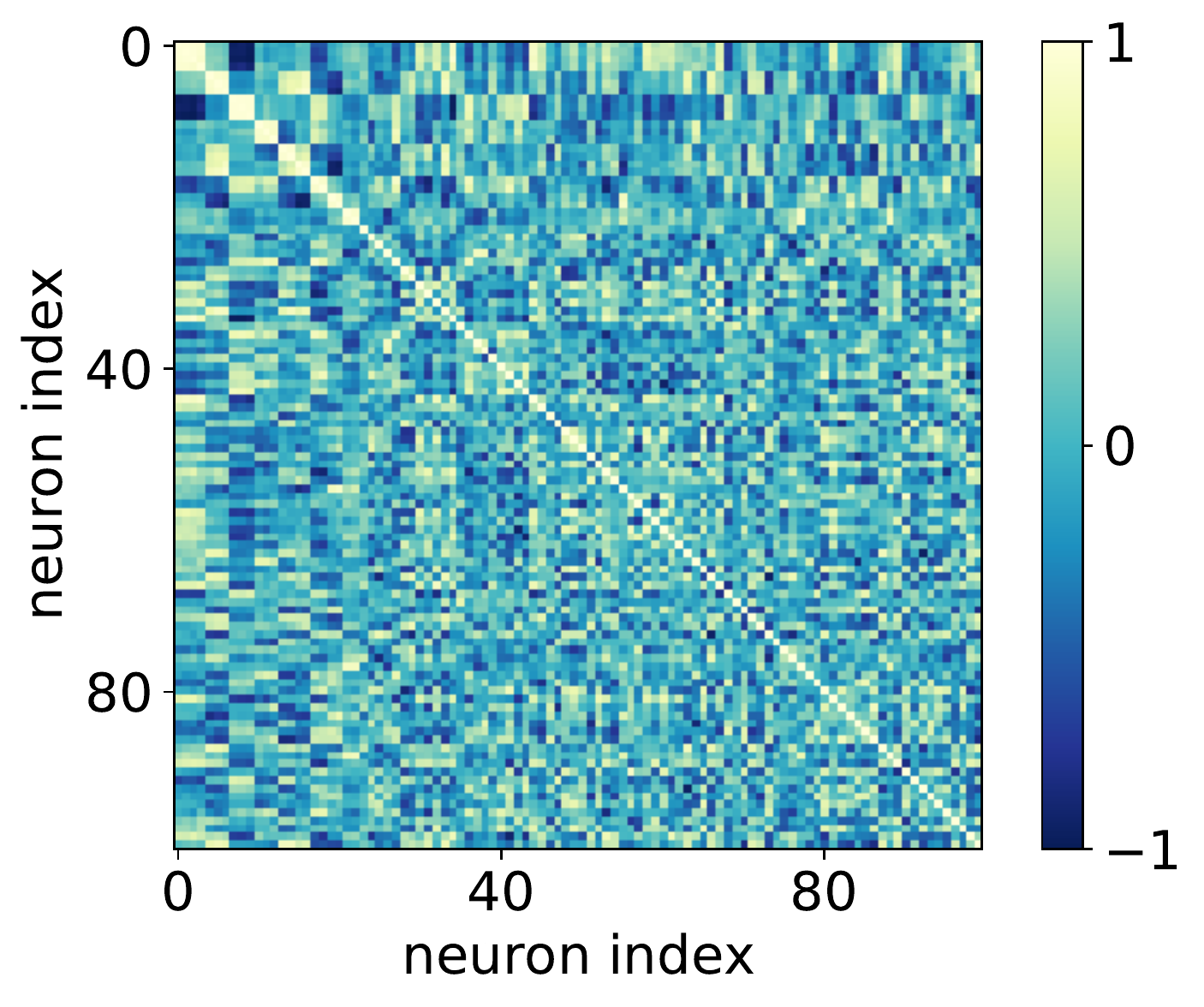}}
	\subfloat[effective ratio]{\includegraphics[width=0.27\textwidth]{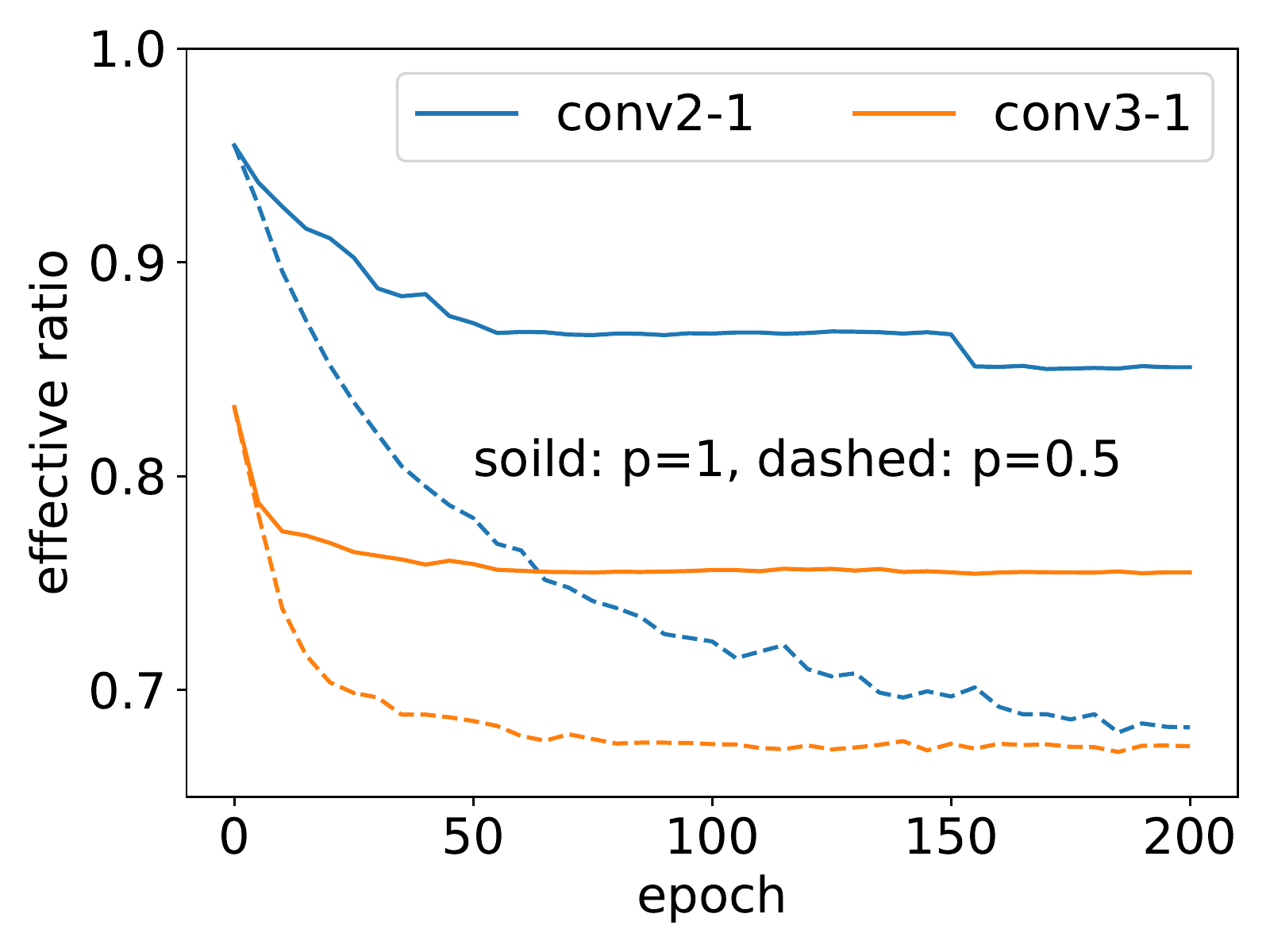}}
  \caption{Sparsity in High-Dimensional NNs with different dropout rates. (a, b) Parameter features of the two-layer tanh NNs with and without dropout. (c) Effective ratio with and without dropout under the task of CIFAR-10 classification with ResNet-18. Conv2-1 and conv3-1 represent the parameters of the first convolutional layer of the second block and the third block of the ResNet, respectively.}
  \label{fig:condense_nd}
\end{figure*} 
To visualize the condensation during the training process, we define the ratio of effective neurons as follows.


\begin{definition}[\textbf{effective ratio}]
For a given NN, the input weight of neuron $j$ in the $l$-th layer is vectorized as $\vtheta_j^{[l]}\in\sR^{m_l}$. Let $U^{[l]}=\{\vu_k^{[l]}\}_{k=1}^{m_l}$ be the set of vectors such that for any $\vtheta_{j}^{[l]}$, there exists an element $\vu\in U^{[l]}$ satisfying $\vu\cdot\vtheta_{j}^{[l]}>0.95$. The effective neuron number $m^{\rm eff}_{l}$ of the $l$-th layer is defined as the minimal size of all possible $U^{[l]}$. The effective ratio is defined as $m^{\rm eff}_{l}/m_{l}$.
\end{definition}

We study the training process of using ResNet-18 to learn CIFAR-10. As shown in Fig. \ref{fig:condense_nd}(c), NNs with dropout tend to have lower effective ratios, and thus tend to exhibit condensation.

\subsubsection{Dropout Improves Generalization}
As the effective neuron number of a condensed network is much smaller than its actual neuron number, it is expected to generalize better. To verify this, we use a two-layer tanh network with $1000$ neurons to learn a teacher two-layer tanh network with two neurons. The number of free parameters in the teacher network is $6$. As shown in Fig. \ref{pic:condense_heat_tanh}, the model with dropout generalizes well when the number of samplings is larger than 6, while the model without dropout generalizes badly. This result is consistent with the rank analysis of non-linear models \cite{zhang2022linear}.


\begin{figure}[!t]
	\centering
	\includegraphics[width=0.6\textwidth]{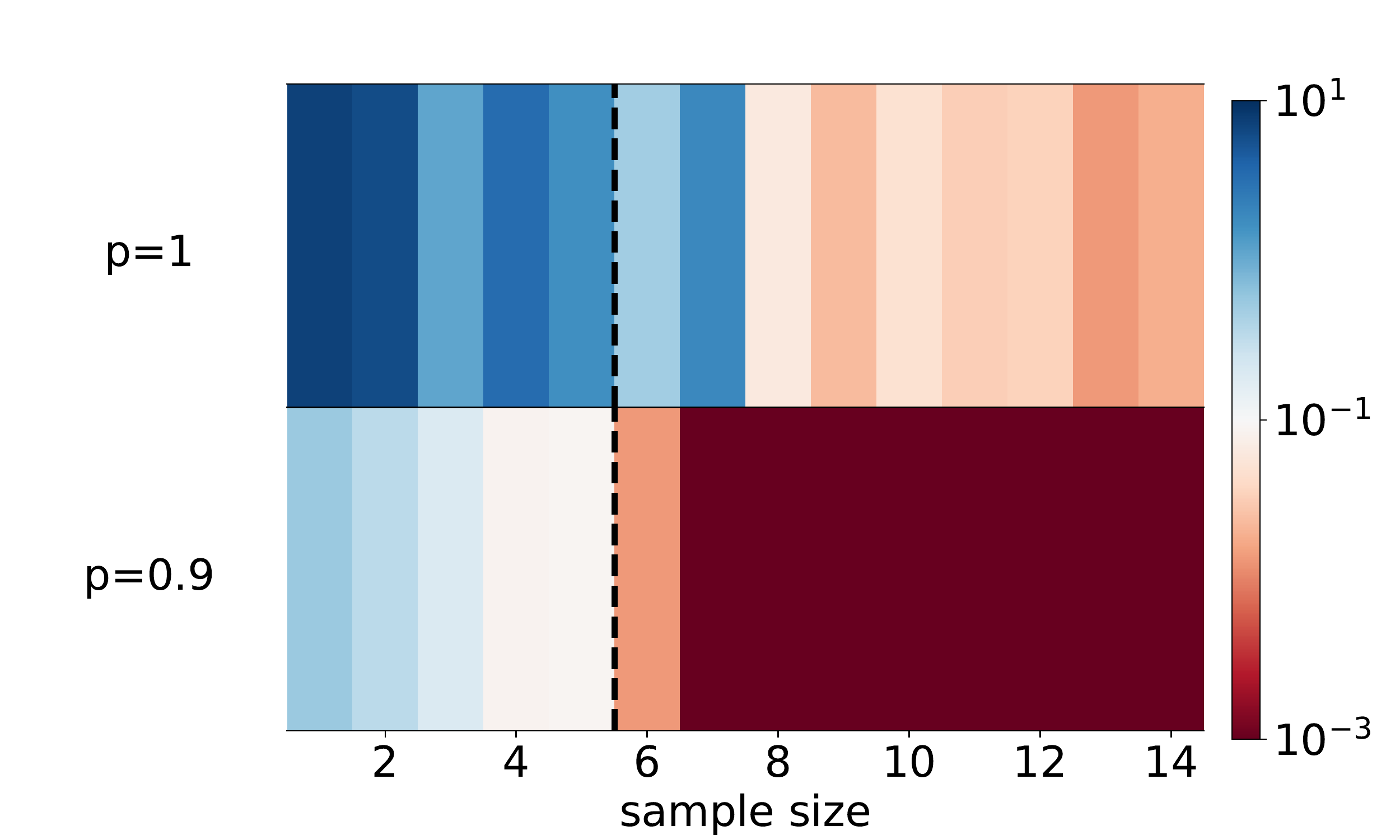}

  \caption{Average test error of the two-layer tanh NNs (color) vs. the number of samples (abscissa) for different dropout rates (ordinate). For all experiments, the width of the hidden layer is $1000$, and the learning rate is $1\times10^{-4}$ with the Adam optimizer. Each test error is averaged over $10$ trials with random initialization. Refer to Appendix \ref{app:relu_further} for further experiments on
ReLU NNs. \label{pic:condense_heat_tanh}}
\end{figure}

\subsection{The Effect of $\texorpdfstring{R_1(\vtheta)}{E=energy, m=mass, c=speed\ of\ light}$ on Condensation}
As can be seen from the implicit regularization term $R_1(\vtheta)$, dropout regularization imposes an additional $l_2$-norm constraint on the output of each neuron. The constraint has an effect on condensation. We illustrate the effect of $R_1(\vtheta)$ by a toy example of a two-layer ReLU network.

We use the following two-layer ReLU network to fit a one-dimensional function:

\begin{equation*}
	f_{\vtheta}(x)=\sum_{j=1}^{m} a_j \sigma(\vw_j \cdot \vx)=\sum_{j=1}^{m} a_j \sigma(w_j x + b_j),
\end{equation*}
where $\vx:=(x,1)^{\T}\in\sR^{2}$, $\vw_j:=(w_j ,b_j )\in\sR^{2}$, $\sigma(x)=\ReLU(x)$. For simplicity, we set $m=2$, and suppose the network can perfectly fit a training data set of two data points generated by a target function of $\sigma(\vw^{*}\cdot\vx)$, denoted as $\vo^{*}:=(\sigma(\vw^{*}\cdot\vx_1),\sigma(\vw^{*}\cdot\vx_2))$. We further assume $\vw^{*}\cdot\vx_{i}>0, i=1,2$. Denote the output of the $j$-th neuron over samples as 
\begin{equation*}
	\vo_{j} = (a_j  \sigma(\vw_j \cdot \vx_1),a_j  \sigma(\vw_j  \cdot\vx_2)).
\end{equation*}
The network output should equal to the target on the training data points after long enough training, i.e., 
\begin{equation*}
	\vo^{*}=\vo_1+\vo_2.
\end{equation*}
There are infinitely many pairs of $\vo_1$ and $\vo_2$ that can fit $\vo^{*}$ well. However, the $R_1(\vtheta)$ term leads the training to a specific pair. $R_1(\vtheta)$ can be written as 
\begin{equation*}
	R_1(\vtheta) = \norm{\vo_1}^2+\norm{\vo_2}^2,
\end{equation*}
and the components of $\vo_{j}$ perpendicular to  $\vo^{*}$ need to cancel each other at the well-trained stage to minimize $R_1(\vtheta)$. As a result, $\vo_{1}$ and $\vo_{2}$ need to be parallel with $\vo^{*}$, i.e., $\vw_1  // \vw_2  // \vw^{*} $, which is the condensation phenomenon.

In the following, we show that minimizing $R_1(\vtheta)$ term can lead to condensation under several settings. We first give some definitions that capture the characteristic of ReLU neurons (also shown in Fig. \ref{pic:relu}).

\begin{figure}[h]
	\centering
	\includegraphics[width=0.8\textwidth]{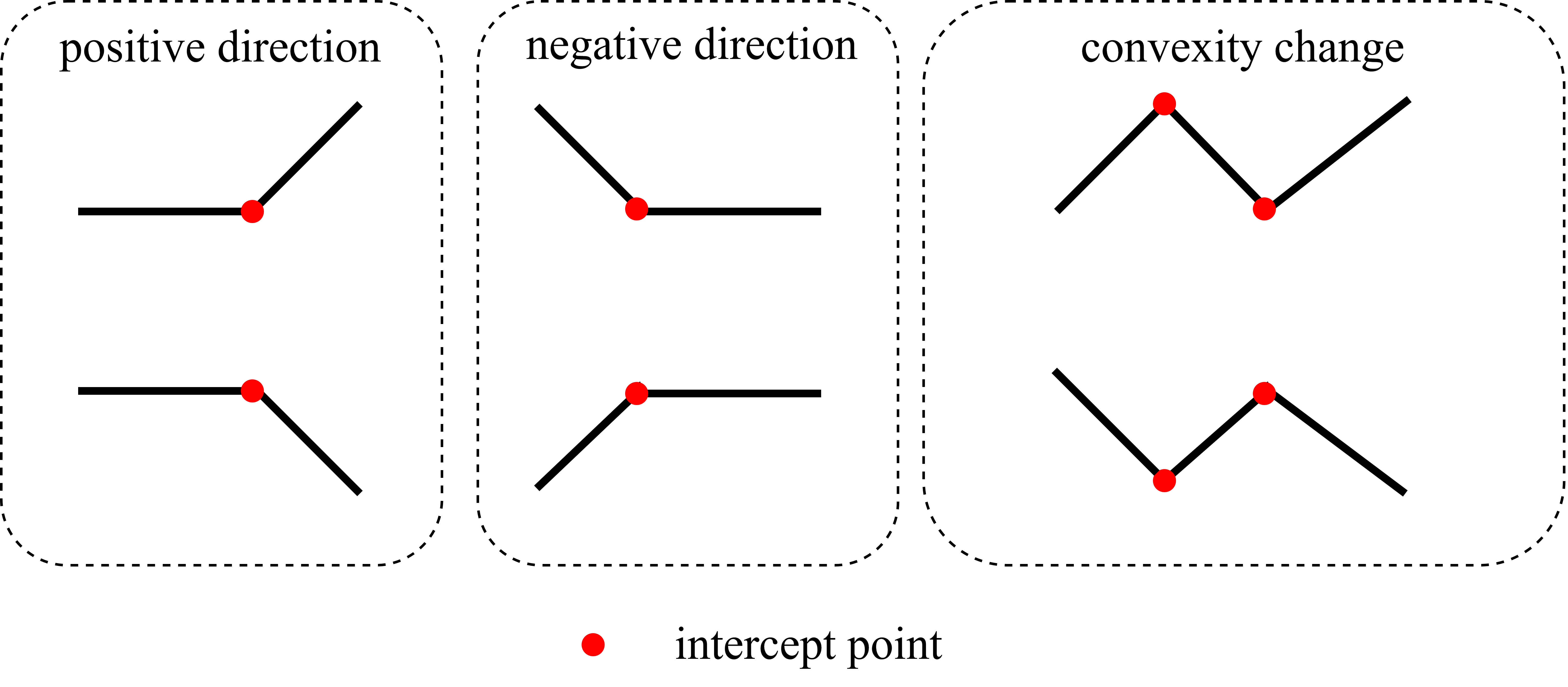}

  \caption{Schematic diagram of some definitions associated with ReLU neurons. \label{pic:relu}}
\end{figure} 

\begin{definition}[\textbf{convexity change of ReLU NNs}]
Consider piecewise linear function $f(t)$, $t\in \sR$, and its linear interval sets $\{[t_{i}, t_{i+1}]\}_{i=1}^{T}$. For any two intervals $[t_{i}, t_{i+2}], [t_{i+1}, t_{i+3}], i \in [T-3]$, if on one of the intervals, $f$ is convex and on the other $f$ is concave, then we call there exists a convexity change.
\end{definition}

\begin{definition}[\textbf{direction and intercept point of ReLU neurons}]
    For a one-dimensional ReLU neuron $a_{j}\sigma(w_j x+b_j)$, its direction is defined as $\mathrm{sign}(w_j)$, and its intercept point is defined as $x=-\frac{b_j}{w_j}$.
\end{definition}


Drawing inspiration from the methodology employed to establish the regularization effect of label noise SGD  \cite{blanc2020implicit}, we show that under the setting of two-layer ReLU NN and one-dimensional input data, the implicit bias of $R_1(\vtheta)$ term corresponds to ``simple'' functions that satisfy two conditions: (I) they have the minimum number of convexity changes required to fit the training points, and (ii) if the intercept points of neurons are in the same inner interval, and the neurons have the same direction, then their intercept points are identical.

\begin{theorem}[\textbf{the effect of $R_1(\vtheta)$ on facilitating condensation}]\label{thm:perb} Consider the following two-layer ReLU NN, 
\begin{equation*}
	f_{\vtheta}(x)=\sum_{j=1}^{m} a_j \sigma(w_j x + b_j)+a x+b, 
\end{equation*}
trained with a one-dimensional dataset $S=\{(x_{i}, y_{i})\}_{i=1}^n$, where $x_1<x_2< \cdots <x_n$. When the MSE of training data $R_S(\vtheta)=0$, if any of the following two conditions holds:

(i) the number of convexity changes of NN in $(x_ {1}, x_ {n})$ can be reduced while $R_S(\vtheta)=0$;

(ii) there exist two neurons with indexes $k_1 \neq k_2$, such that they have the same sign, i.e.,  $\mathrm{sign}(w_{k_1})=\mathrm{sign}(w_{k_2})$, and different intercept points in the same interval, i.e., $-{b_{k_1}}/{w_{k_1}}, -{b_{k_2}}/{w_{k_2}}\in [x_ {i}, x_ {i+1}]$, and $-{b_{k_1}}/{w_{k_1}}\neq -{b_{k_2}}/{w_{k_2}}$ for some $i \in [2: n-1]$;

then there exists parameters $\vtheta^{\prime}$, an infinitesimal perturbation of $\vtheta$, s.t.,

(i) $R_S(\vtheta^{\prime})=0$;

(ii) $R_1(\vtheta^{\prime})<R_1(\vtheta)$.

\end{theorem}

It should be noted that not all functions trained with dropout exhibit obvious condensation. For example, a function with dropout shows no condensation while training with a dataset consisting of only one data point. However, for general datasets, such as the example shown in Fig. \ref{pic:condense_toy} and Fig. \ref{pic:condense_tanh}, NNs reach the condensed solution due to the limit of the convexity changes and the intercept points (also illustrated in Fig. \ref{pic:thm_ill}). 

\begin{figure}[h]
	\centering
	\includegraphics[width=0.8\textwidth]{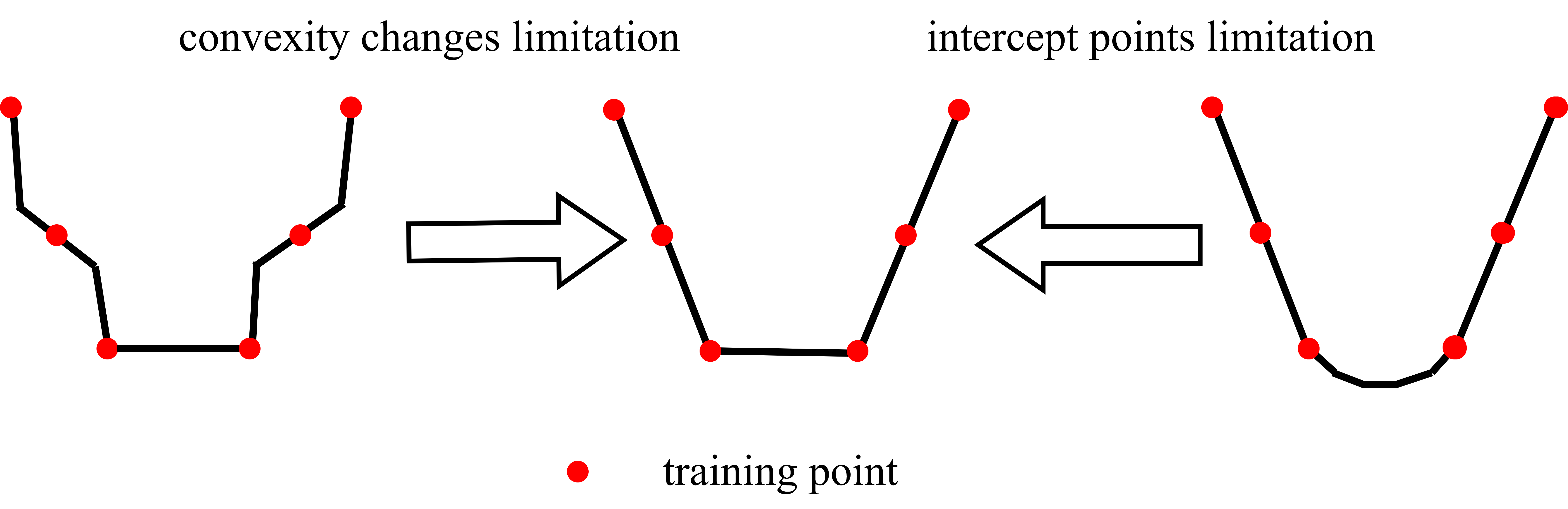}

  \caption{Schematic diagram of the effect of $R_1(\vtheta)$ on the limit of the convexity changes and the intercept points. \label{pic:thm_ill}}
\end{figure}

Although the current study only demonstrates the result for ReLU NNs, it is expected that for general activation functions, such as tanh, the $R_1(\vtheta)$ term also has the effect on facilitating condensation, which is left for future work. This is also confirmed by the experimental results conducted on the tanh NNs above. Furthermore, it is believed that the linear term $ax+b$ utilized to ensure $R_S(\vtheta)=0$ in certain cases is not a fundamental requirement. Our experiments have confirmed that neural networks without the linear term also exhibit the condensation phenomenon. 


\section{Implicit Regularization of Dropout on the Flatness of Solution} \label{sec:flat}
Understanding the mechanism by which dropout improves the generalization of NNs is of great interest and significance. In this section, we study the flatness of the minima found by dropout as inspired by the study of SGD on generalization \cite{keskar2016large}. Our primary focus is to study the effect of $R_1(\vtheta)$ and $R_2(\vtheta)$ on the flatness of loss landscape and network generalization.

\subsection{Dropout Finds Flatter Minima}

We first study the effect of dropout on model flatness and generalization. For a fair comparison of the flatness between different models, we employ the approach used in \cite{li2017visualizing} as follows. To obtain a direction for a network with parameters $\vtheta$, we begin by producing a random Gaussian direction vector $\vd$ with dimensions compatible with $\vtheta$. Then, we normalize each filter in $\vd$ to have the same norm as the corresponding filter in $\vtheta$. For FNNs, each layer can be regarded as a filter, and the normalization process is equivalent to normalizing the layer, while for convolutional neural networks (CNNs), each convolution kernel may have multiple filters, and each filter is normalized individually. Thus, we obtain a normalized direction vector $\vd$ by replacing $\vd_{i, j}$ with $\frac{\vd_{i, j}}{\norm{\vd_{i, j}}}\norm{\vtheta_{i, j}}$, where $\vd_{i,j}$ and $\vtheta_{i, j}$ represent the $j$th filter of the $i$th layer of the random direction $\vd$ and the network parameters $\vtheta$, respectively. Here, $\norm{\cdot}$ denotes the Frobenius norm. 
It is crucial to note that $j$ refers to the filter index. We use the function $L(\alpha)=R_S\left(\vtheta+\alpha \vd\right)$ to characterize the loss landscape around the minima obtained with and without dropout layers. 

For all network structures shown in Fig. \ref{fig:flatness_cnn}, dropout improves the generalization of the network and finds flatter minima. In Fig. \ref{fig:flatness_cnn}(a, b), for both networks trained with and without dropout layers, the training loss values are all close to zero, but their flatness and generalization are still different. In Fig. \ref{fig:flatness_cnn}(c, d), due to the complexity of the dataset, i.e., CIFAR-100 and Multi30k, and network structures, i.e., ResNet-20 and transformer, networks with dropout do not achieve zero training error but the ones with dropout find flatter minima with much better generalization. The accuracy of different network structures is shown in Table \ref{tab:acc}.


\begin{table}[!t]
\centering
\caption{The effect of dropout on model accuracy. }
\label{tab:acc}
\begin{tabular}{|c|c|c|c|}
\hline
Structure & Dataset & With Dropout & Without Dropout \\
\hline
FNN & MNIST & 98.7\% & 98.1\% \\
\hline
VGG-9 & CIFAR-10 & 60.6\% & 59.2\% \\
\hline
ResNet-20 & CIFAR-100 & 54.7\% & 34.1\% \\
\hline
Transformer & Multi30k & 49.3\% & 34.7\% \\
\hline
\end{tabular}
\end{table}

\begin{figure}[h]
	\centering
	\subfloat[flatness of FNN]{\includegraphics[width=0.24\textwidth]{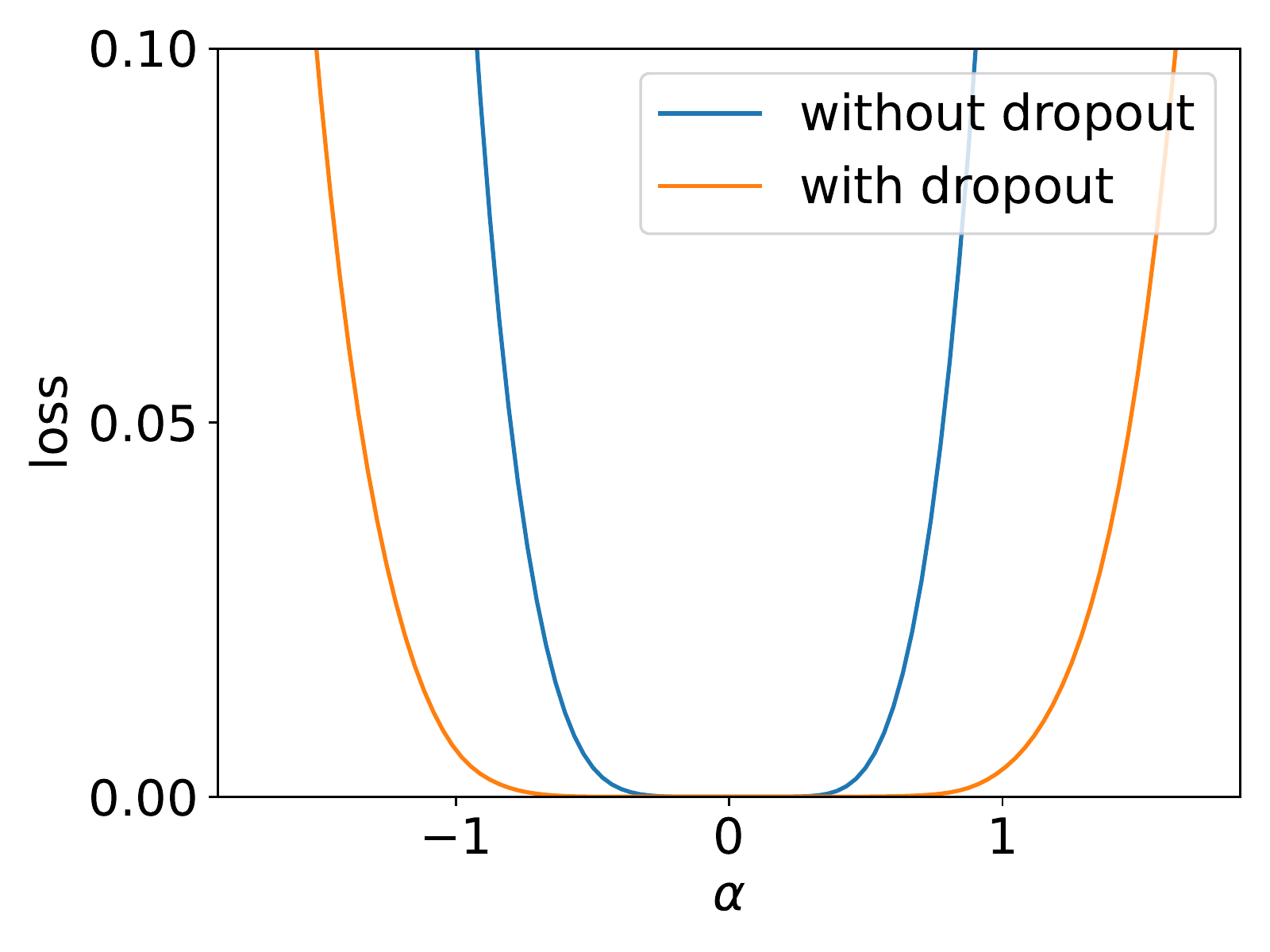}}
	\subfloat[flatness of VGG-9]{\includegraphics[width=0.24\textwidth]{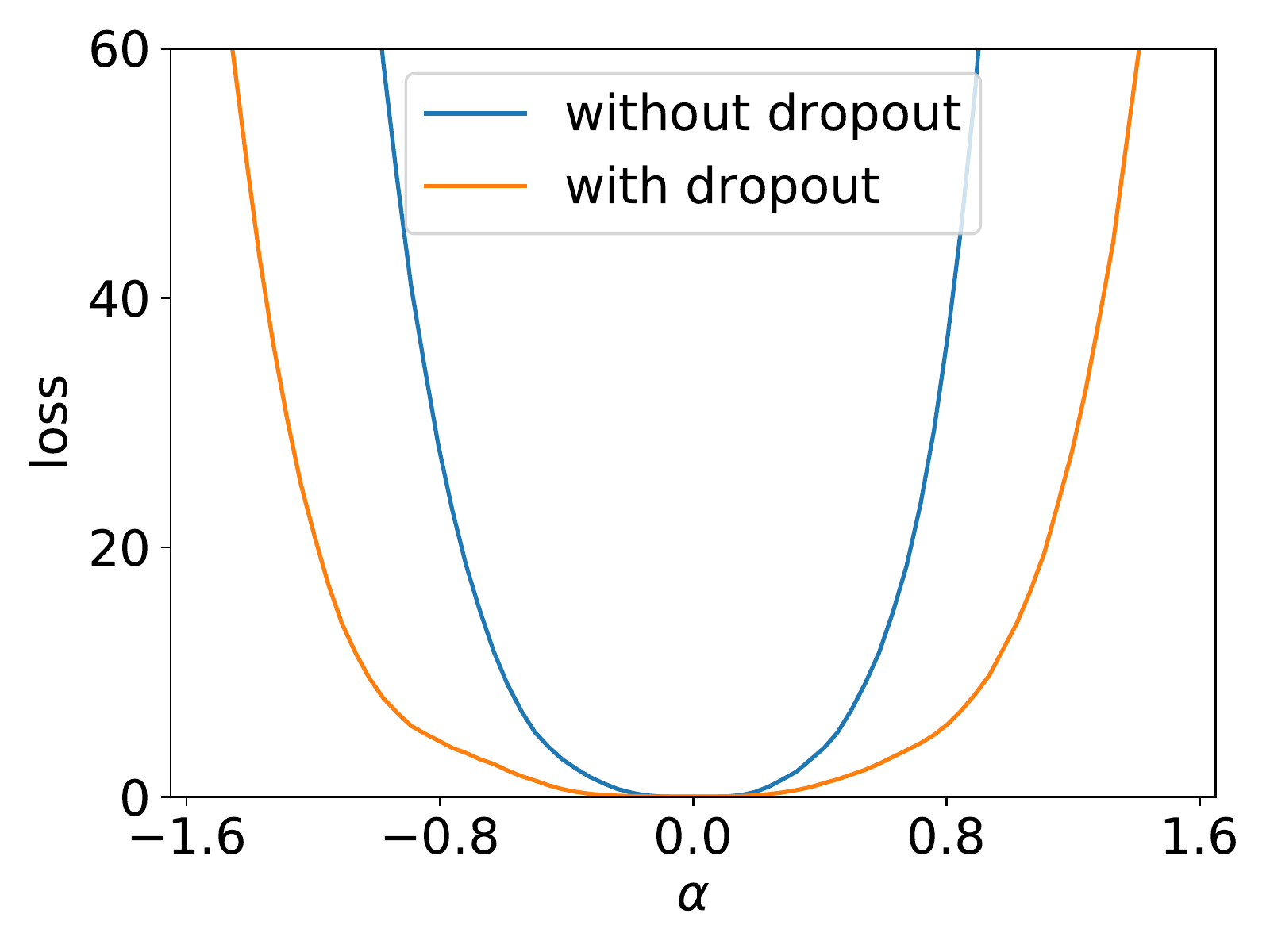}}
	\subfloat[flatness of ResNet-20]{\includegraphics[width=0.24\textwidth]{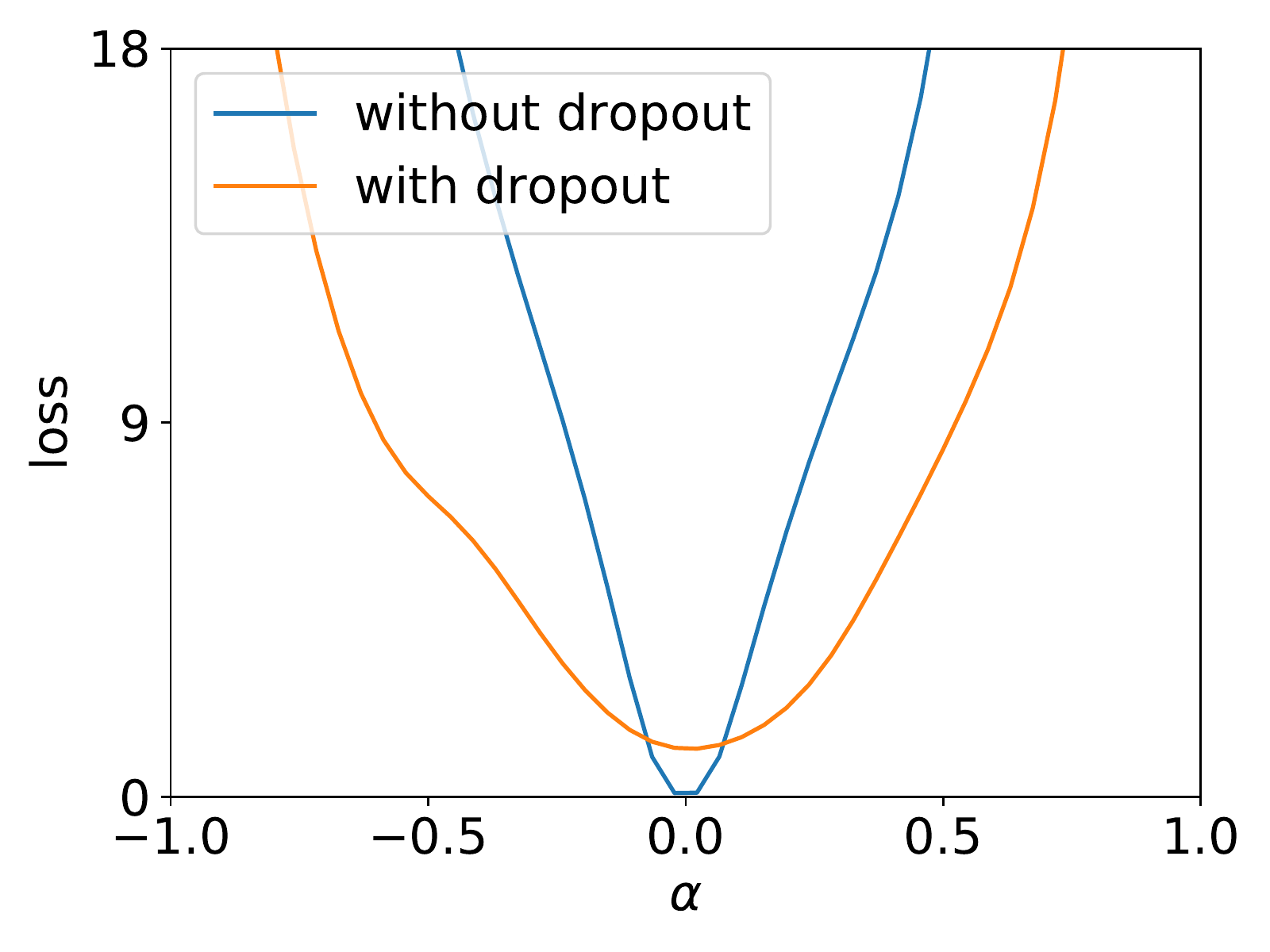}}
	\subfloat[flatness of transformer]{\includegraphics[width=0.24\textwidth]{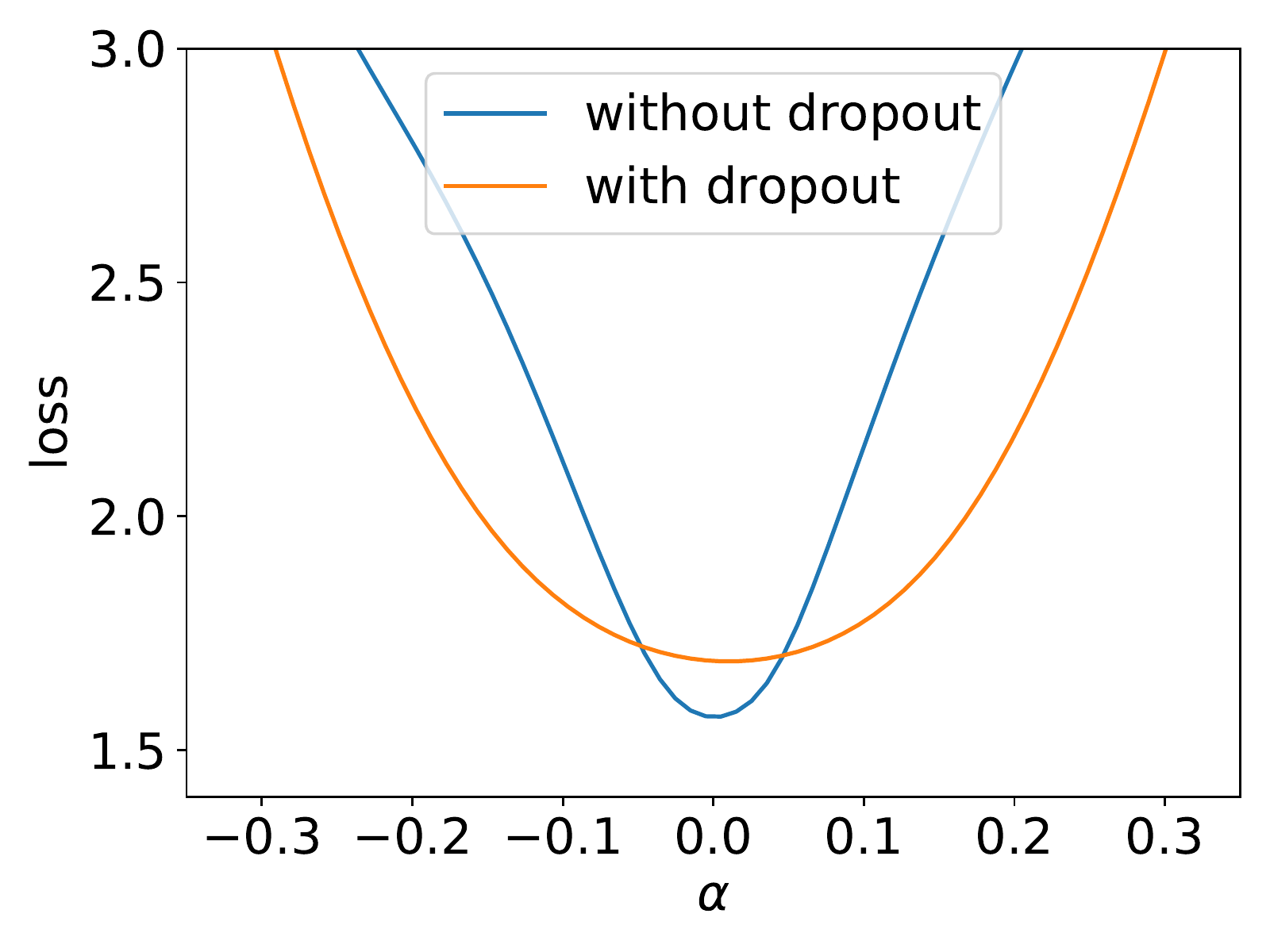}}
  \caption{The 1D visualization of solutions of different network structures obtained with or without dropout layers. (a) The FNN is trained on MNIST dataset. The test accuracy for the model with dropout layers is $98.7\%$ while $98.1\%$ for the model without dropout layers. (b) The VGG-9 network is trained on CIFAR-10 dataset using the first 2048 examples as the training dataset. The test accuracy for the model with dropout layers is $60.6\%$ while $59.2\%$ for the model without dropout layers. (c) The ResNet-20 network is trained on CIFAR-100 dataset using all examples as the training dataset. The test accuracy for the model with dropout layers is $54.7\%$ while $34.1\%$ for the model without dropout layers. (d) The transformer is trained on Multi30k dataset using the first 2048 examples as the training dataset. The test accuracy for the model with dropout layers is $49.3\%$ while $34.7\%$ for the model without dropout layers. \label{fig:flatness_cnn}}
\end{figure}

\subsection{The Effect of $\texorpdfstring{R_1(\vtheta)}{E=energy, m=mass, c=speed\ of\ light}$ on Flatness}

In this subsection, we study the effect of $R_1(\vtheta)$ on flatness under the two-layer ReLU NN setting. Different from the flatness described above by loss interpolation, we define the flatness of the minimum as the sum of the eigenvalues of the Hessian matrix $H$ in this section, i.e., $\mathrm{Tr}(H)$. Note that when $R_S(\vtheta)=0$, we have, 

\begin{equation*}
\begin{aligned}
        \mathrm{Tr}(H)&=\mathrm{Tr}\left(\frac{1}{n}\sum_{i=1}^n\nabla_{\vtheta}f_{\vtheta}(\vx_i)\nabla_{\vtheta}^{\T}f_{\vtheta}(\vx_i)\right)=\frac{1}{n}\sum_{i=1}^n\norm{\nabla_{\vtheta}f_{\vtheta}(\vx_i)}_2^2, 
\end{aligned}
\end{equation*}
thus the definition of flatness above is equivalent to $\frac{1}{n}\sum_{i=1}^n\norm{\nabla_{\vtheta}f_{\vtheta}(\vx_i)}_2^2$.

\begin{theorem}[\textbf{the effect of $R_1(\vtheta)$ on facilitating flatness}] Under the Setting \ref{set:1}--\ref{set:3}, consider a two-layer ReLU NN, 
\begin{equation*}
	f_{\vtheta}(x)=\sum_{j=1}^{m} a_j \sigma(\vw_j \vx), 
\end{equation*}
trained with dataset $S=\{(\vx_{i}, y_{i})\}_{i=1}^n$. Under the gradient flow training with the loss function $R_S(\vtheta)+R_1(\vtheta)$, if $\vtheta_0$ satisfying $R_S(\vtheta_0)=0$ and $\nabla_{\vtheta} R_1(\vtheta_0)\neq 0$, we have 
\begin{equation*}
    \frac{{\rm d}\left(\frac{1}{n}\sum_{i=1}^n\norm{\nabla_{\vtheta}f_{\vtheta_0}(\vx_i)}_2^2\right)}{\rm dt}<0.
\end{equation*}

\end{theorem}

The regularization effect of $R_2(\vtheta)$ also has a positive effect on flatness by constraining the norm of the gradient. In the next subsection, we compare the effect of these two regularization terms on generalization and flatness.

\subsection{Effect of Two Implicit Regularization Terms on Generalization and Flatness}
Although the modified gradient flow is noise-free during training, the model trained with the modified gradient flow can also find a flat minimum that generalizes well, due to the effect of $R_1(\vtheta)$ and $R_2(\vtheta)$. However, the magnitude of their impact on flatness is not yet fully understood. In this subsection, we study the effect of each regularization term through training networks by the following four loss functions:

\begin{equation}
    \begin{aligned}
        L_1(\boldsymbol{\vtheta}) &:=\RS(\boldsymbol{\vtheta})+R_{1}(\vtheta)\\
        &:=\RS(\boldsymbol{\vtheta})+\frac{1-p}{2np} \sum_{i=1}^{n}\sum_{j=1}^{m_{L-1}}\Vert\mW^{[L]}_{j}f^{[L-1]}_{\vtheta, j}(\vx_i)\Vert^2, \\
        L_2(\boldsymbol{\vtheta}, \boldsymbol{\veta})&:=\RS(\boldsymbol{\vtheta})+\tilde{R}_2(\vtheta,\veta)\\
        &:=\RS(\boldsymbol{\vtheta})+\frac{\varepsilon}{4} \left\|\nabla_{\boldsymbol{\vtheta}} \RS^\mathrm{drop}\left(\boldsymbol{\vtheta}, \boldsymbol{\veta}\right)\right\|^{2},\\
        L_3(\boldsymbol{\vtheta}, \boldsymbol{\veta})&:=\RS^\mathrm{drop}(\boldsymbol{\vtheta}, \boldsymbol{\veta})-\tilde{R}_2(\vtheta,\veta)\\
        &:=\RS^\mathrm{drop}(\boldsymbol{\vtheta}, \boldsymbol{\veta})-\frac{\varepsilon}{4} \left\|\nabla_{\boldsymbol{\vtheta}} R_{S}^\mathrm{drop}\left(\boldsymbol{\vtheta}, \boldsymbol{\veta}\right)\right\|^{2},\\
        L_4(\boldsymbol{\vtheta}, \boldsymbol{\veta})&:=\RS^\mathrm{drop}(\boldsymbol{\vtheta}, \boldsymbol{\veta})-R_{1}(\vtheta)\\
        &:=\RS^\mathrm{drop}(\boldsymbol{\vtheta}, \boldsymbol{\veta})-\frac{1-p}{2np} \sum_{i=1}^{n}\sum_{j=1}^{m_{L-1}}\Vert\mW^{[L]}_{j}f^{[L-1]}_{\vtheta, j}(\vx_i)\Vert^2, 
    \end{aligned}    
\end{equation}
where $\tilde{R}_2(\vtheta,\veta)$ is defined as $(\varepsilon/4) \left\|\nabla_{\boldsymbol{\vtheta}} \RS^\mathrm{drop}\left(\boldsymbol{\vtheta}, \boldsymbol{\veta}\right)\right\|^{2}$ for convenience, and we have  $\Exp_{\veta}\tilde{R}_2(\vtheta,\veta)=R_2(\vtheta)$.  For each $L_i, i \in [4]$, we explicitly add or subtract the penalty term of either $R_1(\vtheta)$ or $\tilde{R}_2(\vtheta,\veta)$ to study their effect on dropout regularization. Therefore, $L_1(\vtheta)$ and $L_{3}(\vtheta,\veta)$ are used to study the effect of $R_1(\vtheta)$, while $L_2(\vtheta,\veta)$ and $L_{4}(\vtheta,\veta)$ are for $R_2(\vtheta)$.

We first study the effect of two regularization terms on the generalization of NNs. As shown in Fig. \ref{fig:compar1}, we compare the test accuracy obtained by training with the above four distinct loss functions under different dropout rates and utilize the results of $\RS(\boldsymbol{\vtheta})$ and $\RS^\mathrm{drop}(\boldsymbol{\vtheta}, \boldsymbol{\veta})$ as reference benchmarks. Two different learning rates are considered, with the solid and dashed lines corresponding to $\varepsilon=0.05$ and $\varepsilon=0.005$, respectively. As shown in Fig. \ref{fig:compar1}(a), both approaches show that the training with the $R_1(\vtheta)$ regularization term finds a solution that almost has the same test accuracy as the training with dropout. For $\tilde{R}_2(\vtheta,\veta)$, as shown in Fig. \ref{fig:compar1}(b), the effect of $\tilde{R}_2(\vtheta,\veta)$ only marginally improves the generalization ability of full-batch gradient descent training in comparison to the utilization of $R_1(\vtheta)$.


\begin{figure}[h]
	\centering
	\subfloat[]{\includegraphics[width=0.4\textwidth]{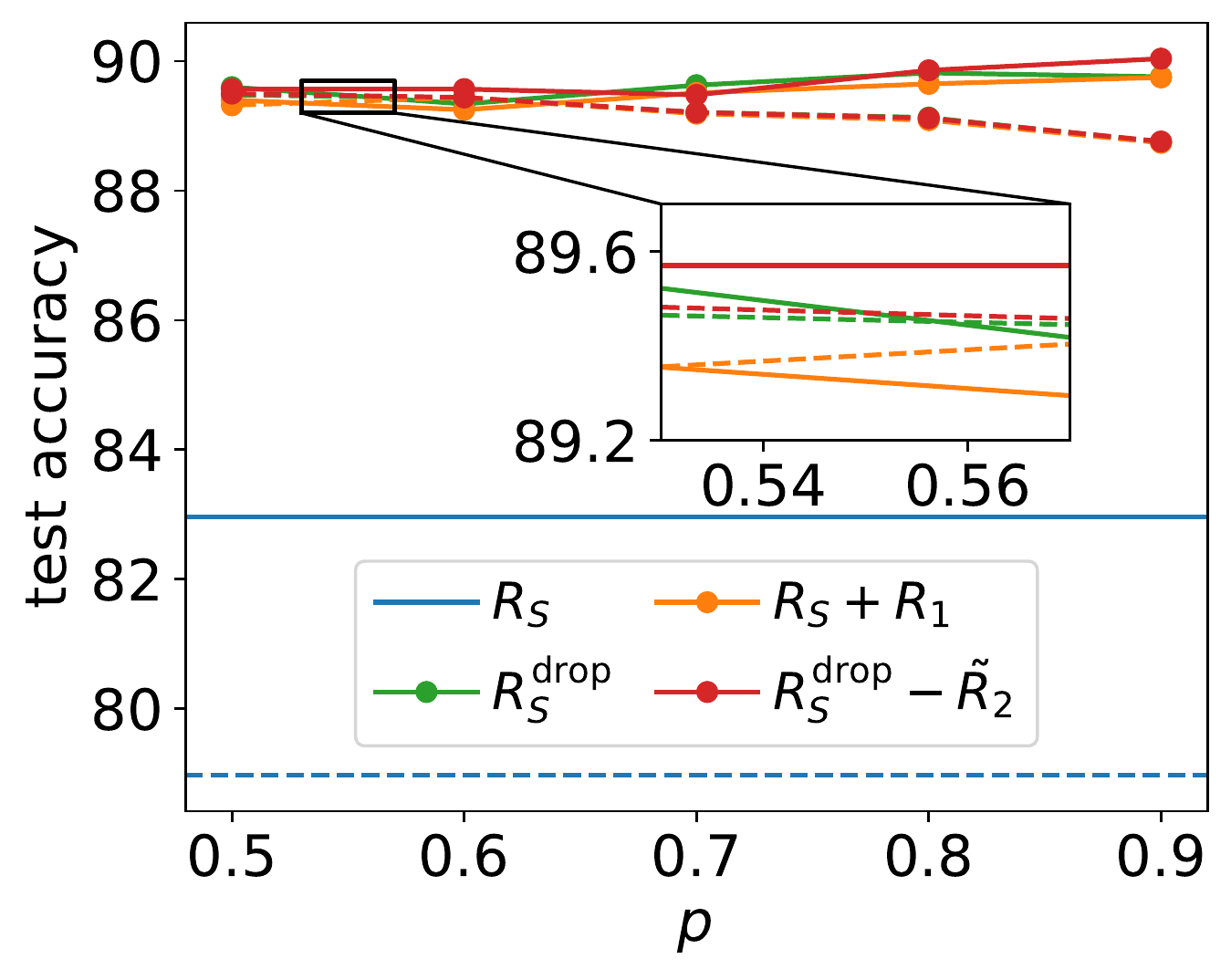}}
    \subfloat[]{\includegraphics[width=0.4\textwidth]{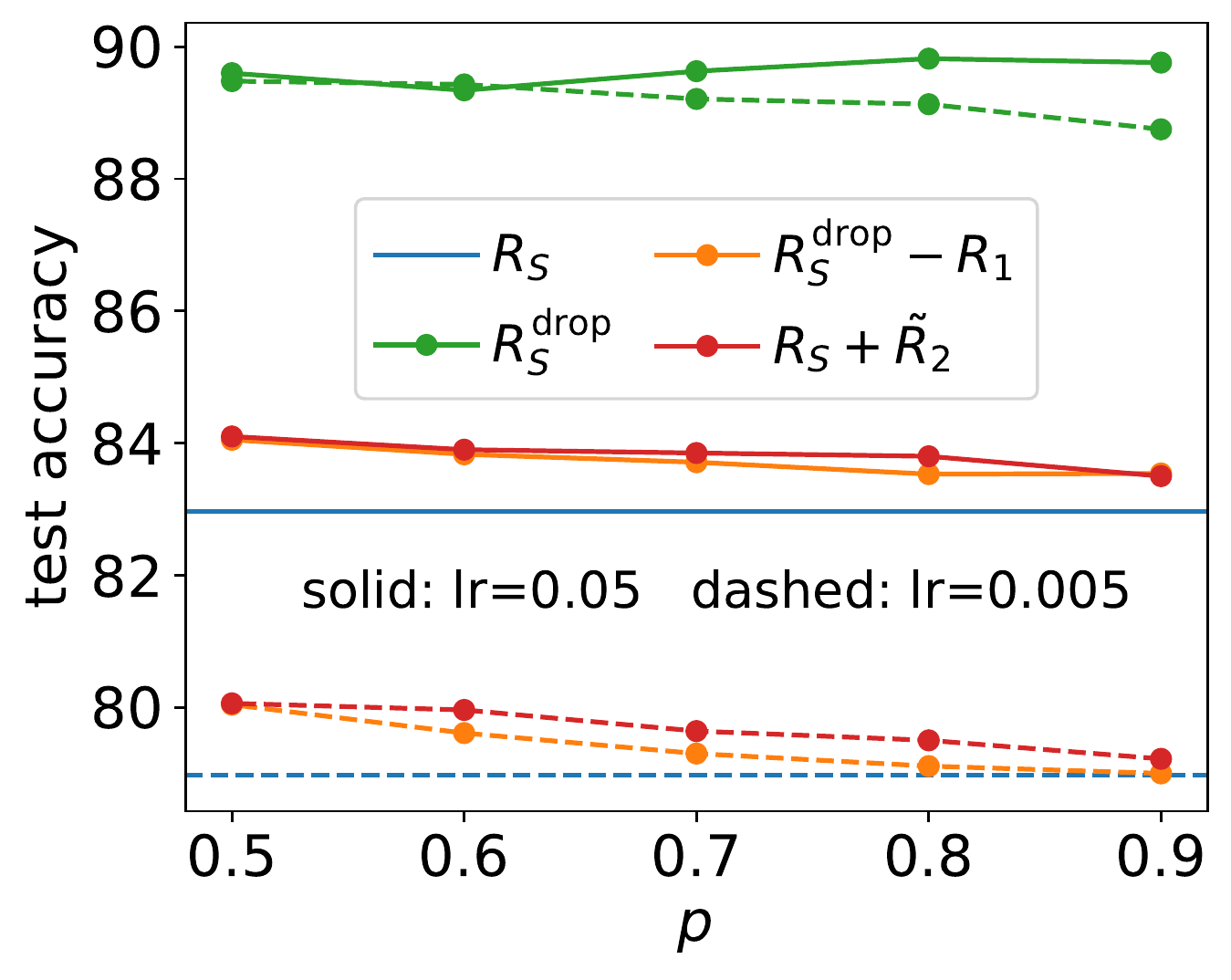}}
    \caption{The classification task on the MNIST dataset (the first 1000 images) using the FNN with size $784$-$1000$-$10$. The test accuracy obtained by training with $L_1(\vtheta), \cdots, L_4(\vtheta)$ and $\RS(\boldsymbol{\vtheta})$, $\RS^\mathrm{drop}(\boldsymbol{\vtheta}, \boldsymbol{\veta})$ under different dropout rates and learning rates. The solid line represents the test accuracy of the network under a large learning rate ($\varepsilon=0.05$), and the dashed line represents the test accuracy of the network under a small learning rate ($\varepsilon=0.005$). \label{fig:compar1}}
\end{figure} 

Then we study the effect of two regularization terms on flatness. To this end, we show a one-dimensional cross-section of the loss $R_S(\vtheta)$ by the interpolation between two minima found by the training of two different loss functions. For either $R_1(\vtheta)$ or $\tilde{R}_2(\vtheta,\veta)$, we use addition or subtraction to study its effect. As shown in Fig. \ref{fig:compare_flatness}(a), for $R_1(\vtheta)$, the loss value of the interpolation between the minima found by the addition approach ($L_1$) and the subtraction approach ($L_3$) stays near zero, which is similar for $\tilde{R}_2(\vtheta,\veta)$ in Fig. \ref{fig:compare_flatness}(b), showing that the higher-order terms of the learning rate $\varepsilon$ in the modified equation have less influence on the training process. We then compare the flatness of minima found by the training with $R_1(\vtheta)$ and $\tilde{R}_2(\vtheta,\veta)$ as illustrated in  Fig. \ref{fig:compare_flatness}(c-f). The results indicate that the minima obtained by the training with $R_1(\vtheta)$ exhibit greater flatness than those obtained by training with $\tilde{R}_2(\vtheta,\veta)$.

The experiments in this section show that, compared with SGD, the unique implicit regularization of dropout, $R_1(\vtheta)$, plays a significant role in improving the generalization and finding flat minima. 



\begin{figure}[h]
	\centering
	\subfloat[$L_1 \& L_3$]{\includegraphics[width=0.3\textwidth]{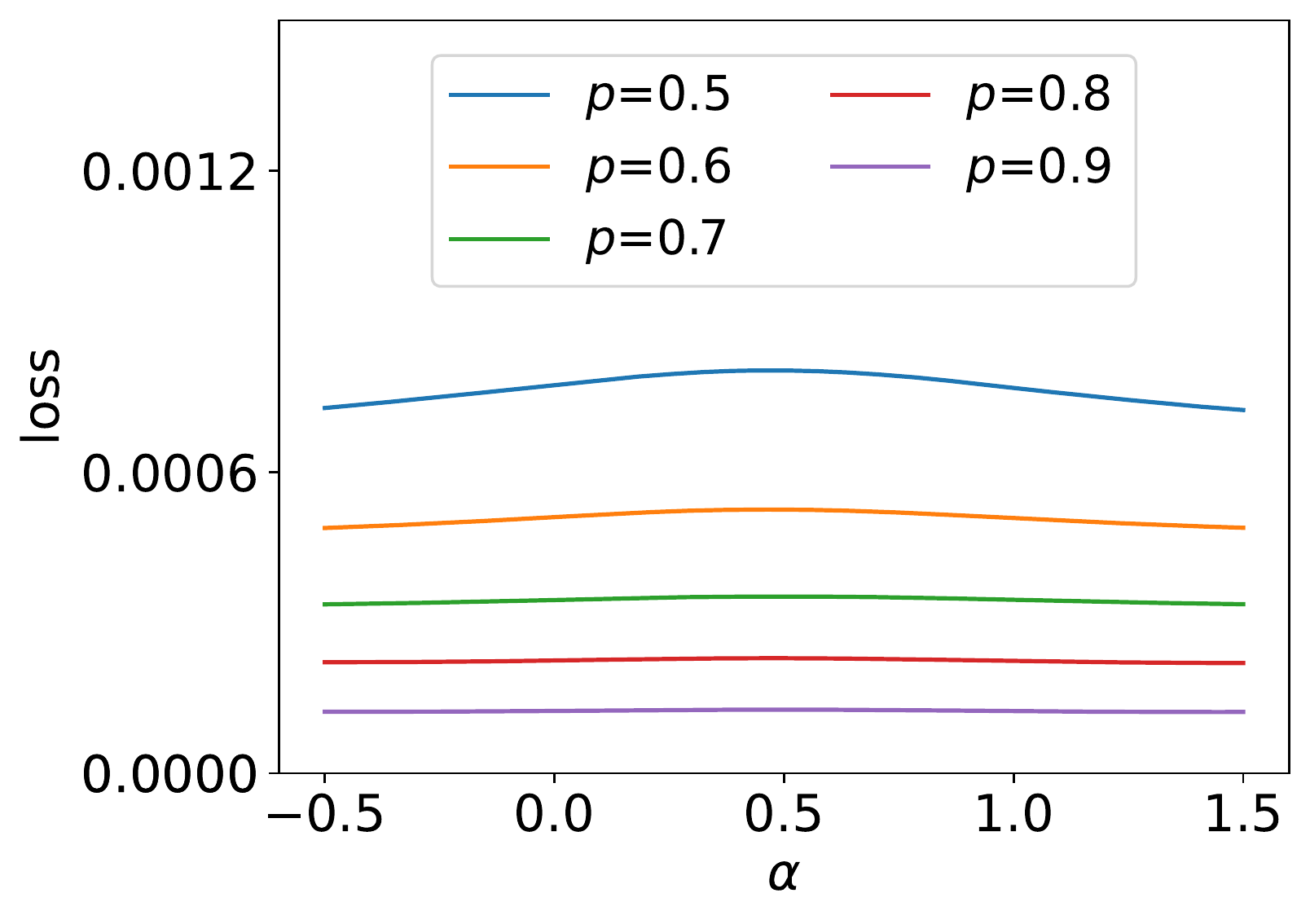}}
    \subfloat[$L_2 \& L_4$]{\includegraphics[width=0.3\textwidth]{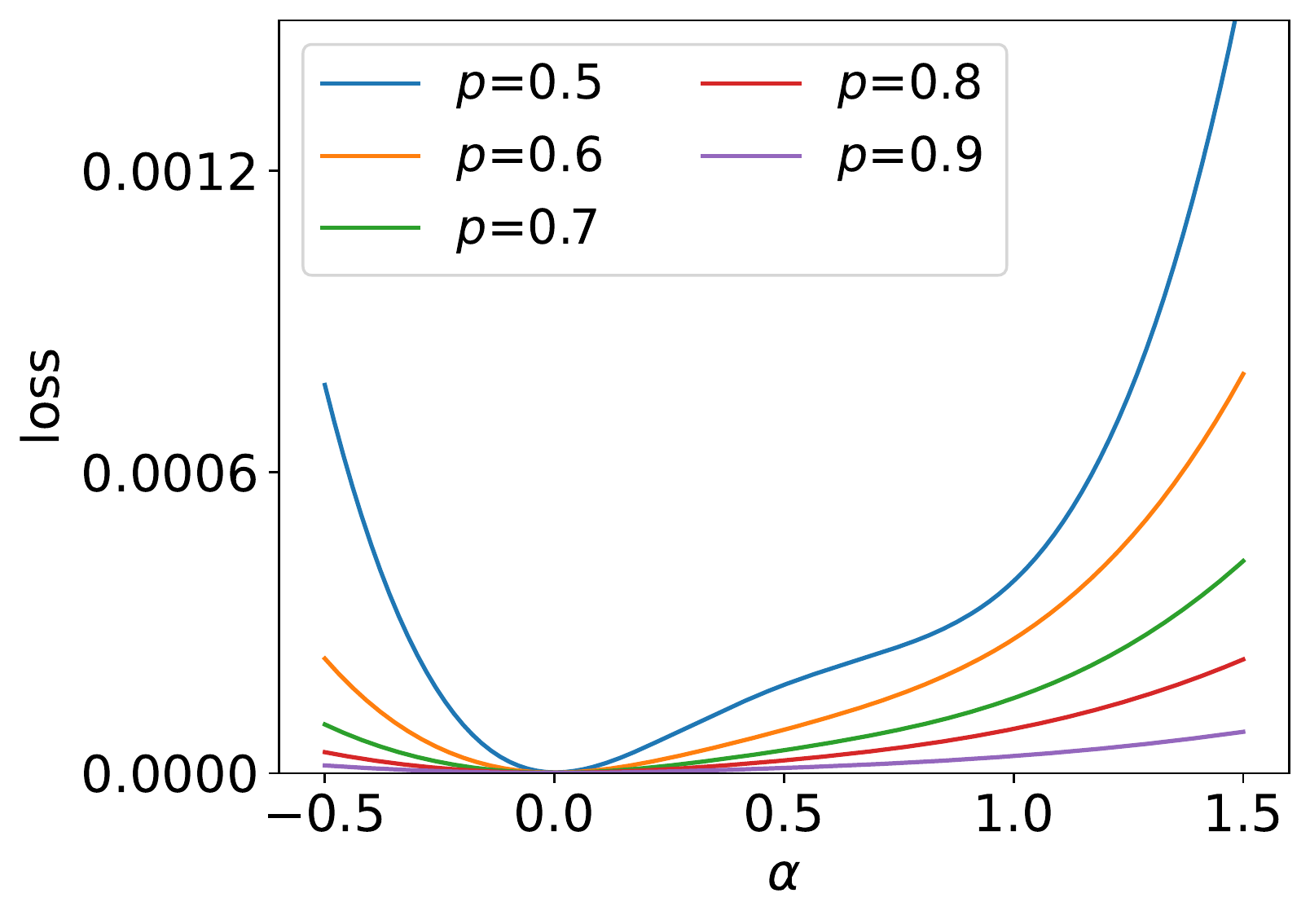}}
	\subfloat[$L_1 \& L_2$]{\includegraphics[width=0.3\textwidth]{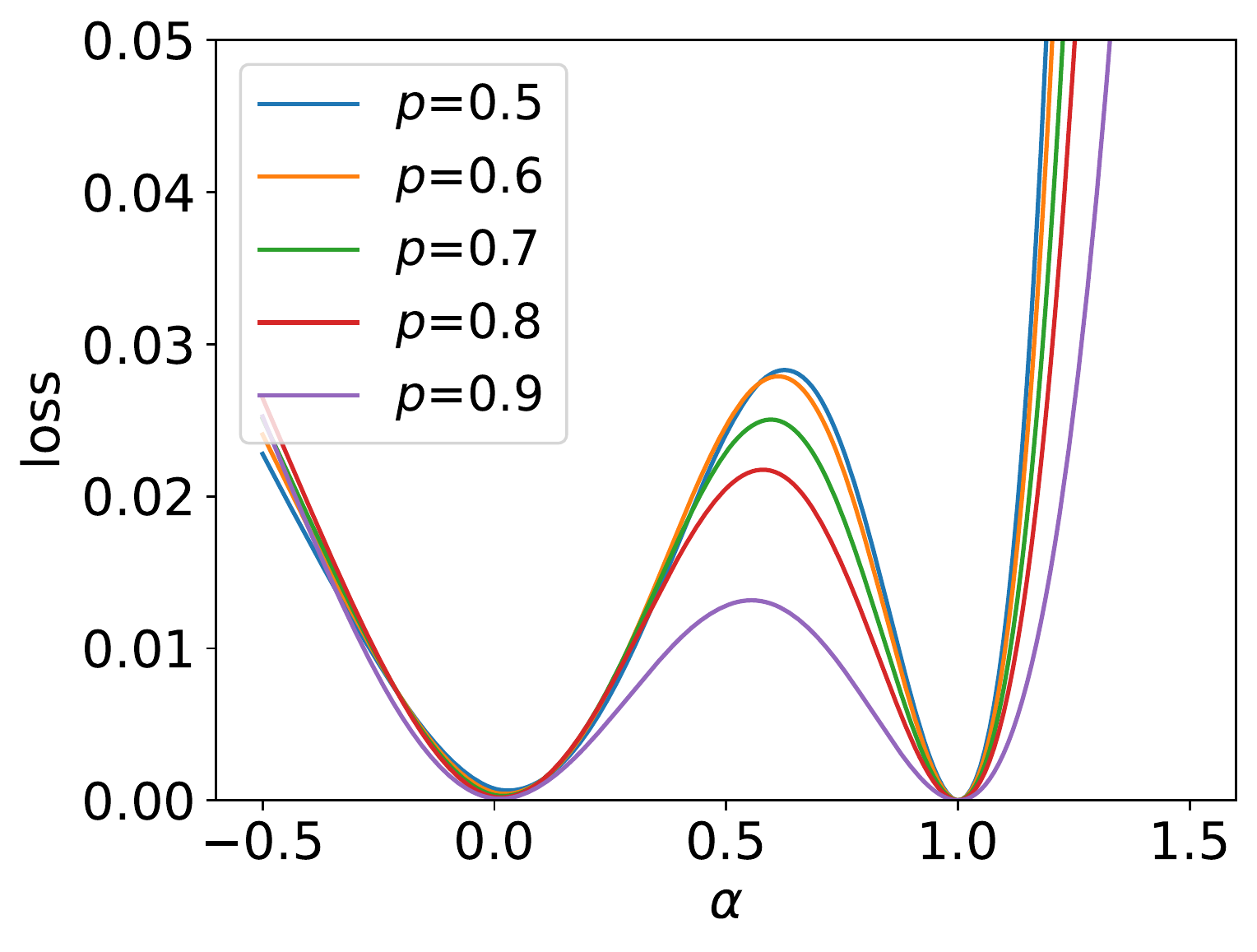}}\\
    \subfloat[$L_1 \& L_4$]{\includegraphics[width=0.3\textwidth]{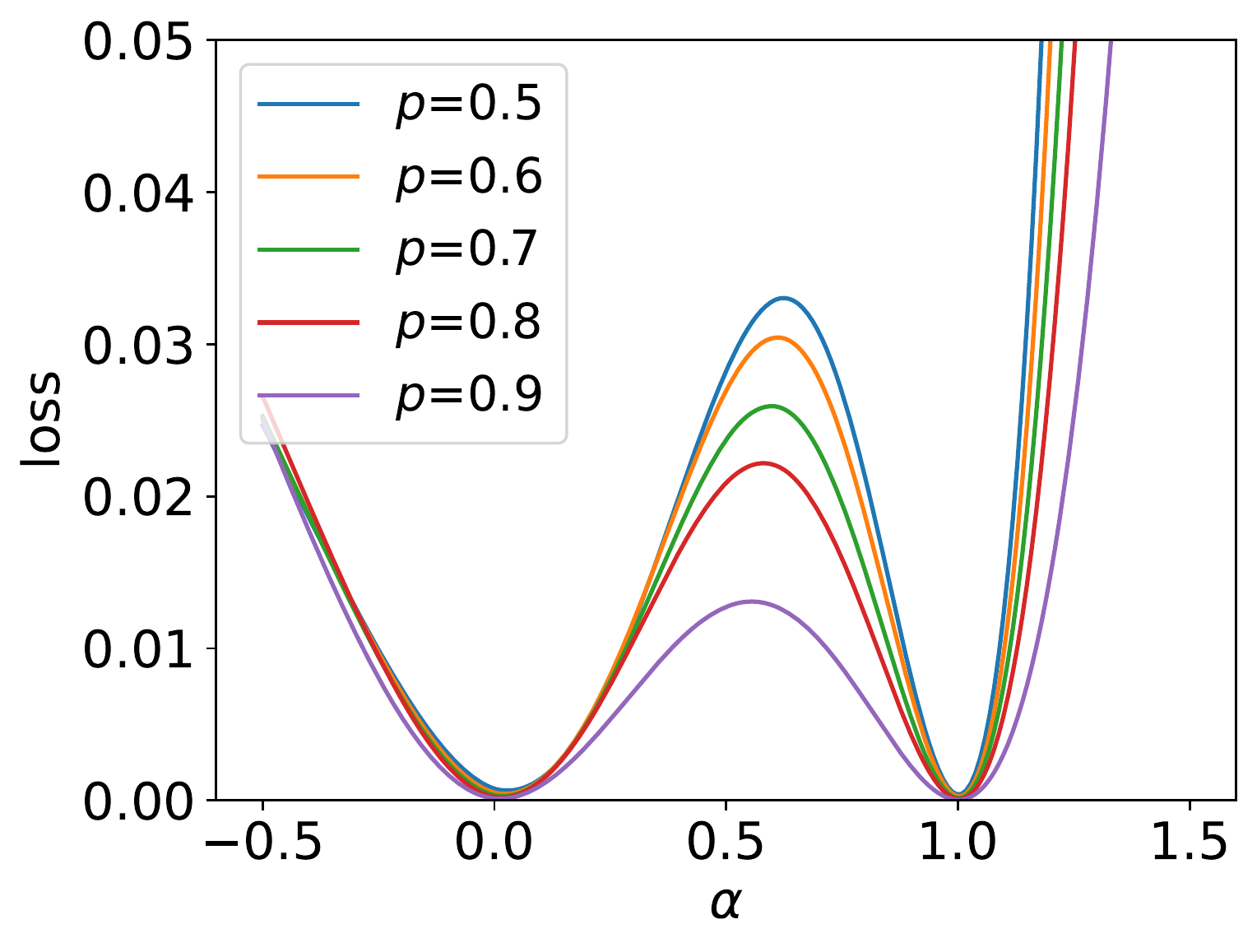}}
	\subfloat[$L_2 \& L_3$]{\includegraphics[width=0.3\textwidth]{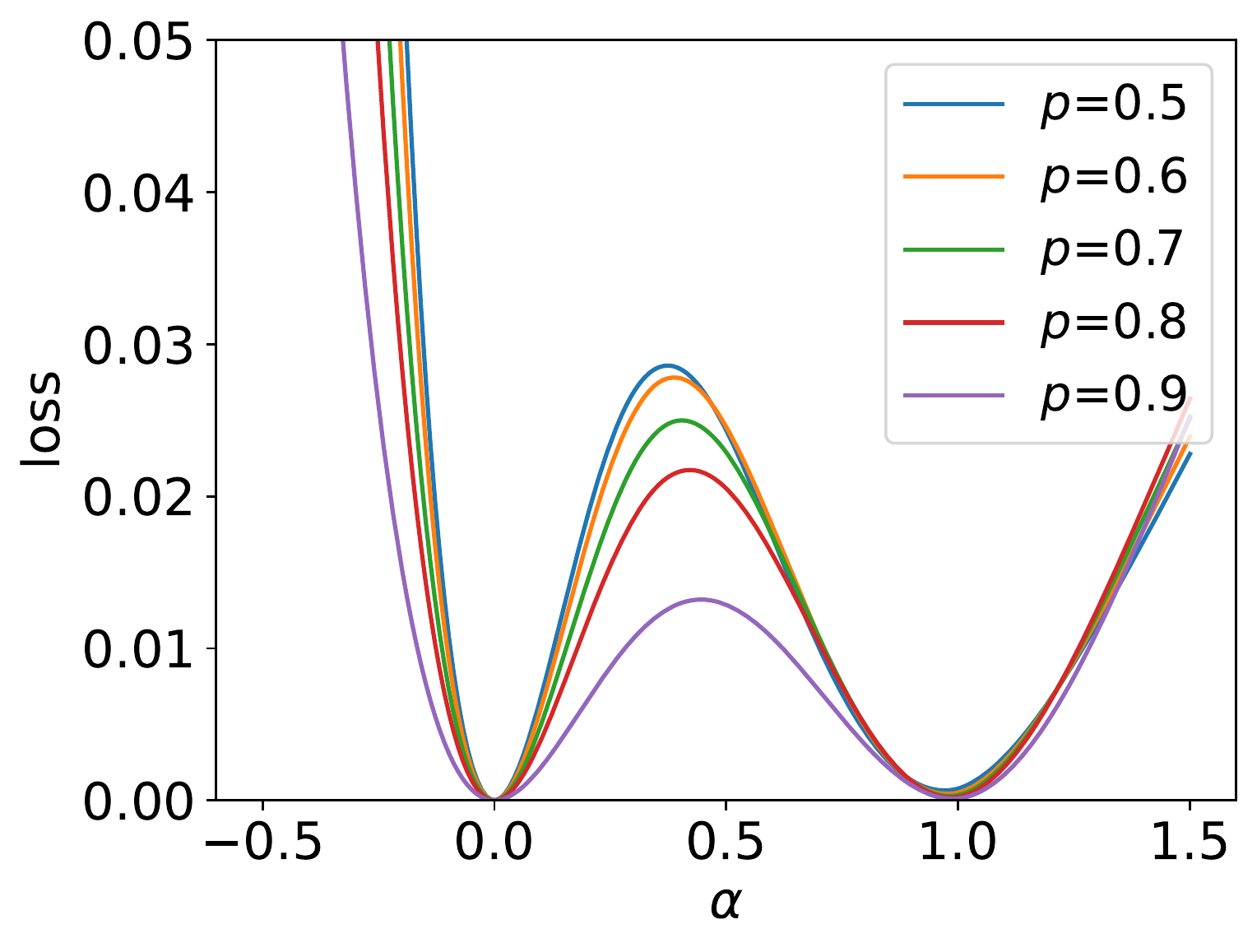}}
    \subfloat[$L_3 \& L_4$]{\includegraphics[width=0.3\textwidth]{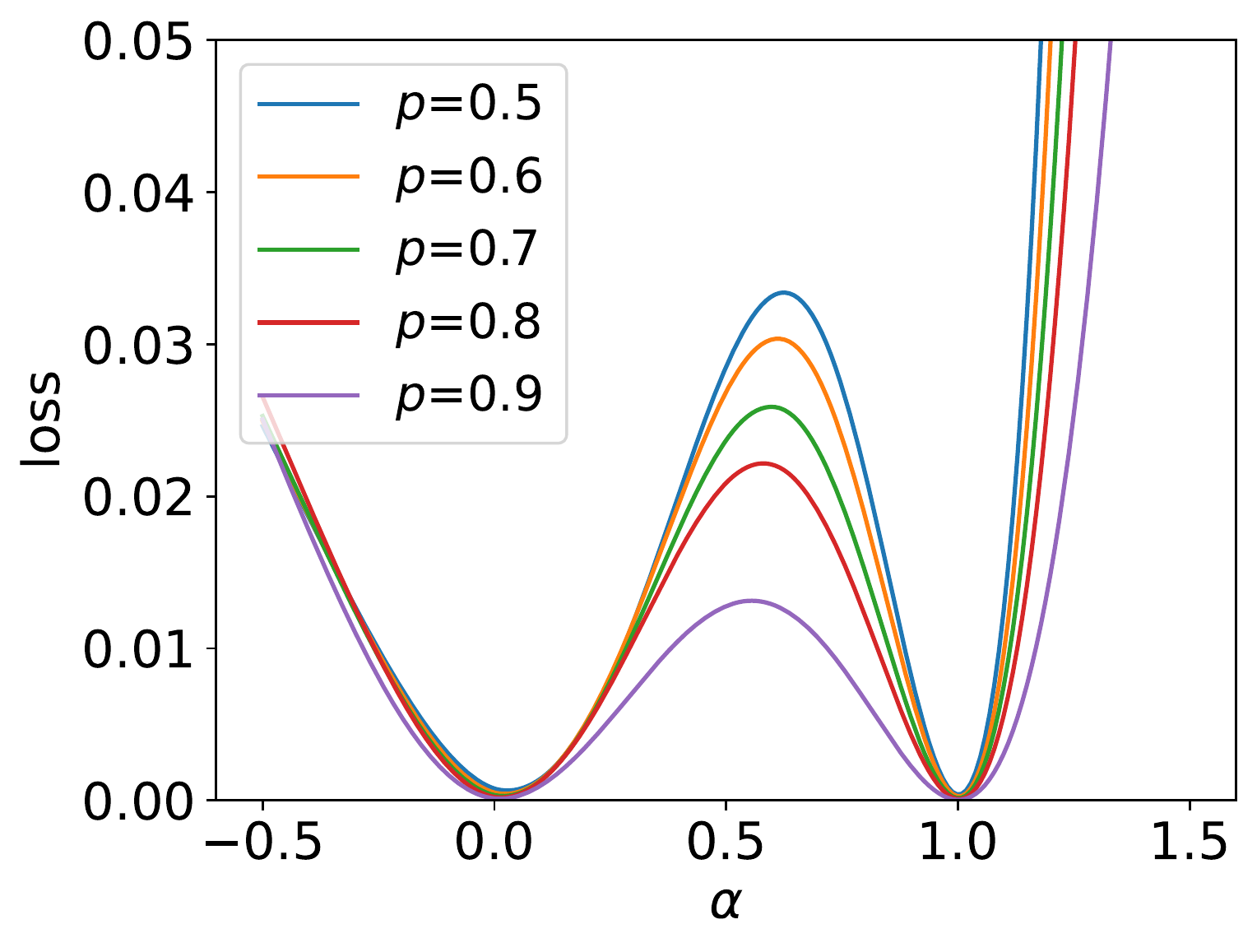}}
    \caption{The classification task on the MNIST dataset (the first 1000 images) using the FNN with size $784$-$1000$-$10$. The $R_S(\vtheta)$ value for interpolation between models with $\alpha$ interpolation factor. For $L_i \& L_j$, there is one trained model at $\alpha=0$ (trained by loss function $L_i$), and the other is at $\alpha=1$  (trained by loss function $L_j$). Different curves represent different dropout rates used for training. \label{fig:compare_flatness}}
    
\end{figure}

\section{Conclusion and Discussion}
In this work, we theoretically study the implicit regularization of dropout and its role in improving the generalization performance of neural networks. Specifically, we derive two implicit regularization terms, $R_1(\vtheta)$ and $R_2(\vtheta)$, and validate their efficacy through numerical experiments.
One important finding of this work is that the unique implicit regularization term $R_1(\vtheta)$ in dropout, unlike SGD, is a key factor in improving the generalization and flatness of the dropout solution. We also found that $R_1(\vtheta)$ can facilitate the weight condensation during training, which may establish a link among weight condensation, flatness, and generalization for further study. This work reveals rich and unique properties of dropout, which are fundamental to a comprehensive understanding of dropout.

Our study also sheds light on the broader issue of simplicity bias in deep learning. We observed that dropout regularization tends to impose a bias toward simple solutions during training, as evidenced by the weight condensation and flatness effects. This is consistent with other perspectives on simplicity bias in deep learning, such as the frequency principle\cite{xu2019training,xu2019frequency,zhang2021linear,luo2019theory,xu2022overview}, which reveals that neural networks often learn data from low to high frequency. Our analysis of dropout regularization provides a detailed understanding of how simplicity bias works in practice, which is essential for understanding why over-parameterized neural networks can fit the training data well and generalize effectively to new data.

Finally, our work highlights the potential benefits of dropout regularization in training neural networks, particularly in the linear regime. As we have shown, dropout regularization can induce weight condensation and avoid the slow training speed often encountered in highly nonlinear networks due to the fact that the training trajectory is close to the stationary point \cite{zhang2021embedding,zhang2022embedding}. This may have important implications for the development of more efficient and effective deep learning algorithms.

\section*{Acknowledgments}
This work is sponsored by the National Key R\&D Program of China  Grant No. 2022YFA1008200, the Shanghai Sailing Program, the Natural Science Foundation of Shanghai Grant No. 20ZR1429000, the National Natural Science Foundation of China Grant No. 62002221, Shanghai Municipal of Science and Technology Major Project No. 2021SHZDZX0102, and the HPC of School of Mathematical Sciences and the Student Innovation Center, and the Siyuan-1 cluster supported by the Center for High Performance Computing at Shanghai Jiao Tong University.

\bibliography{dl}

\begin{thebibliography}{10}

\bibitem{hinton2012improving}
Geoffrey~E Hinton, Nitish Srivastava, Alex Krizhevsky, Ilya Sutskever, and
  Ruslan~R Salakhutdinov.
\newblock Improving neural networks by preventing co-adaptation of feature
  detectors.
\newblock {\em arXiv preprint arXiv:1207.0580}, 2012.

\bibitem{srivastava2014dropout}
Nitish Srivastava, Geoffrey Hinton, Alex Krizhevsky, Ilya Sutskever, and Ruslan
  Salakhutdinov.
\newblock Dropout: a simple way to prevent neural networks from overfitting.
\newblock {\em The journal of machine learning research}, 15(1):1929--1958,
  2014.

\bibitem{tan2019efficientnet}
Mingxing Tan and Quoc Le.
\newblock Efficientnet: Rethinking model scaling for convolutional neural
  networks.
\newblock In {\em International conference on machine learning}, pages
  6105--6114. PMLR, 2019.

\bibitem{helmbold2015inductive}
David~P Helmbold and Philip~M Long.
\newblock On the inductive bias of dropout.
\newblock {\em The Journal of Machine Learning Research}, 16(1):3403--3454,
  2015.

\bibitem{hairer2003geometric}
Ernst Hairer, Christian Lubich, and Gerhard Wanner.
\newblock Geometric numerical integration illustrated by the
  st{\"o}rmer--verlet method.
\newblock {\em Acta numerica}, 12:399--450, 2003.

\bibitem{luo2021phase}
Tao Luo, Zhi-Qin~John Xu, Zheng Ma, and Yaoyu Zhang.
\newblock Phase diagram for two-layer relu neural networks at infinite-width
  limit.
\newblock {\em Journal of Machine Learning Research}, 22(71):1--47, 2021.

\bibitem{zhou2021towards}
Hanxu Zhou, Qixuan Zhou, Tao Luo, Yaoyu Zhang, and Zhi-Qin~John Xu.
\newblock Towards understanding the condensation of neural networks at initial
  training.
\newblock {\em arXiv preprint arXiv:2105.11686}, 2021.

\bibitem{zhou2022empirical}
Hanxu Zhou, Qixuan Zhou, Zhenyuan Jin, Tao Luo, Yaoyu Zhang, and Zhi-Qin~John
  Xu.
\newblock Empirical phase diagram for three-layer neural networks with infinite
  width.
\newblock {\em Advances in Neural Information Processing Systems}, 2022.

\bibitem{jacot2018neural}
Arthur Jacot, Cl{\'{e}}ment Hongler, and Franck Gabriel.
\newblock Neural tangent kernel: Convergence and generalization in neural
  networks.
\newblock In {\em Advances in Neural Information Processing Systems}, pages
  8580--8589, 2018.

\bibitem{keskar2016large}
Nitish~Shirish Keskar, Dheevatsa Mudigere, Jorge Nocedal, Mikhail Smelyanskiy,
  and Ping Tak~Peter Tang.
\newblock On large-batch training for deep learning: Generalization gap and
  sharp minima.
\newblock {\em arXiv preprint arXiv:1609.04836}, 2016.

\bibitem{neyshabur2017exploring}
Behnam Neyshabur, Srinadh Bhojanapalli, David McAllester, and Nathan Srebro.
\newblock Exploring generalization in deep learning.
\newblock {\em arXiv preprint arXiv:1706.08947}, 2017.

\bibitem{zhu2018anisotropic}
Zhanxing Zhu, Jingfeng Wu, Bing Yu, Lei Wu, and Jinwen Ma.
\newblock The anisotropic noise in stochastic gradient descent: Its behavior of
  escaping from sharp minima and regularization effects.
\newblock {\em arXiv preprint arXiv:1803.00195}, 2018.

\bibitem{mcallester2013pac}
David McAllester.
\newblock A pac-bayesian tutorial with a dropout bound.
\newblock {\em arXiv preprint arXiv:1307.2118}, 2013.

\bibitem{wan2013regularization}
Li~Wan, Matthew Zeiler, Sixin Zhang, Yann Lecun, and Rob Fergus.
\newblock Regularization of neural networks using dropconnect.
\newblock In {\em In Proceedings of the International Conference on Machine
  learning}. Citeseer, 2013.

\bibitem{mou2018dropout}
Wenlong Mou, Yuchen Zhou, Jun Gao, and Liwei Wang.
\newblock Dropout training, data-dependent regularization, and generalization
  bounds.
\newblock In {\em International conference on machine learning}, pages
  3645--3653. PMLR, 2018.

\bibitem{wager2013dropout}
Stefan Wager, Sida Wang, and Percy~S Liang.
\newblock Dropout training as adaptive regularization.
\newblock {\em Advances in neural information processing systems}, 26:351--359,
  2013.

\bibitem{mianjy2018implicit}
Poorya Mianjy, Raman Arora, and Rene Vidal.
\newblock On the implicit bias of dropout.
\newblock In {\em International Conference on Machine Learning}, pages
  3540--3548. PMLR, 2018.

\bibitem{barrett2020implicit}
David Barrett and Benoit Dherin.
\newblock Implicit gradient regularization.
\newblock In {\em International Conference on Learning Representations}, 2020.

\bibitem{smith2020origin}
Samuel~L Smith, Benoit Dherin, David Barrett, and Soham De.
\newblock On the origin of implicit regularization in stochastic gradient
  descent.
\newblock In {\em International Conference on Learning Representations}, 2020.

\bibitem{wu2018sgd}
Lei Wu, Chao Ma, and Weinan E.
\newblock How sgd selects the global minima in over-parameterized learning: A
  dynamical stability perspective.
\newblock {\em Advances in Neural Information Processing Systems}, 31, 2018.

\bibitem{yaida2018fluctuation}
Sho Yaida.
\newblock Fluctuation-dissipation relations for stochastic gradient descent.
\newblock In {\em International Conference on Learning Representations}, 2018.

\bibitem{zhang2021embedding}
Yaoyu Zhang, Zhongwang Zhang, Tao Luo, and Zhiqin~J Xu.
\newblock Embedding principle of loss landscape of deep neural networks.
\newblock {\em Advances in Neural Information Processing Systems},
  34:14848--14859, 2021.

\bibitem{zhang2022embedding}
Yaoyu Zhang, Yuqing Li, Zhongwang Zhang, Tao Luo, and Zhi-Qin~John Xu.
\newblock Embedding principle: a hierarchical structure of loss landscape of
  deep neural networks.
\newblock {\em Journal of Machine Learning vol}, 1:1--45, 2022.

\bibitem{maennel2018gradient}
Hartmut Maennel, Olivier Bousquet, and Sylvain Gelly.
\newblock Gradient descent quantizes relu network features.
\newblock {\em arXiv preprint arXiv:1803.08367}, 2018.

\bibitem{pellegrini2020analytic}
Franco Pellegrini and Giulio Biroli.
\newblock An analytic theory of shallow networks dynamics for hinge loss
  classification.
\newblock {\em Advances in Neural Information Processing Systems}, 33, 2020.

\bibitem{bai2022embedding}
Zhiwei Bai, Tao Luo, Zhi-Qin~John Xu, and Yaoyu Zhang.
\newblock Embedding principle in depth for the loss landscape analysis of deep
  neural networks.
\newblock {\em arXiv preprint arXiv:2205.13283}, 2022.

\bibitem{zhang2022linear}
Yaoyu Zhang, Zhongwang Zhang, Leyang Zhang, Zhiwei Bai, Tao Luo, and
  Zhi-Qin~John Xu.
\newblock Linear stability hypothesis and rank stratification for nonlinear
  models.
\newblock {\em arXiv preprint arXiv:2211.11623}, 2022.

\bibitem{blanc2020implicit}
Guy Blanc, Neha Gupta, Gregory Valiant, and Paul Valiant.
\newblock Implicit regularization for deep neural networks driven by an
  ornstein-uhlenbeck like process.
\newblock In {\em Conference on learning theory}, pages 483--513. PMLR, 2020.

\bibitem{li2017visualizing}
Hao Li, Zheng Xu, Gavin Taylor, Christoph Studer, and Tom Goldstein.
\newblock Visualizing the loss landscape of neural nets.
\newblock {\em arXiv preprint arXiv:1712.09913}, 2017.

\bibitem{xu2019training}
Zhi-Qin~John Xu, Yaoyu Zhang, and Yanyang Xiao.
\newblock Training behavior of deep neural network in frequency domain.
\newblock In {\em International Conference on Neural Information Processing},
  pages 264--274. Springer, 2019.

\bibitem{xu2019frequency}
Zhi-Qin~John Xu, Yaoyu Zhang, Tao Luo, Yanyang Xiao, and Zheng Ma.
\newblock Frequency principle: Fourier analysis sheds light on deep neural
  networks.
\newblock {\em Communications in Computational Physics}, 28(5):1746--1767,
  2020.

\bibitem{zhang2021linear}
Yaoyu Zhang, Tao Luo, Zheng Ma, and Zhi-Qin~John Xu.
\newblock A linear frequency principle model to understand the absence of
  overfitting in neural networks.
\newblock {\em Chinese Physics Letters}, 38(3):038701, 2021.

\bibitem{luo2019theory}
Tao Luo, Zheng Ma, Zhi-Qin~John Xu, and Yaoyu Zhang.
\newblock Theory of the frequency principle for general deep neural networks.
\newblock {\em CSIAM Transactions on Applied Mathematics}, 2(3):484--507, 2021.

\bibitem{xu2022overview}
Zhi-Qin~John Xu, Yaoyu Zhang, and Tao Luo.
\newblock Overview frequency principle/spectral bias in deep learning.
\newblock {\em arXiv preprint arXiv:2201.07395}, 2022.

\bibitem{simonyan2014very}
Karen Simonyan and Andrew Zisserman.
\newblock Very deep convolutional networks for large-scale image recognition.
\newblock {\em arXiv preprint arXiv:1409.1556}, 2014.

\bibitem{kingma2015adam}
Diederik~P Kingma and Jimmy Ba.
\newblock Adam: A method for stochastic optimization.
\newblock In {\em ICLR (Poster)}, 2015.

\bibitem{he2016deep}
Kaiming He, Xiangyu Zhang, Shaoqing Ren, and Jian Sun.
\newblock Deep residual learning for image recognition.
\newblock In {\em Proceedings of the IEEE conference on computer vision and
  pattern recognition}, pages 770--778, 2016.

\bibitem{vaswani2017attention}
Ashish Vaswani, Noam Shazeer, Niki Parmar, Jakob Uszkoreit, Llion Jones,
  Aidan~N Gomez, {\L}ukasz Kaiser, and Illia Polosukhin.
\newblock Attention is all you need.
\newblock In {\em Advances in neural information processing systems}, pages
  5998--6008, 2017.

\end{thebibliography}
\bibliographystyle{unsrt}

\newpage
\appendix
\section{Experimental Setups} \label{app:1}
For Fig. \ref{pic:condense_toy}, Fig. \ref{pic:condense_relu},  Fig. \ref{pic:sgd_condense_relu} and Fig. \ref{pic:condense_heat_relu}, we use the ReLU FNN with the width of $1000$ to fit the target function as follows, 
\begin{equation*}
    f(x)=\frac{1}{2}\sigma(-x-\frac{1}{3})+\frac{1}{2}\sigma(x-\frac{1}{3}),
\end{equation*}
where $\sigma(x)=\ReLU(x)$. For Fig. \ref{pic:sgd_condense_relu}, we train the network using Adam with a learning rate of $1\times10^{-4}$ and a batch size of $2$. For Fig. \ref{pic:condense_relu}, we add dropout layers behind the hidden layer with $p=0.9$ for the two-layer experiments and add dropout layers between two hidden layers and behind the last hidden layer with $p=0.9$ for the three-layer experiments. We train the network using Adam with a learning rate of $1\times10^{-4}$. We initialize the parameters in the linear regime, $\vtheta \sim N\left(0, \frac{1}{m^{0.2}}\right)$, where $m=1000$ is the width of the hidden layer.

For Fig. \ref{fig:R1}, Fig. \ref{fig:compar1}, and Fig. \ref{fig:compare_flatness}, we use the FNN with size $784$-$1000$-$10$ to classify the MNIST dataset (the first 1000 images). We add dropout layers behind the hidden layer with different dropout rates. We train the network using GD with a learning rate of $5\times10^{-3}$.

For Fig. \ref{fig:R2}, Fig. \ref{fig:R22}(b), we use VGG-9 \cite{simonyan2014very} w/o dropout layers to classify CIFAR-10 using GD or SGD. For experiments with dropout layers, we add dropout layers after the pooling layers, the dropout rates of dropout layers are $0.5$. For the experiments shown in Fig. \ref{fig:R2}, we only use the first 1000 images for training to compromise with the computational burden. 

For Fig. \ref{pic:condense_tanh}, Fig. \ref{pic:add_dropout_tanh}  Fig. \ref{pic:sgd_condense_tanh} and Fig. \ref{pic:condense_heat_tanh}, we use the tanh FNN with the width of $1000$ to fit the target function as follows, 
\begin{equation*}
    f(x)=\sigma(x-6)+\sigma(x+6),
\end{equation*}
where $\sigma(x)=\mathrm{tanh}(x)$. We add dropout layers behind the hidden layer with $p=0.9$ for the two-layer experiments and add dropout layers between two hidden layers and behind the last hidden layer with $p=0.9$ for the three-layer experiments. We train the network using Adam with a learning rate of $1\times10^{-4}$. We initialize the parameters in the linear regime, $\vtheta \sim N\left(0, \frac{1}{m^{0.2}}\right)$, where $m=1000$ is the width of the hidden layer. For Fig. \ref{pic:sgd_condense_tanh}, we train the network using Adam with a batch size of $2$. For Fig. \ref{pic:condense_heat_tanh}, each test error is averaged over $10$ trials with random initialization. 

For Fig. \ref{fig:condense_nd}(a,b), we use the FNN with size $10$-$100$-$1$ to fit the target function, 
\begin{equation*}
    f(x)=\sigma(\sum_{i=1}^{10}x_i),
\end{equation*}
where $\sigma(x)=\mathrm{tanh}(x)$, $\vx\in \sR^{10}$, $x_i$ is the $i$th component of $\vx$ and the training size is $30$. We add dropout layers behind the hidden layer with $p=0.5$. We train the network using Adam with a learning rate of $0.001$.

For Fig. \ref{fig:condense_nd}(c), we use ResNet-18 to classify CIFAR-10 with and without dropout. We add dropout layers behind the last activation function of each block with $p=0.5$. We train the network using Adam with a learning rate of $0.001$ and a batch size 0f $128$.

For Fig. \ref{fig:flatness_cnn}(a), we use the FNN with size $784$-$1024$-$1024$-$10$. We add dropout layers behind the first and the second layers with $p=0.8$ and $p=0.5$, respectively. We train the network using default Adam optimizer \cite{kingma2015adam} with a learning rate of $1\times10^{-4}$.

For Fig. \ref{fig:flatness_cnn}(b), we use VGG-9 to compare the loss landscape flatness w/o dropout layers. For experiments with dropout layers, we add dropout layers after the pooling layers, with $p=0.8$. Models are trained using full-batch GD with Nesterov momentum, with the first $2048$ images as training set for $300$ epochs. The learning rate is initialized at $0.1$, and divided by a factor of 10 at epochs 150, 225, and 275.

For Fig. \ref{fig:flatness_cnn}(c), we use ResNet-20 \cite{he2016deep} to compare the loss landscape flatness w/o dropout layers. For experiments with dropout layers, we add dropout layers after the convolutional layers with $p=0.8$. We only consider the parameter matrix corresponding to the weight of the first convolutional layer of the first block of the ResNet-20. Models are trained using full-batch GD, with a training size of $50000$ for $1200$ epochs. The learning rate is initialized at $0.01$.

For Fig. \ref{fig:flatness_cnn}(d), we use transformer \cite{vaswani2017attention} with $d_{\mathrm{model}}=50, d_k=d_v=20, d_{\mathrm{ff}}=256, h=4, N=3$, the meaning of the parameters is consistent with the original paper. For experiments with dropout layers, we apply dropout to the output of each sub-layer, before it is added to the sub-layer input and normalized. In addition, we apply dropout to the sums of the embeddings and the positional encodings in both the encoder and decoder stacks. We set $p=0.9$ for dropout layers. For the English-German translation problem, we use the cross-entropy loss with label smoothing trained by full-batch Adam based on the Multi30k dataset. The learning rate strategy is the same as that in \cite{vaswani2017attention}. The warm-up step is $4000$ epochs, the training step is $10000$ epochs. We only use the first $2048$ examples for training to compromise with the computational burden.

For Fig. \ref{fig:R22}(a), we classify the
Fashion-MNIST dataset by training four-layer NNs of width $4096$ with a training
batch size of $64$. We add dropout layers behind the hidden layers with different dropout rates. The learning rate is $5\times10^{-3}$.

\newpage
\section{Proofs for Main Paper}

\subsection{Proof for Lemma \ref{lem:1}}
\begin{lemma*}[\textbf{the expectation of dropout loss}] 
       Given an $L$-layer FNN with dropout $\vf_{\vtheta, \veta}^\mathrm{drop} (\vx)$, under Setting 1--3, we have the expectation of dropout MSE: 
    \begin{align*}
     \Exp_{\veta}&(R_S^\mathrm{drop}\left(\vtheta, \veta\right))= R_S\left(\vtheta\right)+ R_1(\vtheta).
    \end{align*}

\end{lemma*}
\begin{proof}

With the definition of MSE, we have, 
\begin{align*}
     R&_S^\mathrm{drop}\left(\vtheta, \veta\right)=\frac{1}{2n}\sum_{i=1}^{n}\left(\vf_{\vtheta, \veta}^\mathrm{drop} (\vx_i)-\vy_i \right)^2 \\
    &=\frac{1}{2n}\sum_{i=1}^{n}\left(\vf_{\vtheta} (\vx_i)-\vy_i+\sum_{j=1}^{m_{L-1}}\mW^{[L]}_{j} (\veta)_{j} f^{[L-1]}_{\vtheta, j}(\vx_i)  \right)^2.
\end{align*}
With the definition of $\veta$, we have $\Exp(\veta)=\mzero$, thus,
\begin{equation*}
\begin{aligned}
    \Exp_{\veta} R_S^\mathrm{drop}\left(\vtheta, \veta\right)&= R_S\left(\vtheta\right)\\
    &+ \frac{1}{2n}\Exp_{\veta} \sum_{i=1}^{n}\left( \sum_{j=1}^{m_{L-1}}\mW^{[L]}_{j} (\veta)_{j} f^{[L-1]}_{\vtheta, j}(\vx_i)\right)^2.
\label{equ:2}
\end{aligned}
\end{equation*}
At the same time, we have $\Exp((\veta)_k(\veta)_j)=0, k \neq j$ and $\Exp((\veta)_k^2)=\frac{1-p}{p}$. Thus we have,
\begin{equation*}
\begin{aligned}
    \Exp_{\veta}  R_S^\mathrm{drop}\left(\vtheta, \veta\right)&= R_S\left(\vtheta\right)\\
    &+\frac{1-p}{2np}\sum_{i=1}^{n}\sum_{j=1}^{m_{L-1}}\Vert\mW^{[L]}_{j}f^{[L-1]}_{\vtheta, j}(\vx_i)\Vert^2.
\end{aligned}
\end{equation*}
\end{proof}

\subsection{Proof for the modified gradient flow of dropout}

\textbf{Modified gradient flow of dropout.} Under Setting 1--3, the mean iterate of $\vtheta$, with a learning rate $\varepsilon \ll 1$, stays close to the path of gradient flow on a modified loss $\dot{\vtheta}=-\nabla_{\vtheta}\tilde{R}_S^\mathrm{drop}(\vtheta,\veta)$, where the modified loss $\tilde{R}_S^{\mathrm{drop}}(\vtheta,\veta)$ satisfies:
\begin{equation*}
    \Exp_{\veta}\tilde{R}_S^{\mathrm{drop}}(\vtheta,\veta)=\RS \left(\vtheta\right)+ R_1(\vtheta)+R_2(\vtheta)+O(\varepsilon^2).
\end{equation*}


\begin{proof}

We assume the ODE of the GD with dropout has the form $\dot{\vtheta }=f(\vtheta)$, and introduce a modified gradient flow $\dot{\vtheta}=\tilde{f}(\vtheta)$, where, 
\begin{equation*}
    \tilde{f}(\vtheta)=f(\vtheta)+\varepsilon f_{1}(\vtheta)+\varepsilon^{2} f_{2}(\vtheta)+O(\varepsilon^{3}).
\end{equation*}

Thus by Taylor expansion, for small but finite $\varepsilon$, we have, 
\begin{equation*}
\begin{aligned}
\vtheta(t+\varepsilon) &=\vtheta(t)+\varepsilon \widetilde{f}(\vtheta(t))+\frac{\varepsilon^{2}}{2} \nabla_{\vtheta} \widetilde{f}(\vtheta(t)) \widetilde{f}(\vtheta(t))+O\left(\varepsilon^{3}\right) \\
&=\vtheta(t)+\varepsilon f(\vtheta(t))\\
& \quad +\varepsilon^{2}\left(f_{1}(\vtheta(t))+\frac{1}{2} \nabla f(\vtheta(t)) f(\vtheta(t))\right)+O\left(\varepsilon^{3}\right).
\end{aligned}
\end{equation*}

For GD, we have, 
\begin{equation*}
    \vtheta_{t+1}=\vtheta_{t}+\varepsilon f(\vtheta_t).
\end{equation*}

Combining $\vtheta(t+\varepsilon)=\vtheta_{t+1}$, $\vtheta(t)=\vtheta_{t}$, we have,
\begin{equation*}
\begin{aligned}
f(\vtheta) &= -\nabla_{\vtheta}\RS^\mathrm{drop}(\vtheta, \veta)\\
f_{1}(\vtheta)&=-\frac{1}{2}\nabla_{\vtheta} f(\vtheta(t)) f(\vtheta(t))\\
&=-\frac{1}{4}\nabla_{\vtheta} \Vert f(\vtheta(t)) \Vert^2\\
&=-\frac{1}{4}\nabla_{\vtheta} \Vert \nabla_{\vtheta}\RS^\mathrm{drop}(\vtheta, \veta) \Vert^2.
\end{aligned}
\end{equation*}

Thus, with the random variable $\veta$, 

\begin{equation*}
    \tilde{f}(\vtheta)=-\nabla_{\vtheta}\RS^\mathrm{drop}(\vtheta, \veta)- \frac{\varepsilon}{4}\nabla_{\vtheta} \Vert \nabla_{\vtheta}\RS^\mathrm{drop}(\vtheta, \veta) \Vert^2+O(\varepsilon^{2}).
\end{equation*}
In the average sense and combining Lemma \ref{lem:1}, we have, 
\begin{equation*}
\begin{aligned}
    \dot{\vtheta}&=\Exp_{\veta}\left(-\nabla_{\vtheta}\RS^\mathrm{drop}(\vtheta, \veta)- \frac{\varepsilon}{4}\nabla_{\vtheta} \Vert \nabla_{\vtheta}\RS^\mathrm{drop}(\vtheta, \veta) \Vert^2+O(\varepsilon^{2})\right)\\
    &=-\nabla_{\vtheta}\Exp_{\veta}\left(\RS^\mathrm{drop}(\vtheta, \veta)+ \frac{\varepsilon}{4} \Vert \nabla_{\vtheta}\RS^\mathrm{drop}(\vtheta, \veta) \Vert^2+O(\varepsilon^{2})\right)\\
    &=-\nabla_{\vtheta}\Big ( \RS \left(\vtheta\right)+ \frac{1-p}{2np}\sum_{i=1}^{n} \sum_{j=1}^{m_{L-1}}\Vert \mW^{[L]}_{j}f^{[L-1]}_{\vtheta, j}(\vx_i) \Vert^2 \\
    & \quad +\frac{\varepsilon}{4}\Exp_{\veta} \Vert \nabla_{\vtheta} \RS ^\mathrm{drop}\left(\vtheta, \veta\right) \Vert ^2 + O(\varepsilon^{2})\Big).
\end{aligned}
\end{equation*}

We have, 
\begin{equation*}
\begin{aligned}
    \tilde{R}_S^{\mathrm{drop}}(\vtheta)&=\RS \left(\vtheta\right)+ \frac{1-p}{2np}\sum_{i=1}^{n} \sum_{j=1}^{m_{L-1}}\Vert \mW^{[L]}_{j}f^{[L-1]}_{\vtheta, j}(\vx_i) \Vert^2\\
    & \quad +\frac{\varepsilon}{4}\Exp_{\veta} \Vert \nabla_{\vtheta} \RS ^\mathrm{drop}\left(\vtheta, \veta\right) \Vert ^2+ O(\varepsilon^{2}).
\end{aligned}
\end{equation*}

\end{proof}


\subsection{Proof for Theorem \ref{thm:perb}}

\begin{theorem*}[\textbf{the effect of $R_1(\vtheta)$ on facilitating condensation}] Consider the following two-layer ReLU NN, 
\begin{equation*}
	f_{\vtheta}(x)=\sum_{j=1}^{m} a_j \sigma(w_j x + b_j)+a x+b, 
\end{equation*}
trained with a one-dimensional dataset $S=\{(x_{i}, y_{i})\}_{i=1}^n$, where $x_1<x_2< \cdots <x_n$. When the MSE of training data $R_S(\vtheta)=0$, if any of the following two conditions holds:

(i) the number of convexity changes of NN in $(x_ {1}, x_ {n})$ can be reduced while $R_S(\vtheta)=0$;

(ii) there exist two neurons with indexes $k_1 \neq k_2$, such that they have the same sign, i.e.,  $\mathrm{sign}(w_{k_1})=\mathrm{sign}(w_{k_2})$, and different intercept points in the same interval, i.e., $-{b_{k_1}}/{w_{k_1}}, -{b_{k_2}}/{w_{k_2}}\in [x_ {i}, x_ {i+1}]$, and $-{b_{k_1}}/{w_{k_1}}\neq -{b_{k_2}}/{w_{k_2}}$ for some $i \in [2: n-1]$;\\
then there exists parameters $\vtheta^{\prime}$, an infinitesimal perturbation of $\vtheta$, s.t.,

(i) $R_S(\vtheta^{\prime})=0$;

(ii) $R_1(\vtheta^{\prime})<R_1(\vtheta)$.
\end{theorem*}

\begin{proof}
The implicit regularization term $R_1(\vtheta)$ can be expresses as follows:
\begin{equation*}
    R_1(\vtheta)=\frac{1-p}{2np}\sum_{i=1}^{n}\sum_{j=1}^{m}(a_{j}\sigma(w_{j}x_{i}+b_{j}))^2, 
\end{equation*}
where $m$ is the width of the NN, and $\sigma(x)=\ReLU(x)$. It is worth noting that the term $ax+b$ will be absorbed by $R_S(\vtheta)$, so it will not affect the value of $R_1(\vtheta)$.

First, we show the proof outline, which is inspired by \cite{blanc2020implicit}. However, \cite{blanc2020implicit} study the setting of noise SGD, and only consider the limitation of convexity changes. For the first part, i.e., the limitation of the number of convexity changes, the proof outline is as follows. For any set of consecutive datapoints $\{(x_i, y_i), (x_{i+1}, y_{i+1}), (x_{i+2}, y_{i+2})\}$, we assume the linear interpolation of three datapoints is convex. If $f_{\vtheta}(x)$ is not convex in $(x_i, x_{i+2})$, there exist two neurons, one with a positive output layer weight and the other with a negative output layer weight, and the intercept point of both neurons are within the interval $(x_i, x_{i+2})$. Then, we can make the intercept points of the two neurons move to both sides by giving a specific perturbation direction, and through this movement, the intercept point of the neurons can be gathered at the data points $x_i, x_{i+2}$. At the same time, we can verify that this moving mode can reduce the value of $R_1(\vtheta)$ term while keeping the training error at $0$. The main proof mainly classifies the different directions of the above two neurons to construct different types of perturbation direction.
For the second part, i.e.,  the limitation of the number of intercept points in the same inner interval, the proof outline is as follows. For two neurons with the same direction and different intercept points in the same interval $[x_ {i}, x_ {i+1}], i\in[n-1]$, without loss of generality, suppose $\{x_\tau\}_{\tau=i+1}^n$ is the input dataset activated by two neurons at the same time. According to the mean value inequality, when the output values of two neurons are constant at each data point $\{x_\tau\}_{\tau=i+1}^n$, the $R_1(\vtheta)$ term is minimized. We can achieve this by moving the two intercept points of neurons to the interval within the two intercept points through perturbation.

\textbf{The limitation of convexity changes.} Without loss of generality, suppose $f_{\vtheta}$ on the interval $\left(x_i, x_{i+2}\right)$ is convex. If $f(\theta, x)$ fits these three points, but is not convex, then it must have a convexity change. This convexity change corresponds to at least two ReLU neurons $a_j \sigma(w_j x + b_j), j \in [2]$, and the output weights of the two have opposite signs. WLOG, we assume the intercept point of the first neuron is less than the intercept point of the second neuron, i.e., $-b_1/w_1<-b_2/w_2$. After reasonably assuming that the output weights of neurons with small turning points are negative, we classify the signs of the two neurons as follows,
\begin{equation*}
\begin{aligned}
    \textrm{Case (1) : }&w_1>0, a_1<0, w_2>0, a_2>0, \\
    \textrm{Case (2) : }&w_1>0, a_1<0, w_2<0, a_2>0, \\
    \textrm{Case (3) : }&w_1<0, a_1<0, w_2>0, a_2>0, \\
    \textrm{Case (4) : }&w_1<0, a_1<0, w_2<0, a_2>0.
\end{aligned}
\end{equation*}

Case (1) : For any sufficiently small $\varepsilon>0$, we perturb the parameters of the two neurons as follows:

\begin{equation*}
\begin{array}{cl}
\tilde{w}_1=w_1(1-\varepsilon), & \tilde{b}_1=b_1+x_{i+1} w_1 \varepsilon, \\
\tilde{w}_2=w_2-\frac{a_1}{a_2}\left(\tilde{w}_1-w_1\right), & \tilde{b}_2=b_2-\frac{a_1}{a_2}\left(\tilde{b}_1-b_1\right).
\end{array}
\end{equation*}

Here we only consider the case where the intercept point of the first neuron $-\frac{b_1}{w_1}\leq x_{i+1}$. For the case where $-\frac{b_1}{w_1}> x_{i+1}$, it is easy to obtain from the proof of the limitation of the number of intercept points below. In the same way, we can only consider the case where $-\frac{b_2}{w_2}\geq x_{i+1}$. 
We first consider the case where $-\frac{b_1}{w_1}< x_{i+1}$ and $-\frac{b_2}{w_2}> x_{i+1}$. By studying the movement of the intercept point of the two neurons, we have,
\begin{equation*}
        -\frac{\tilde{b}_1}{\tilde{w}_1}-(-\frac{b_1}{w_1})=-\frac{\varepsilon\left(w_1 x_{i+1}+b_1\right)}{w_1(1-\varepsilon)}<0, 
\end{equation*}
where $0<\varepsilon<1$, $w_1>0$ and $w_1 x_{i+1}+b_1>0$. As for the second neuron, we have,
\begin{equation*}
        -\frac{\tilde{b}_2}{\tilde{w}_2}-(-\frac{b_2}{w_2})=\frac{w_1 a_1 \varepsilon\left(w_2 x_{i+1}+b_2\right)}{w_2\left(a_2 w_2+a_1 w_1 \varepsilon\right)}>0, 
\end{equation*}
where $0<\varepsilon<1$, $w_1>0$, $w_2>0$, $a_1<0$, $w_2 x_{i+1}+b_2<0$ and $a_2 w_2+a_1 w_1 \varepsilon>0$. Thus the intercept point of the first neuron moves left and the other moves right under the above perturbation. 

We then verify the invariance of the values of NN's output on the input data. We study this invariance within three data point sets $\{x_\tau\}_{\tau=1}^i$, $\{x_{i+1}\}$ and $\{x_\tau\}_{\tau=i+2}^n$. For $\{x_\tau\}_{\tau=1}^i$, the outputs of both neurons remain zero, so the output of the network is unchanged. For $\{x_{i+1}\}$, the output of the second neuron remains zero. As for the first neuron, we have,
\begin{equation*}
\begin{aligned}
       \tilde{a}_1\sigma(\tilde{w}_1 x_{i+1} \tilde{b}_1)&= a_1\sigma(w_1(1-\varepsilon) x_{i+1}+b_1+w_1 \varepsilon x_{i+1})\\
       &=a_1\sigma(w_1 x_{i+1}+b_1),
\end{aligned}
\end{equation*}
thus the output of the network on $\{x_{i+1}\}$ is unchanged. For $\{x_\tau\}_{\tau=i+2}^n$, we study the output value changes of the two neurons separately. For the first neuron, we have, 
\begin{equation*}
\begin{aligned}
    \tilde{a}_1&\sigma(\tilde{w}_1 x_\tau+\tilde{b}_1)-a_1\sigma(w_1 x_\tau+b_1)\\
    &=a_1(w_1(1-\varepsilon) x_\tau+b_1+w_1 \varepsilon x_{i+1}-w_1 x_\tau-b_1)\\
    &=\varepsilon a_1( w_1\left(x_{i+1}-x_\tau\right)).
\end{aligned}
\end{equation*}

For the second neuron, we have, 
\begin{equation*}
\begin{aligned}
    \tilde{a}_2&\sigma(\tilde{w}_2 x_\tau+\tilde{b}_2)-a_2\sigma(w_2 x_\tau+b_2)\\
    &=-\frac{a_1}{a_2}{a}_2\left(\left(\tilde{w}_1-w_1\right) x_\tau+\tilde{b}_1-b_1\right)\\
    &=-\varepsilon a_1( w_1\left(x_{i+1}-x_\tau\right)).
\end{aligned}
\end{equation*}
Thus output value changes of both neurons remain zero. After verifying that the output of the network remains constant on the input data points, we study the effect of this perturbation on the $R_1(\vtheta)$ term. Noting that the output values of these two neurons on the input data $\{x_\tau\}_{\tau=1}^{i+1}$ remain unchanged, we only need to study the effect of the output changes on the input data $\{x_\tau\}_{\tau=i+2}^n$. For $\{x_\tau\}_{\tau=i+2}^n$, we have,

\begin{equation}
\begin{aligned}
     a_1^2=\tilde{a}_1^2, \quad &a_2^2=\tilde{a}_2^2\\
    \sigma(\tilde{a}_1 x_\tau+\tilde{b}_1)^2<& \sigma(a_1 x_\tau+b_1)^2\\
    \sigma(\tilde{a}_2 x_\tau+\tilde{b}_2)^2<& \sigma(a_2 x_\tau+b_2)^2\\
    \label{equ:inequ}
\end{aligned}
\end{equation}

Thus we have, 
\begin{equation*}
\begin{aligned}
    \tilde{R}_1(\vtheta)&-R_1(\vtheta)\\
    =&\sum_{\tau=i+2}^n \Big(\tilde{a}_1^2 \sigma(\tilde{w}_1 x_\tau + \tilde{b}_1)^2+\tilde{a}_2^2 \sigma(\tilde{w}_2 x_\tau + \tilde{b}_2)^2\\
    &- a_1^2 \sigma(w_1 x_\tau + b_1)^2-a_2^2 \sigma(w_2 x_\tau + b_2)^2\Big)\\
    <&-\Theta(\varepsilon).
\end{aligned}
\end{equation*}

As for $\frac{b_1}{w_1} = x_{i+1}$(or $\frac{b_2}{w_2} = x_{i+1}$), we can still get a consistent conclusion through the above perturbation, the only difference is that the intercept point of the first(or second) neuron will not move to the left(or right).

For the remaining three cases, we show the direction of the perturbation respectively, and we can similarly verify that the loss value remains unchanged and the $R_1(\vtheta)$ term decreases by $\Theta(\varepsilon)$.

Case (2) and Case (3): The perturbation is shown as follows:
\begin{equation*}
\begin{array}{cc}
\tilde{w}_1=w_1(1-\varepsilon), & \tilde{b}_1=b_1+x_{i+1} w_1 \varepsilon, \\
\tilde{w}_2=w_2+\frac{a_1}{a_2}\left(\tilde{w}_1-w_1\right), & \tilde{b}_2=b_2+\frac{a_1}{a_2}\left(\tilde{b}_1-b_1\right), \\
a=-a_1\left(\tilde{w}_1-w_1\right), & b=-a_1\left(\tilde{b}_1-b_1\right).
\end{array}
\end{equation*}

Case (4): The perturbation is shown as follows:
\begin{equation*}
\begin{array}{cl}
\tilde{w}_2=w_2(1-\varepsilon), & \tilde{b}_2=b_2+x_{i+1} w_2 \varepsilon, \\
\tilde{w}_1=w_1-\frac{a_2}{a_1}\left(\tilde{w}_2-w_2\right), & \tilde{b}_1=b_1-\frac{a_2}{a_1}\left(\tilde{b}_2-b_2\right).
\end{array}
\end{equation*}

\textbf{The limitation of the number of intercept points.} For any inner interval $[x_ {i}, x_ {i+1}]\cap(x_ {2}, x_ {n-1}), i\in[n-1]$, we take two neurons with the same direction and different intercept points in the same inner interval denoted as $a_j\sigma(w_jx+b_j)$, $j\in [2]$. In order to ensure consistency with the above proof, we denote the boundary points of the inner interval are $x_{i+1}$ and $x_{i+2}$.
 Without loss of generality, we assume that the input weights of two neurons $w_1,w_2>0$, and the output weight of the first neuron $a_1<0$, and we assume the intercept point of the first neuron is less than the intercept point of the second neuron, i.e., $-b_1/w_1<-b_2/w_2$. We study $a_2>0$ and $a_2<0$ respectively, where the first case corresponds to the case $-\frac{b_1}{w_1}> x_{i+1}$ in Case (1) above. We first study the case where $a_2>0$. We perturb the parameters of two neurons as follows:
\begin{equation*}
\begin{array}{cl}
\tilde{w}_1=w_1(1-\varepsilon), & \tilde{b}_1=b_1+x_{i+1} w_1 \varepsilon, \\
\tilde{w}_2=w_2-\frac{a_1}{a_2}\left(\tilde{w}_1-w_1\right), & \tilde{b}_2=b_2-\frac{a_1}{a_2}\left(\tilde{b}_1-b_1\right).
\end{array}
\end{equation*}

 By studying the movement of the intercept points of the two neurons, we have,
\begin{equation*}
        -\frac{\tilde{b}_1}{\tilde{w}_1}-(-\frac{b_1}{w_1})=-\frac{\varepsilon\left(w_1 x_{i+1}+b_1\right)}{w_1(1-\varepsilon)}>0, 
\end{equation*}
where $0<\varepsilon<1$, $w_1>0$ and $w_1 x_{i+1}+b_1<0$. As for the second neuron, we have,
\begin{equation*}
        -\frac{\tilde{b}_2}{\tilde{w}_2}-(-\frac{b_2}{w_2})=\frac{w_1 a_1 \varepsilon\left(w_2 x_{i+1}+b_2\right)}{w_2\left(a_2 w_2+a_1 w_1 \varepsilon\right)}>0, 
\end{equation*}
where $0<\varepsilon<1$, $w_1>0$, $w_2>0$, $a_1<0$, $w_2 x_{i+1}+b_2<0$ and $a_2 w_2+a_1 w_1 \varepsilon>0$. Thus the intercept points of two neurons move right under the above perturbation.

We then verify the invariance of the values of NN's output on the input data. We study this invariance within two data point sets $\{x_\tau\}_{\tau=1}^{i+1}$ and $\{x_\tau\}_{\tau=i+2}^n$. For $\{x_\tau\}_{\tau=1}^{i+1}$, the outputs of both neurons remain zero, so the output of the network is unchanged. For $\{x_\tau\}_{\tau=i+2}^n$, we study the output value changes of the two neurons separately. For the first neuron, we have, 
\begin{equation*}
\begin{aligned}
    \tilde{a}_1&\sigma(\tilde{w}_1 x_\tau+\tilde{b}_1)-a_1\sigma(w_1 x_\tau+b_1)\\
    &=a_1(w_1(1-\varepsilon) x_\tau+b_1+w_1 \varepsilon x_{i+1}-w_1 x_\tau-b_1)\\
    &=\varepsilon a_1( w_1\left(x_{i+1}-x_\tau\right)).
\end{aligned}
\end{equation*}

For the second neuron, we have, 
\begin{equation*}
\begin{aligned}
    \tilde{a}_2&\sigma(\tilde{w}_2 x_\tau+\tilde{b}_2)-a_2\sigma\left(w_2 x_\tau+b_2\right)\\
    &=-\frac{a_1}{a_2}{a}_2\left(\left(\tilde{w}_1-w_1\right) x_\tau+\tilde{b}_1-b_1\right)\\
    &=-\varepsilon a_1( w_1\left(x_{i+1}-x_\tau\right)).
\end{aligned}
\end{equation*}

With the same relation shown in Equution (\ref{equ:inequ}), we have, 
\begin{equation*}
    \tilde{R}_1(\vtheta)-R_1(\vtheta)<-\Theta(\varepsilon).
\end{equation*}

For $a_2<0$, we consider three cases: 
\begin{equation*}
\begin{aligned}
    \textrm{Case (1): }&a_1w_1=a_2w_2, \\
    \textrm{Case (2): }&a_1w_1> a_2w_2, \\
    \textrm{Case (3): }&a_1w_1< a_2w_2.
\end{aligned}
\end{equation*}

For Case (1), we perturb the parameters of two neurons as follows:
\begin{equation*}
\begin{array}{cl}
\tilde{w}_1=w_1, & \tilde{b}_1=b_1- \varepsilon, \\
\tilde{w}_2=w_2, & \tilde{b}_2=b_2-\frac{a_1}{a_2}\left(\tilde{b}_1-b_1\right). 
\end{array}
\end{equation*}

We first study the movement of the intercept points of the two neurons, 

\begin{equation*}
        -\frac{\tilde{b}_1}{\tilde{w}_1}-(-\frac{b_1}{w_1})=\frac{\varepsilon}{w_1}>0, 
\end{equation*}
where $0<\varepsilon<1$, $w_1>0$. In the same way, we have  
\begin{equation*}
        -\frac{\tilde{b}_2}{\tilde{w}_2}-(-\frac{b_2}{w_2})<0. 
\end{equation*}
For the invariance of the values of NN's output on the input dataset $\{x_\tau\}_{\tau=i+2}^n$, we have, 
\begin{equation*}
\begin{aligned}
    \tilde{a}_1&\sigma(\tilde{w}_1 x_\tau+\tilde{b}_1)-a_1\sigma(w_1 x_\tau+b_1)\\
    &=a_1(w_1 x_\tau+b_1-\varepsilon-w_1 x_\tau-b_1)\\
    &=-\varepsilon a_1,\\
    \tilde{a}_2&\sigma(\tilde{w}_2 x_\tau+\tilde{b}_2)-a_2\sigma(w_2 x_\tau+b_2)\\
    &=-\frac{a_1}{a_2}{a}_2\left(\left(\tilde{w}_1-w_1\right) x_\tau+\tilde{b}_1-b_1\right)\\
    &=\varepsilon a_1.
\end{aligned}
\end{equation*}
Thus, we can easily get, for each input point $x_\tau$ in $\{x_\tau\}_{\tau=i+2}^n$, 
\begin{equation*}
\begin{aligned}
    \tilde{a}_1&\sigma(\tilde{w}_1 x_\tau+\tilde{b}_1)+\tilde{a}_2\sigma(\tilde{w}_2 x_\tau+\tilde{b}_2)\\
    &-a_1\sigma(w_1 x_\tau+b_1)-a_2\sigma(w_2 x_\tau+b_2)=0, \\
    |\tilde{a}_1&\sigma(\tilde{w}_1 x_\tau+\tilde{b}_1)-\tilde{a}_2\sigma(\tilde{w}_2 x_\tau+\tilde{b}_2)|\\
    &<|a_1\sigma(w_1 x_\tau+b_1)-a_2\sigma(w_2 x_\tau+b_2)|. 
\end{aligned}
\end{equation*}
Then, we have, 
\begin{equation*}
\begin{aligned}
    \tilde{R}_1(\vtheta)&-R_1(\vtheta)\\
    =&\sum_{\tau=i+2}^n \Big(\tilde{a}_1^2 \sigma(\tilde{w}_1 x_\tau + \tilde{b}_1)^2+\tilde{a}_2^2 \sigma(\tilde{w}_2 x_\tau + \tilde{b}_2)^2\\
    &- a_1^2 \sigma(w_1 x_\tau + b_1)^2-a_2^2 \sigma(w_2 x_\tau + b_2)^2\Big)\\
    <&-\Theta(\varepsilon).
\end{aligned}
\end{equation*}

For the remaining three cases, we show the direction of the perturbation respectively, and we can similarly verify that the loss value remains unchanged and the $R_1(\vtheta)$ term decreases by $\Theta(\varepsilon)$.

Case (2): The perturbation is shown as follows:
\begin{equation*}
\begin{array}{cc}
\tilde{w}_1=w_1(1+\varepsilon), & \tilde{b}_1=b_1-\frac{a_2b_2-a_1b_1}{a_1w_1-a_2w_2} w_1 \varepsilon, \\
\tilde{w}_2=w_2-\frac{a_1}{a_2}\left(\tilde{w}_1-w_1\right), & \tilde{b}_2=b_2-\frac{a_1}{a_2}\left(\tilde{b}_1-b_1\right).
\end{array}
\end{equation*}

Case (3): The perturbation is shown as follows:
\begin{equation*}
\begin{array}{cl}
\tilde{w}_1=w_1(1-\varepsilon), & \tilde{b}_1=b_1+x_{i+1} w_1 \varepsilon, \\
\tilde{w}_2=w_2-\frac{a_1}{a_2}\left(\tilde{w}_1-w_1\right), & \tilde{b}_2=b_2-\frac{a_1}{a_2}\left(\tilde{b}_1-b_1\right).
\end{array}
\end{equation*}

\end{proof}

\begin{theorem*}[\textbf{the effect of $R_1(\vtheta)$ on facilitating flatness}] Under the Setting \ref{set:1}--\ref{set:3}, consider a two-layer ReLU NN, 
\begin{equation*}
	f_{\vtheta}(x)=\sum_{j=1}^{m} a_j \sigma(\vw_j \vx), 
\end{equation*}
trained with dataset $S=\{(\vx_{i}, y_{i})\}_{i=1}^n$. Under the gradient flow training with the loss function $R_S(\vtheta)+R_1(\vtheta)$, if $\vtheta_0$ satisfying $R_S(\vtheta_0)=0$ and $\nabla_{\vtheta} R_1(\vtheta_0)\neq 0$, we have 
\begin{equation*}
    \frac{{\rm d}\left(\frac{1}{n}\sum_{i=1}^n\norm{\nabla_{\vtheta}f_{\vtheta_0}(\vx_i)}_2^2\right)}{\rm dt}<0.
\end{equation*}
\end{theorem*}
\begin{proof}
Recall the quantity that characterizes the flatness of model $f_{\vtheta}$,
\begin{equation*}
\begin{aligned}
\frac{1}{n}\sum_{i=1}^n&\norm{\nabla_{\vtheta}f_{\vtheta}(\vx_i)}_2^2\\
&=\frac{1}{n}\sum_{i=1}^n\sum_{j=1}^m\left(\sigma^2(\vw_j^{\T}\vx_i)+a_j^2\sigma^{'}(\vw_j^{\T}\vx_i)^2\Norm{\vx_i}_2^2\right).
\end{aligned}
\end{equation*}

Then, under the assumption of $R_{S}(\vtheta)=0$, we obtain that $\nabla_{\vtheta}\left(R_{S}(\vtheta)+R_1(\vtheta)\right) =\nabla_{\vtheta}\left(R_1(\vtheta)\right)$,
we notice that for  sufficiently small  amount of time $t$, the sign of  $\vw_j^{\T}\vx_i$ remains the same. Hence given data $\vx_i, i\in[n]$, and parameter set $\vtheta_j:=(a_k, \vw_j), j \in [m]$,   if $\vw_j^{\T}\vx_i<0$, for some indices $i, j$,
 then $ \norm{\nabla_{\vtheta_j}f_{\vtheta}(\vx_i)}_2^2\equiv0$. The analysis above reveals that for certain neuron whose index is $j$, we shall focus the data $\vx_i$ satisfying 
$\vw_j^{\T}\vx_i>0$, i.e., the following set constitutes our data of interest for the $j$-th neuron
\[
D_j:=\{(\vx_k, y_k)\}_{i=1}^{n_j}:=\{(\vx_i, y_i) \mid \vw_j^{\T}\vx_i>0\} \subseteq S,
\] 
as  $\vtheta_j$ is trained under the gradient flow of $R_1(\vtheta)$, i.e.
\begin{align*}
\frac{\D a_j}{\D t}&=-\frac{1-p}{2np}\sum_{i=1}^n a_j\sigma^2(\vw_j^{\T} \vx_i),\\
\frac{\D \vw_j}{\D t}&=-\frac{1-p}{2np}\sum_{i=1}^n a_j^2\sigma(\vw_j^{\T}\vx_i)\vx_i^{\T},  
\end{align*}
then the gradient flow of the flatness reads
 
\begin{align*}
     \frac{\D}{\D t}&\sum_{k=1}^{n_j}\norm{\nabla_{\vtheta_j}f_{{\vtheta}}(\vx_k)}_2^2 \\
     & = \sum_{k=1}^{n_j}\left(\left<2(\vw_j^\T\vx_k)\vx_k, \frac{\D \vw_j}{\D t} \right>+ 2a_j\Norm{\vx_i}_2^2\frac{\D a_j}{\D t}  \right)\\
&=-\frac{1-p}{np}\sum_{i=1}^n\sum_{k=1}^{n_j}a_j^2(\vw_j^\T\vx_k)(\vw_j^\T\vx_i)\left<\vx_k, \vx_i\right>\mathbf{1}_{\vw_j^\T\vx_i>0}\\
&~~-\frac{1-p}{np}\sum_{i=1}^n\sum_{k=1}^{n_j}a^2_j\Norm{\vx_i}_2^2(\vw_j^\T\vx_i)^2,
\end{align*}
by definition of $D_j$,  we obtain that the first term can be written as 
\begin{equation*}
\begin{aligned}
  & \frac{1-p}{np}\sum_{i=1}^n\sum_{k=1}^{n_j}a_j^2(\vw_j^\T\vx_k)(\vw_j^\T\vx_i)\left<\vx_k, \vx_i\right>\mathbf{1}_{\vw_j^\T\vx_i>0}\\
  &= \frac{1-p}{np}\sum_{i=1}^{n_j}\sum_{k=1}^{n_j}a_j^2(\vw_j^\T\vx_k)(\vw_j^\T\vx_i)\left<\vx_k, \vx_i\right> \\
  &= \frac{1-p}{np}\Norm{\sum_{k=1}^{n_j}a_j(\vw_j^\T\vx_k)\vx_k}_2^2\geq 0,
\end{aligned}
\end{equation*}
hence 
\begin{align*}
\frac{\D}{\D t}\sum_{k=1}^{n_j}\norm{\nabla_{\vtheta_j}f_{{\vtheta}}(\vx_k)}_2^2 & <0,    
\end{align*}
which finishes the proof.

\end{proof}

\newpage
\section{Additional Experimental Results}\label{app:exp}

\subsection{Verification of $\texorpdfstring{R_2(\vtheta)}{E=energy, m=mass, c=speed\ of\ light}$ on Complete Datasets}\label{app:R2_sgd}

 As shown in Fig. \ref{fig:R22}, the learning rate $\varepsilon$ and the regularization coefficient $\lambda$ are similar when they reach the maximum test accuracy (red point).

\begin{figure}[h]
	\centering
	\subfloat[]{\includegraphics[width=0.4\textwidth]{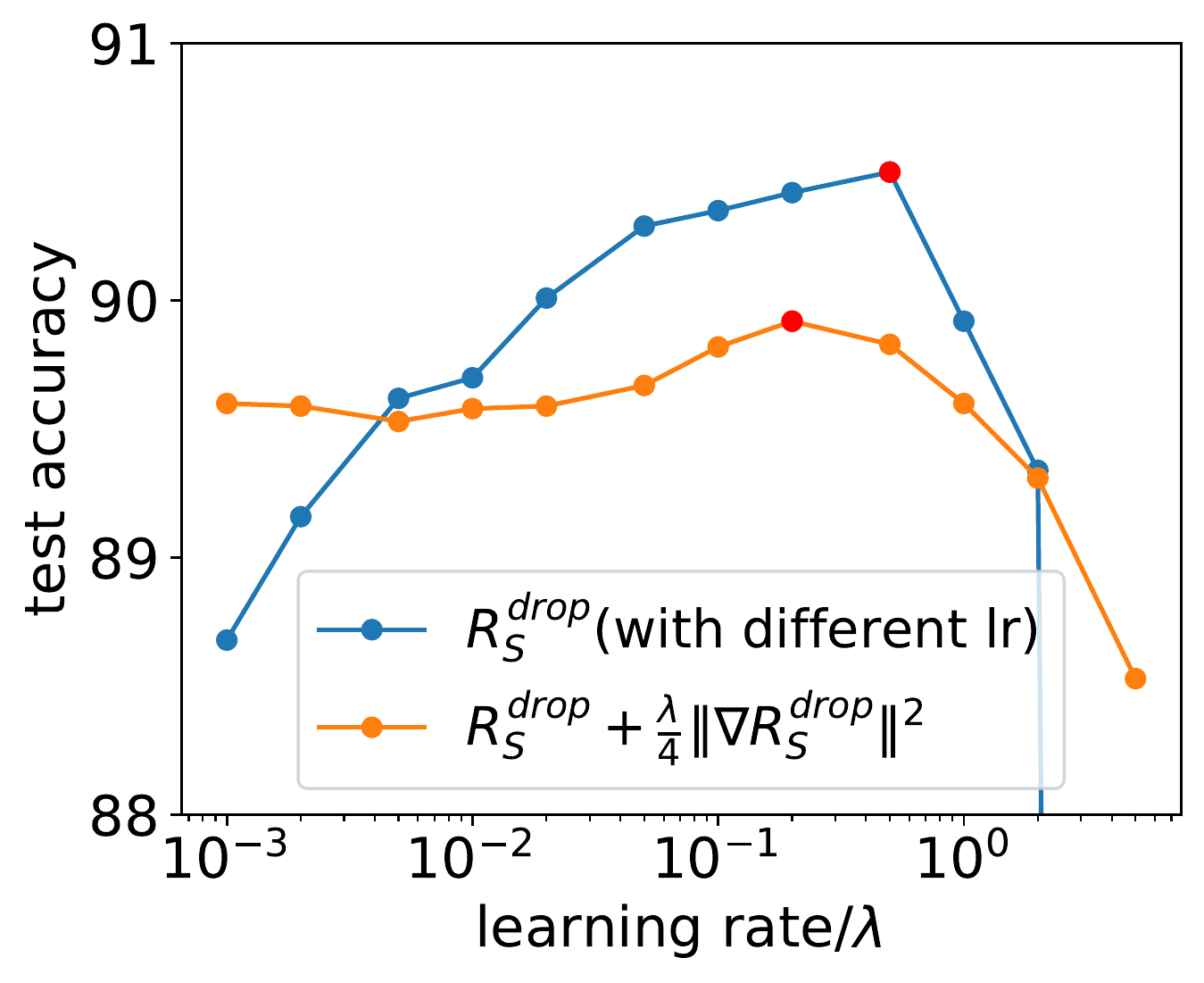}}
	\subfloat[]{\includegraphics[width=0.4\textwidth]{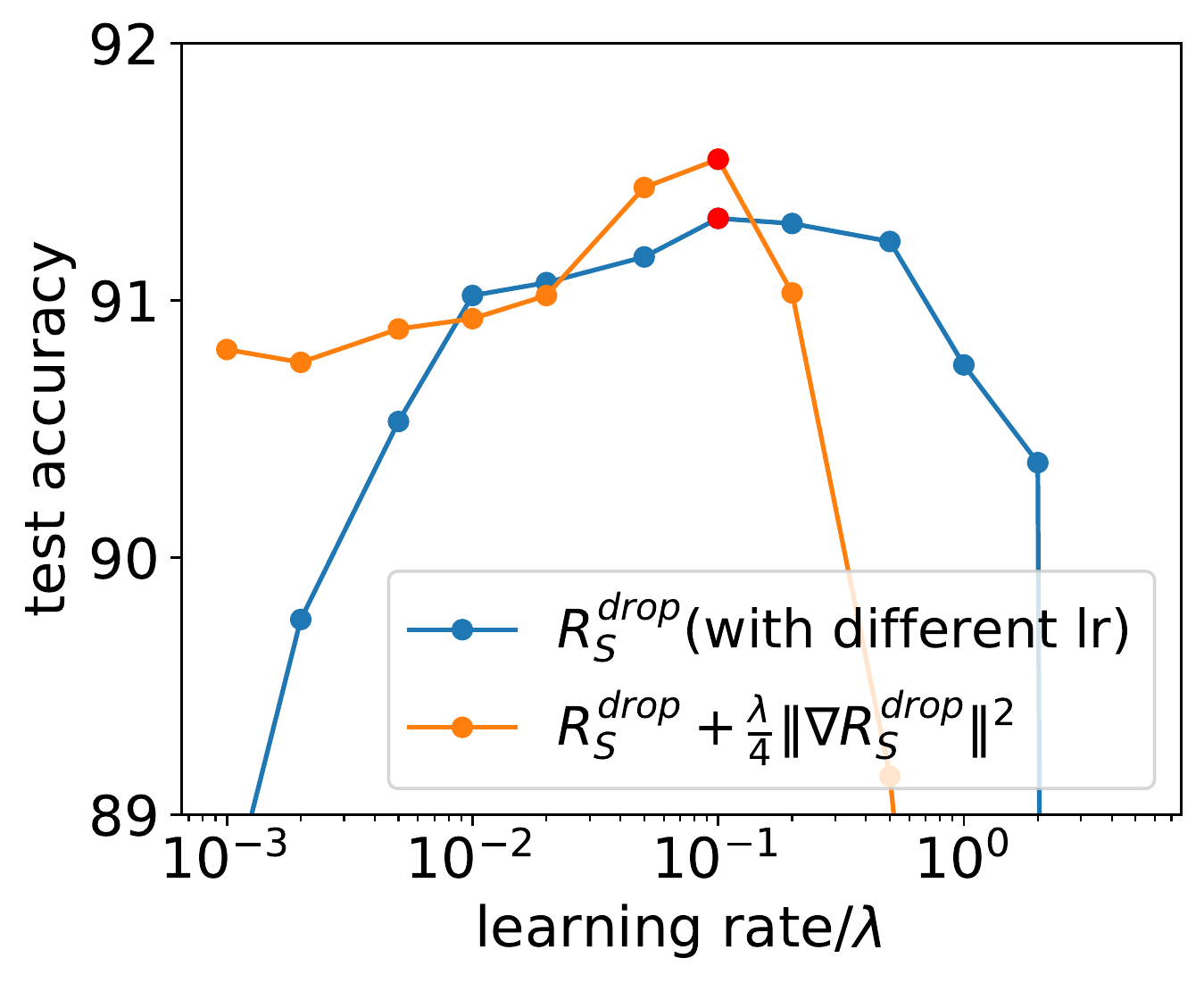}}
    \caption{
    For different training tasks, the test accuracy is obtained by training the network with SGD under different learning rates and regularization coefficients. The red dots indicate the location of the maximum test accuracy of the NNs obtained by training with both two loss functions. For loss function $\RS ^\mathrm{drop}\left(\vtheta, \veta\right)$, we train the NNs with different learning rates $\varepsilon$. For loss function $\RS ^\mathrm{drop}\left(\vtheta, \veta\right)+(\lambda/4)\Vert \nabla_{\vtheta} \RS ^\mathrm{drop}\left(\vtheta, \veta\right) \Vert ^2$, we train the NNs with different regularization coefficient $\lambda$, with a small and unchanged learning rate ($\varepsilon=5\times 10^{-3}$). (a) Classify the Fashion-MNIST dataset by training four-layer NNs of width 4096 with a training batch size of 64. (b) Classify the CIFAR-10 datasets by training VGG-9 with a training batch size of 32. 
    \label{fig:R22}}
\end{figure} 

\subsection{Dropout Facilitates Condensation on ReLU NNs} \label{app:relu_further}
 In this subsection, We verify that the ReLU network facilitates condensation under dropout shown in Fig. \ref{pic:condense_relu}, similar to the situation in the tanh NNs shown in Fig. \ref{pic:condense_tanh} in the main text. We only plot the neurons with non-zero output value in the data interval $[x_1, x_n]$ in the situation in the ReLU NNs. For the neurons with constant zero output value in the data interval, they will not affect the training process and the NN's output. Further, we study the results of ReLU NNs training under SGD, and the relationship between the model rank and generalization error, as shown in Figs. \ref{pic:sgd_condense_relu}, \ref{pic:condense_heat_relu}, which correspond to the tanh NNs shown in Figs. \ref{pic:sgd_condense_tanh}, \ref{pic:condense_heat_tanh} in the main text.

\begin{figure}[h]
	\centering
	\subfloat[$p=1$, output]{\includegraphics[width=0.24\textwidth]{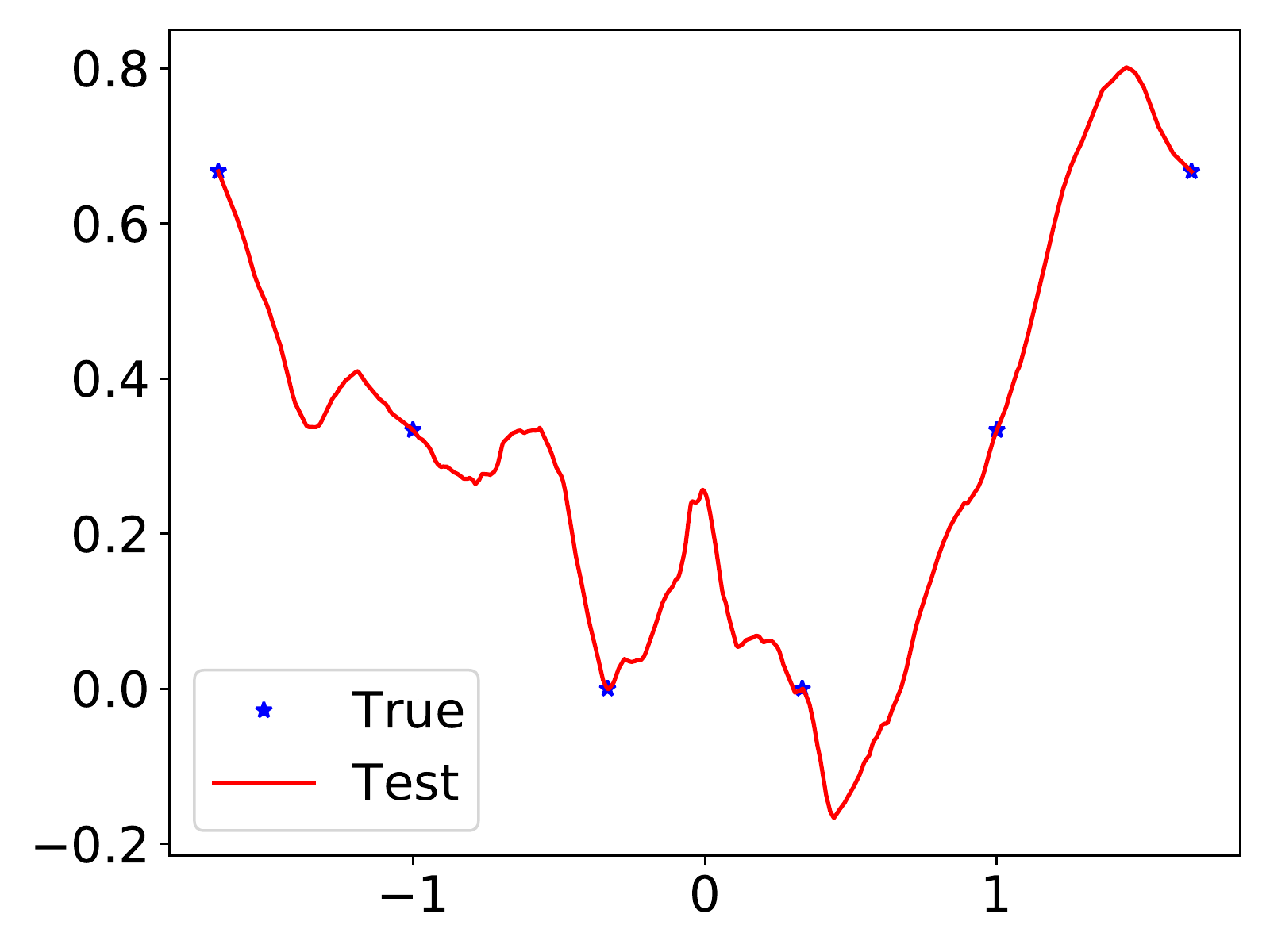}}
	\subfloat[$p=0.9$, output]{\includegraphics[width=0.24\textwidth]{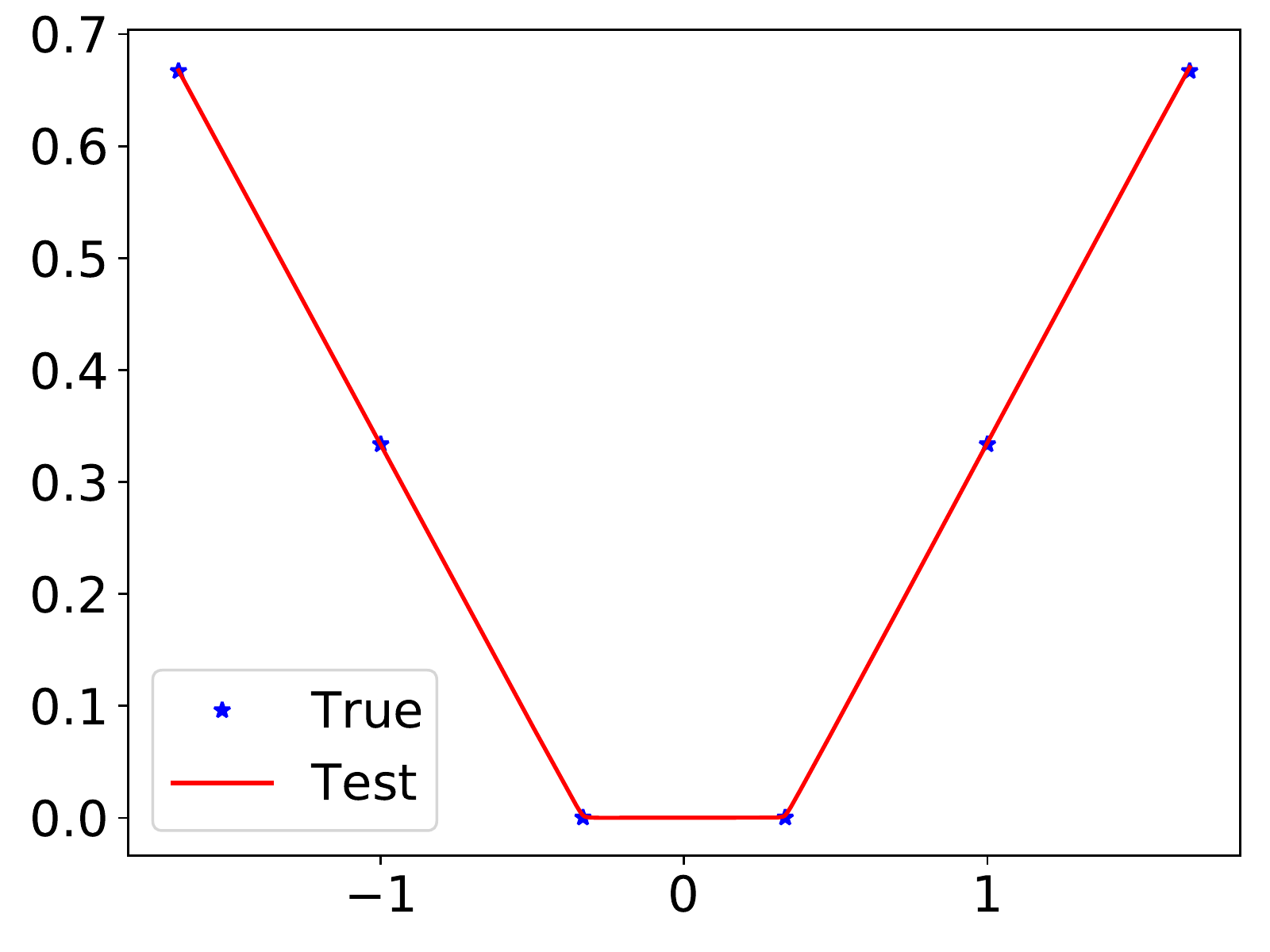}}
	\subfloat[$p=1$, feature]{\includegraphics[width=0.24\textwidth]{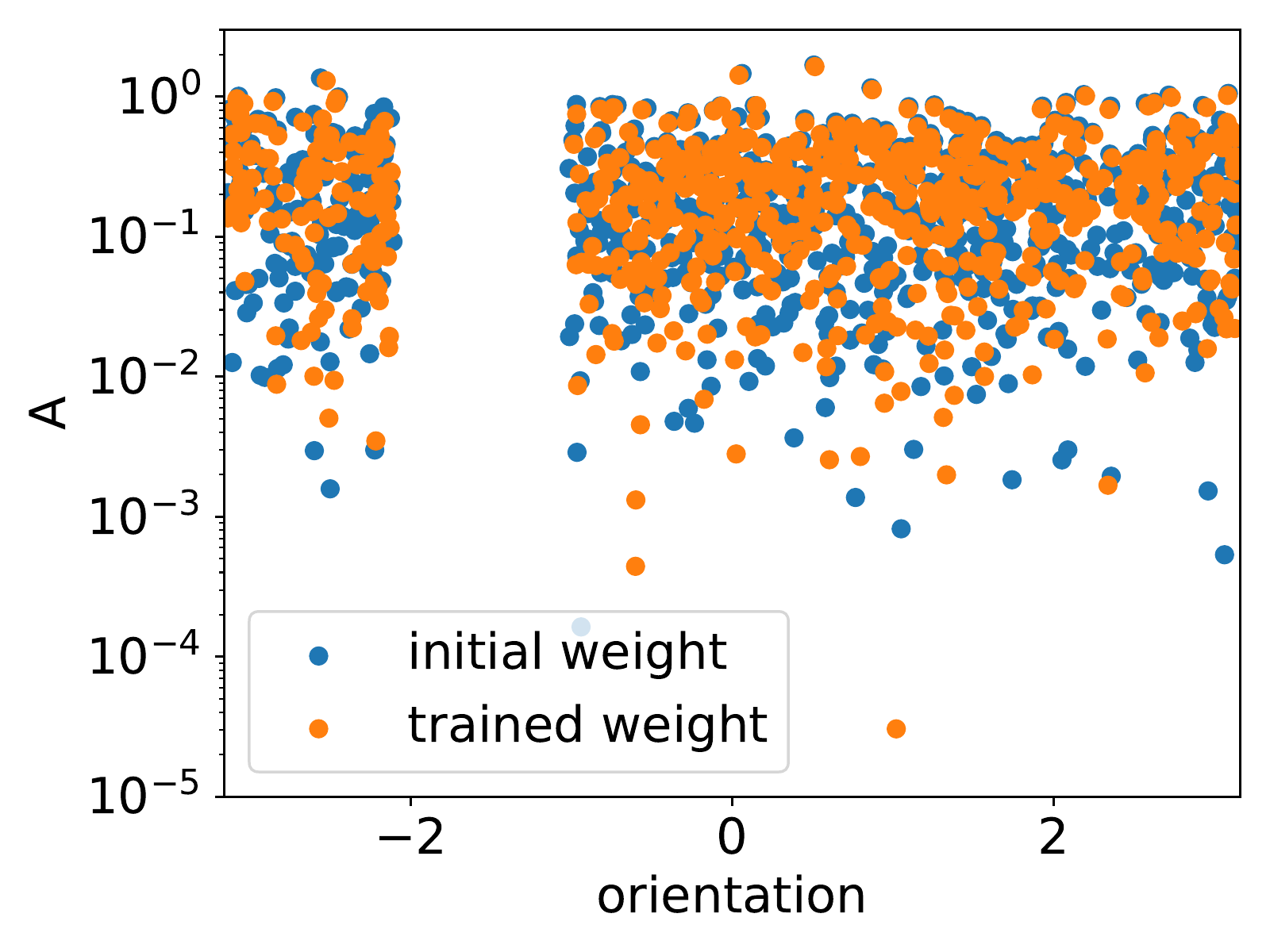}}
	\subfloat[$p=0.9$, feature]{\includegraphics[width=0.24\textwidth]{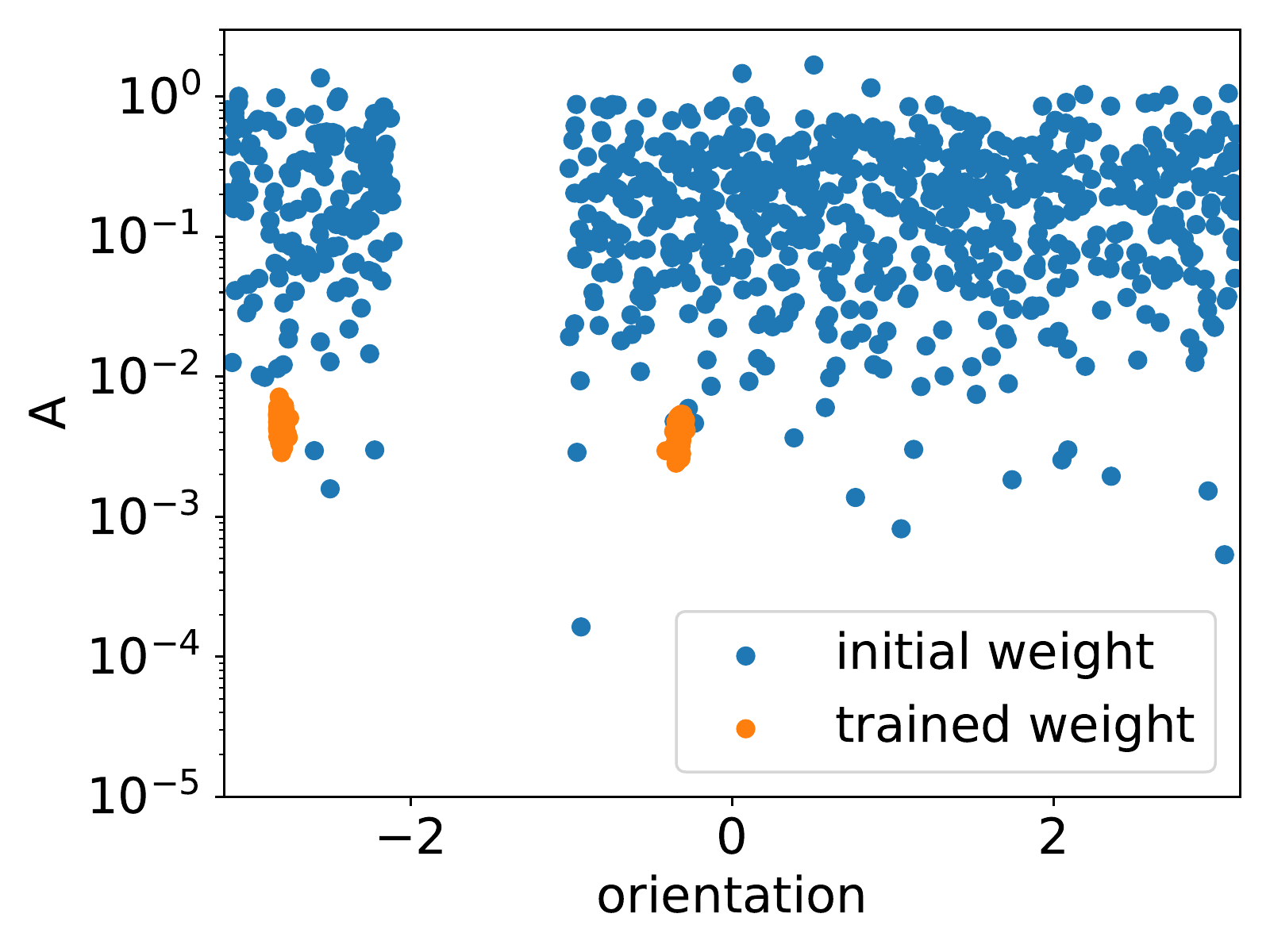}}\\
	\subfloat[$p=1$, output]{\includegraphics[width=0.24\textwidth]{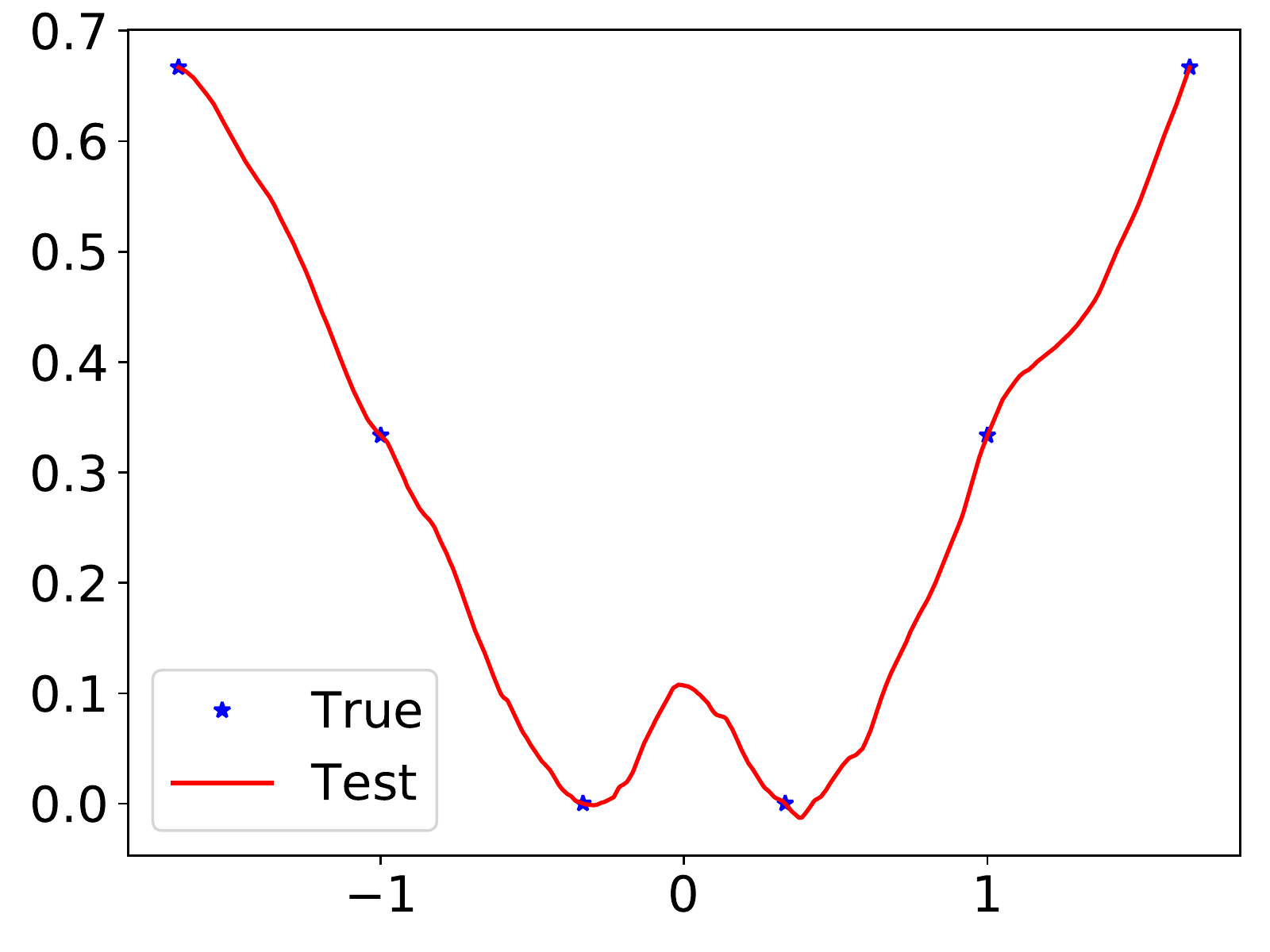}}
	\subfloat[$p=0.9$, output]{\includegraphics[width=0.24\textwidth]{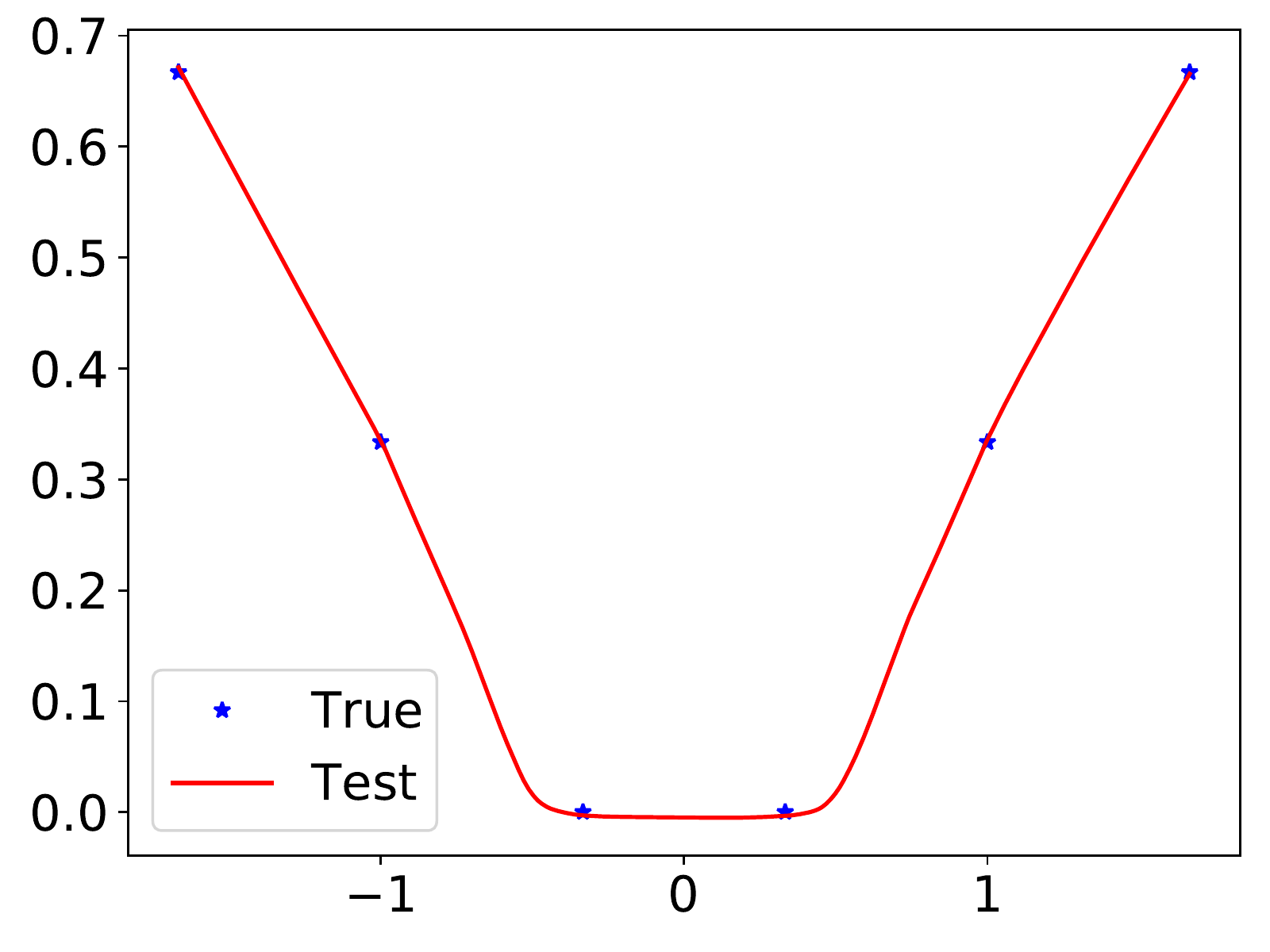}}
	\subfloat[$p=1$, feature]{\includegraphics[width=0.24\textwidth]{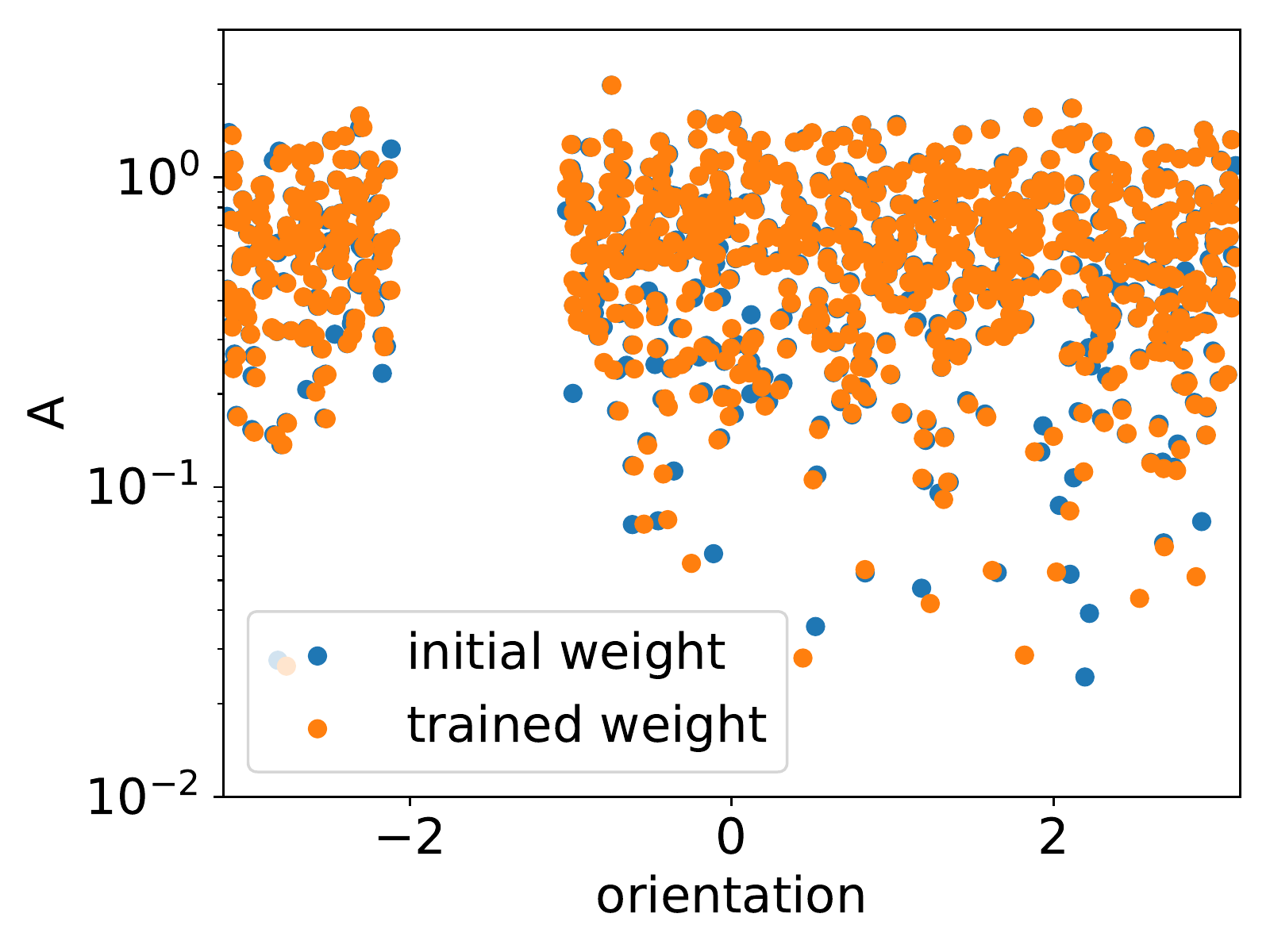}}
	\subfloat[$p=0.9$, feature]{\includegraphics[width=0.24\textwidth]{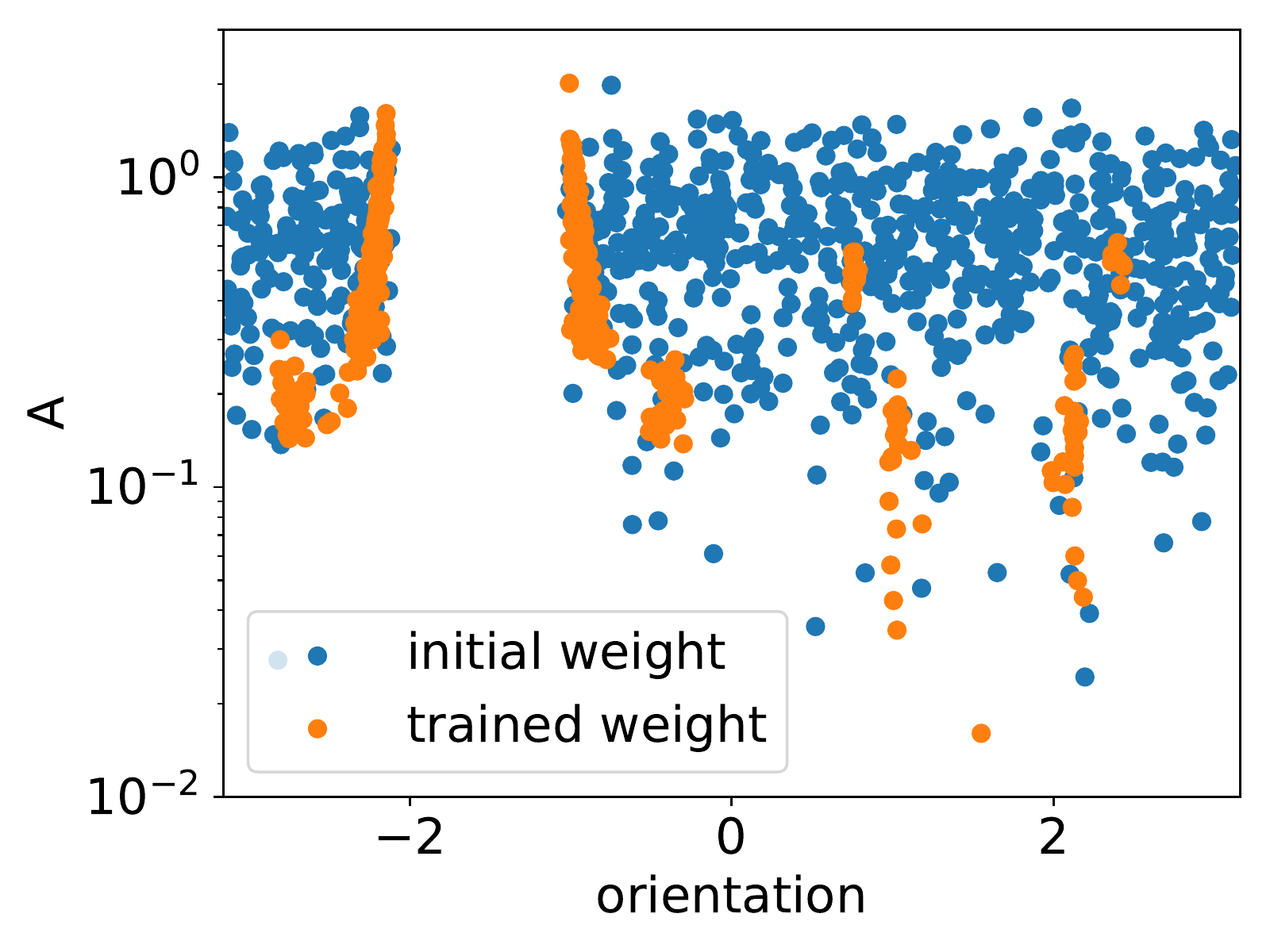}}

  \caption{ReLU NNs outputs and features under different dropout rates. The width of the hidden layers is $1000$, and the learning rate for different experiments is $1\times10^{-3}$. In (c,d,g,h), blue dots and orange dots are for the weight feature distribution at the initial and final training stages, respectively. The top row is the result of two-layer networks, with the dropout layer after the hidden layer. The bottom row is the result of three-layer networks, with the dropout layer between the two hidden layers and after the last hidden layer. \label{pic:condense_relu}}
\end{figure} 

  \begin{figure}[h]
	\centering
	\subfloat[batch size $=2$, output]{\includegraphics[width=0.4\textwidth]{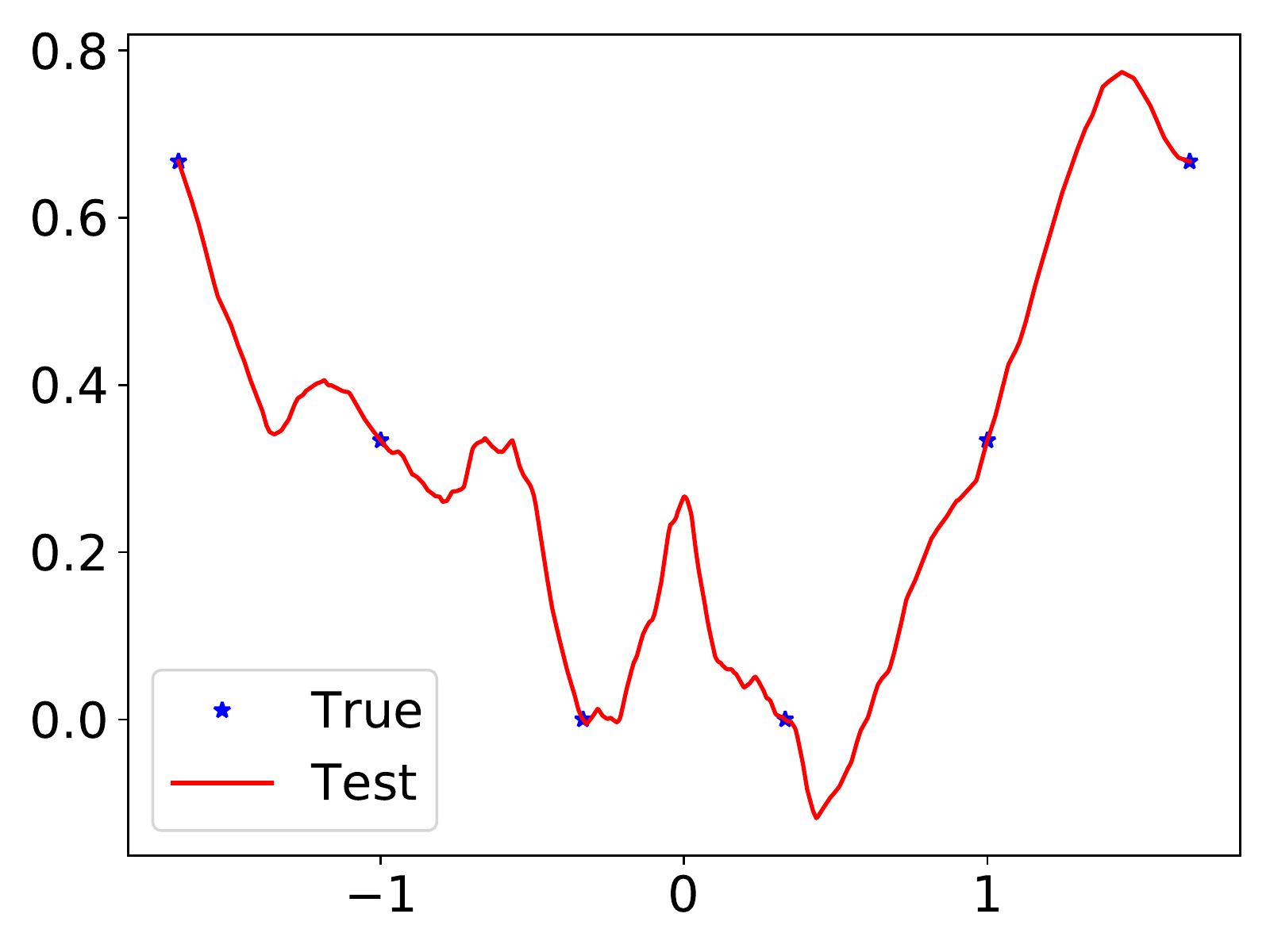}}
	\subfloat[batch size $=2$, feature]{\includegraphics[width=0.4\textwidth]{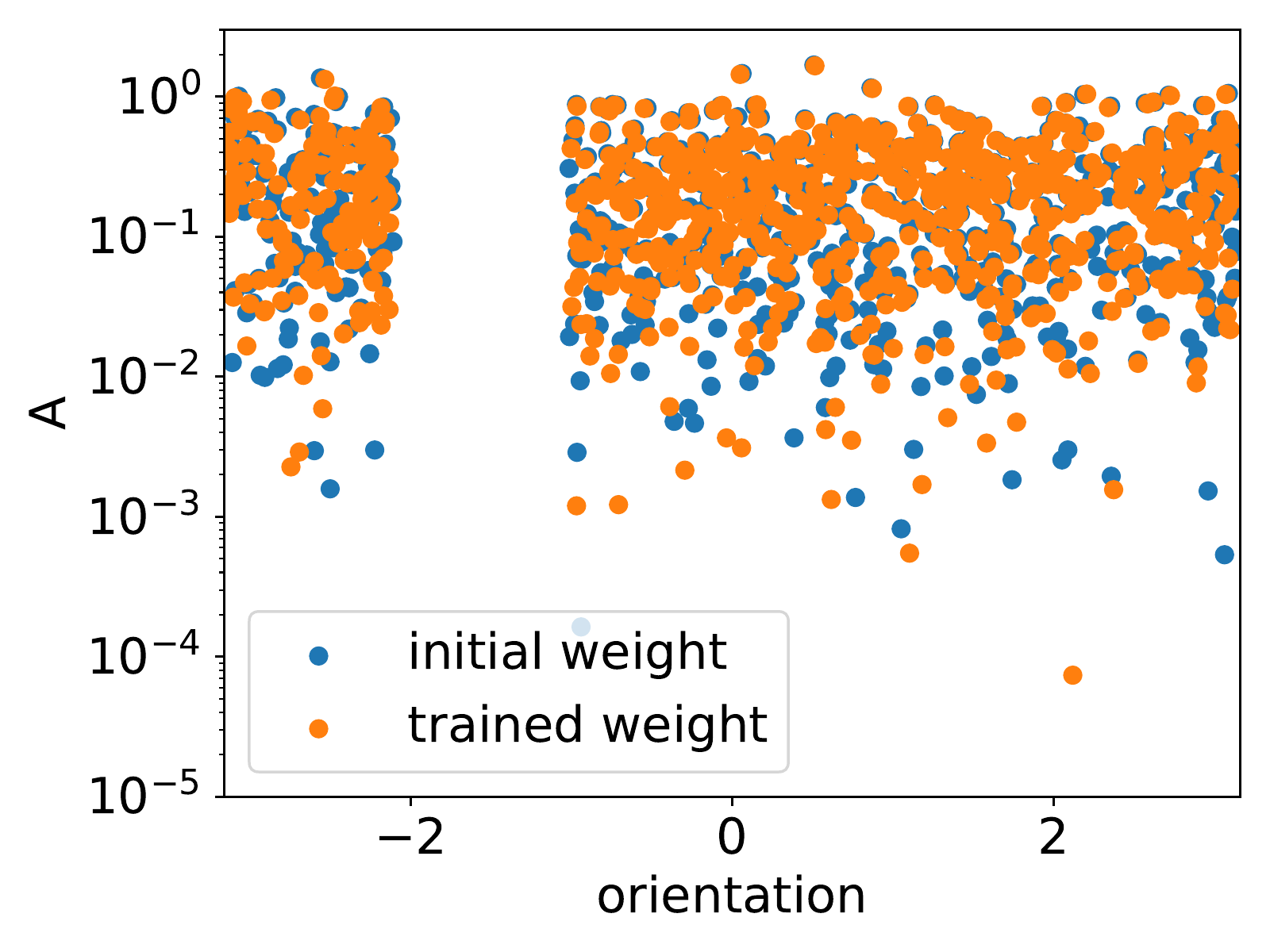}}

  \caption{Two-layer ReLU NN output and feature under a batch size of 2. The width of the hidden layer is $1000$, and the learning rate is $1\times10^{-3}$. In (b), blue dots and orange dots are for the weight feature distribution at the initial and final training stages, respectively. \label{pic:sgd_condense_relu}}
\end{figure}

\begin{figure}[h]
	\centering
	\includegraphics[width=0.6\textwidth]{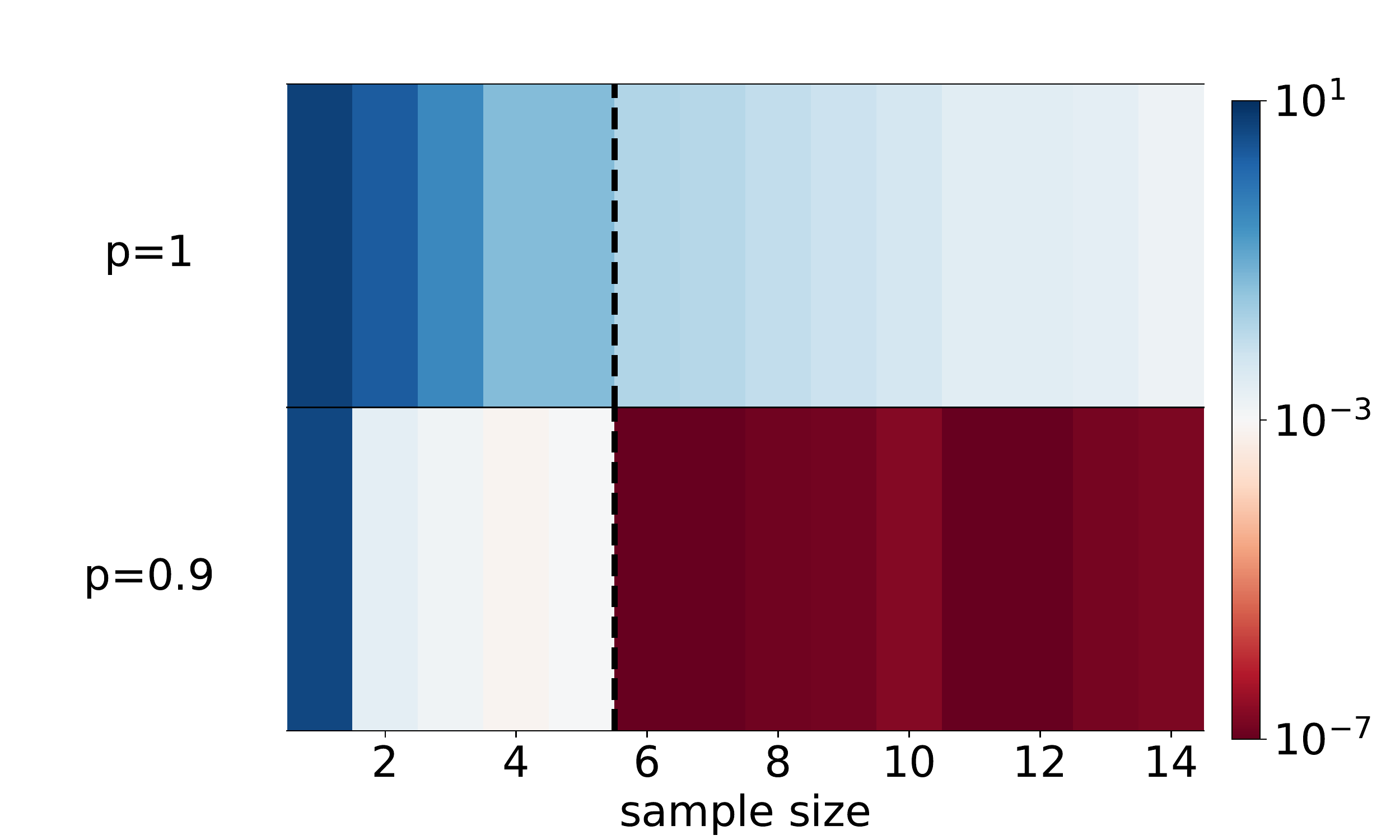}

  \caption{Average test error of the two-layer ReLU NNs (color) vs. the number of samples (abscissa) for different dropout rates (ordinate). For all experiments, the width of the hidden layer is $1000$, and the learning rate is $1\times10^{-4}$ with the Adam optimizer. Each test error is averaged over $10$ trials with random initialization. \label{pic:condense_heat_relu}}
\end{figure} 

\subsection{Detailed Features of Tanh NNs} \label{app:tanh_further}

In order to eliminate the influence of the inhomogeneity of the tanh activation function on the parameter features of Fig. \ref{pic:condense_tanh}, we drew the normalized scatter diagrams between $\norm{a_j}$, $\norm{\vw_j}$ and the orientation, as shown in Fig. \ref{fig:condense_tanh_further}. Obviously, for the network with dropout, both the input weight and the output weight have weight condensation, while the network without dropout does not have weight condensation.

\begin{figure}[!t]
	\centering
	\subfloat[two-layer NN, initialization]{\includegraphics[width=0.3\textwidth]{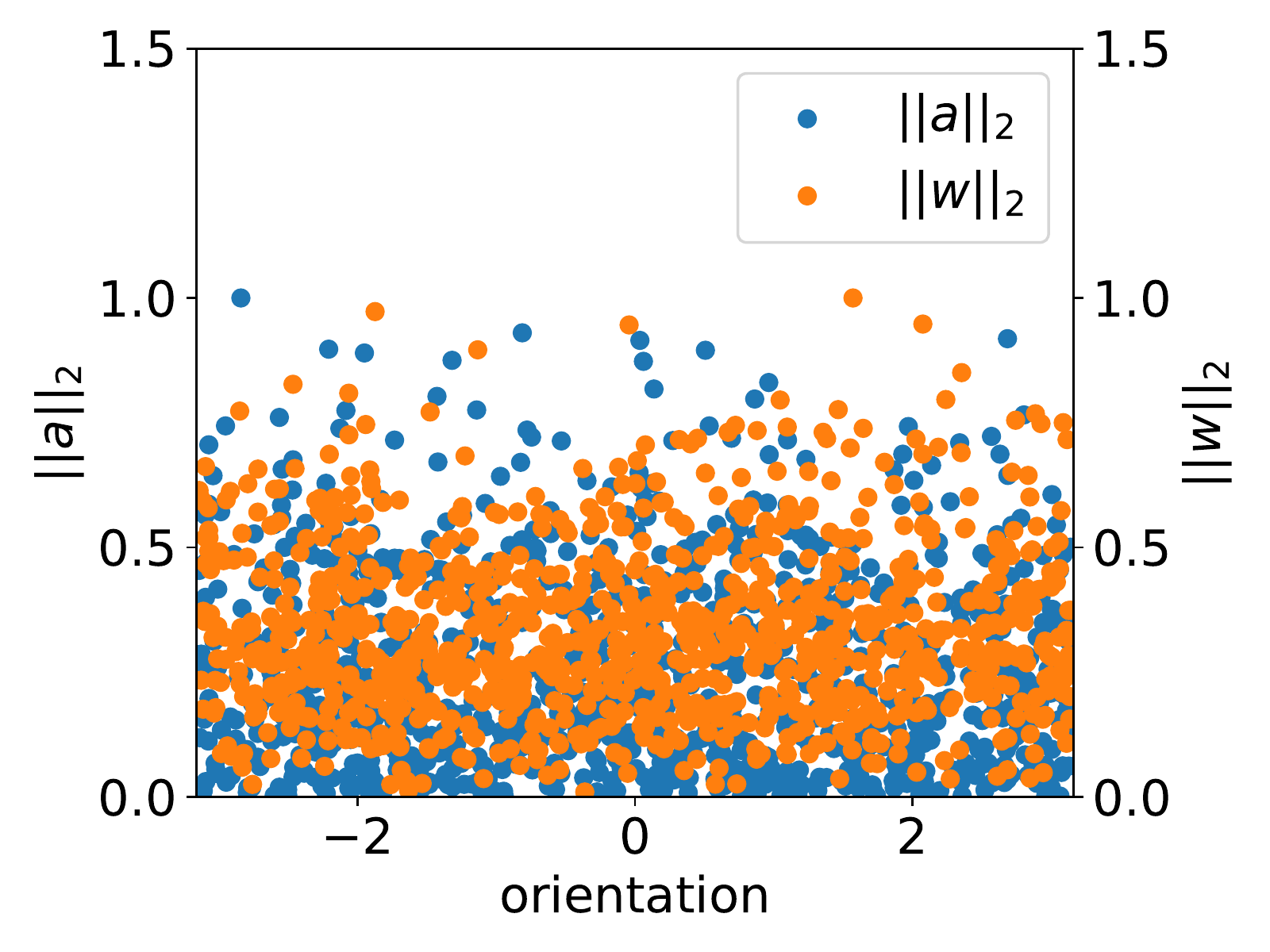}}
	\subfloat[three-layer NN, initialization]{\includegraphics[width=0.3\textwidth]{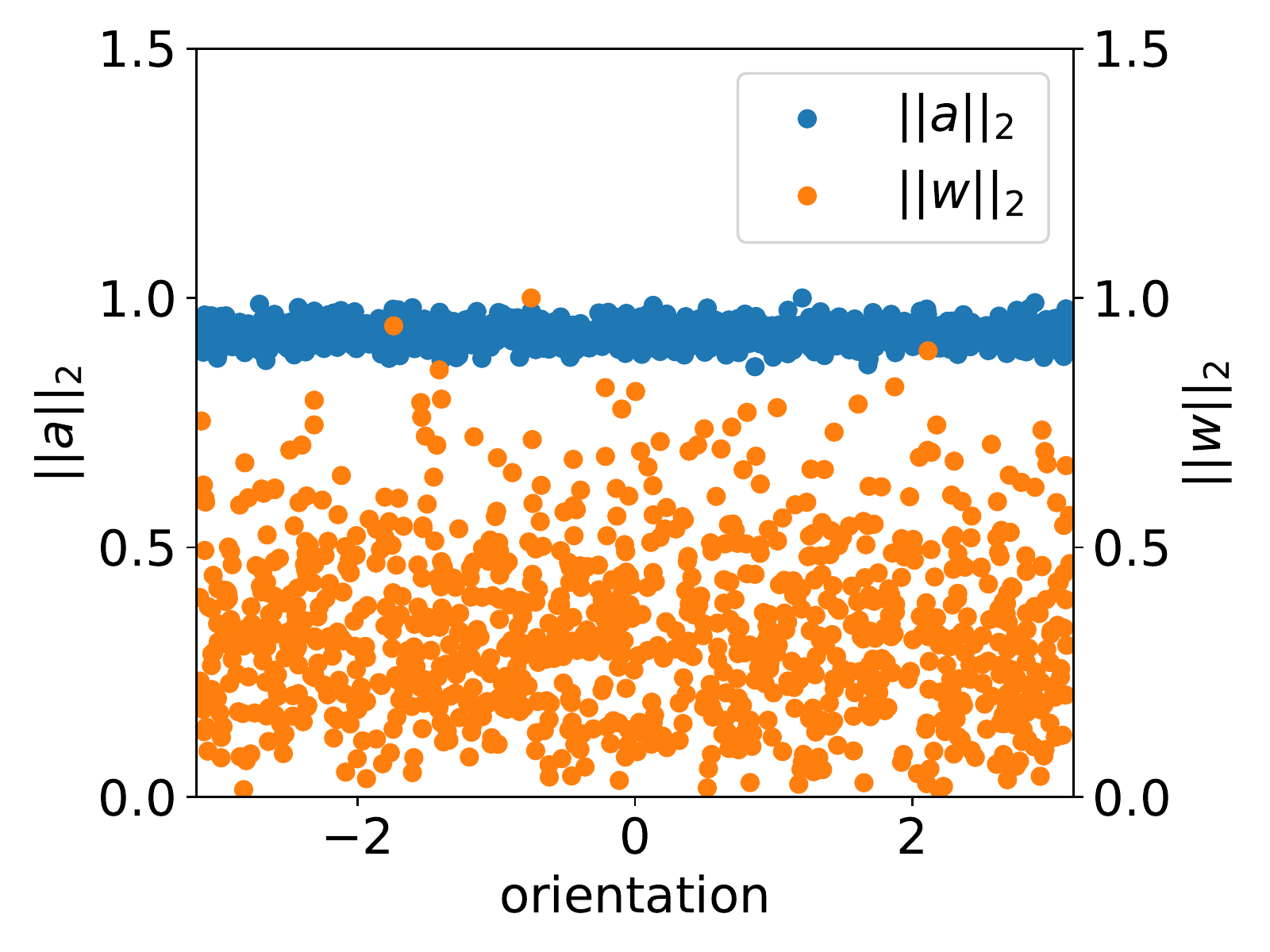}}
 \subfloat[two-layer NN, $p=0.9$]{\includegraphics[width=0.3\textwidth]{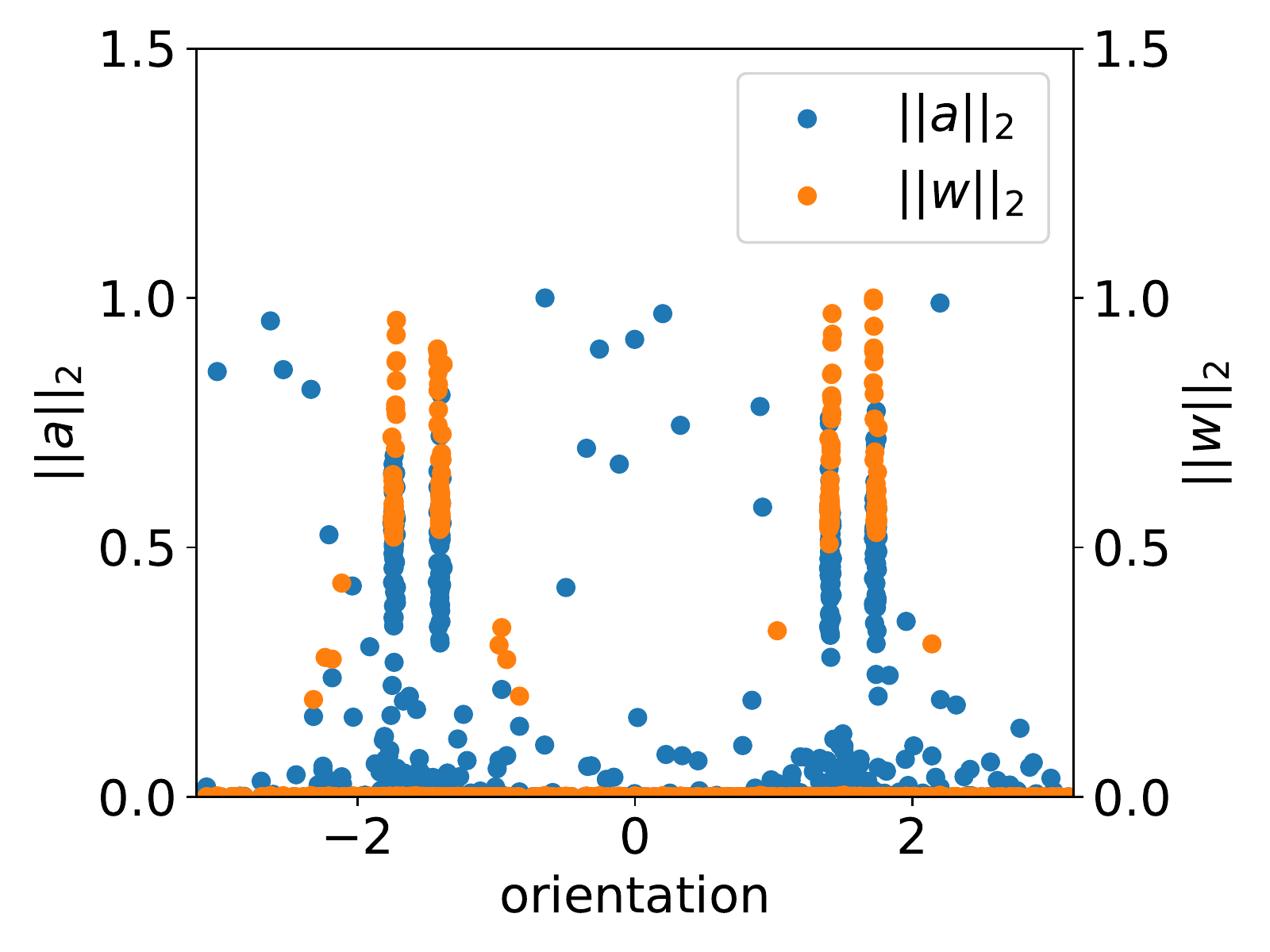}} \\
  \subfloat[three-layer NN, $p=0.9$]{\includegraphics[width=0.3\textwidth]{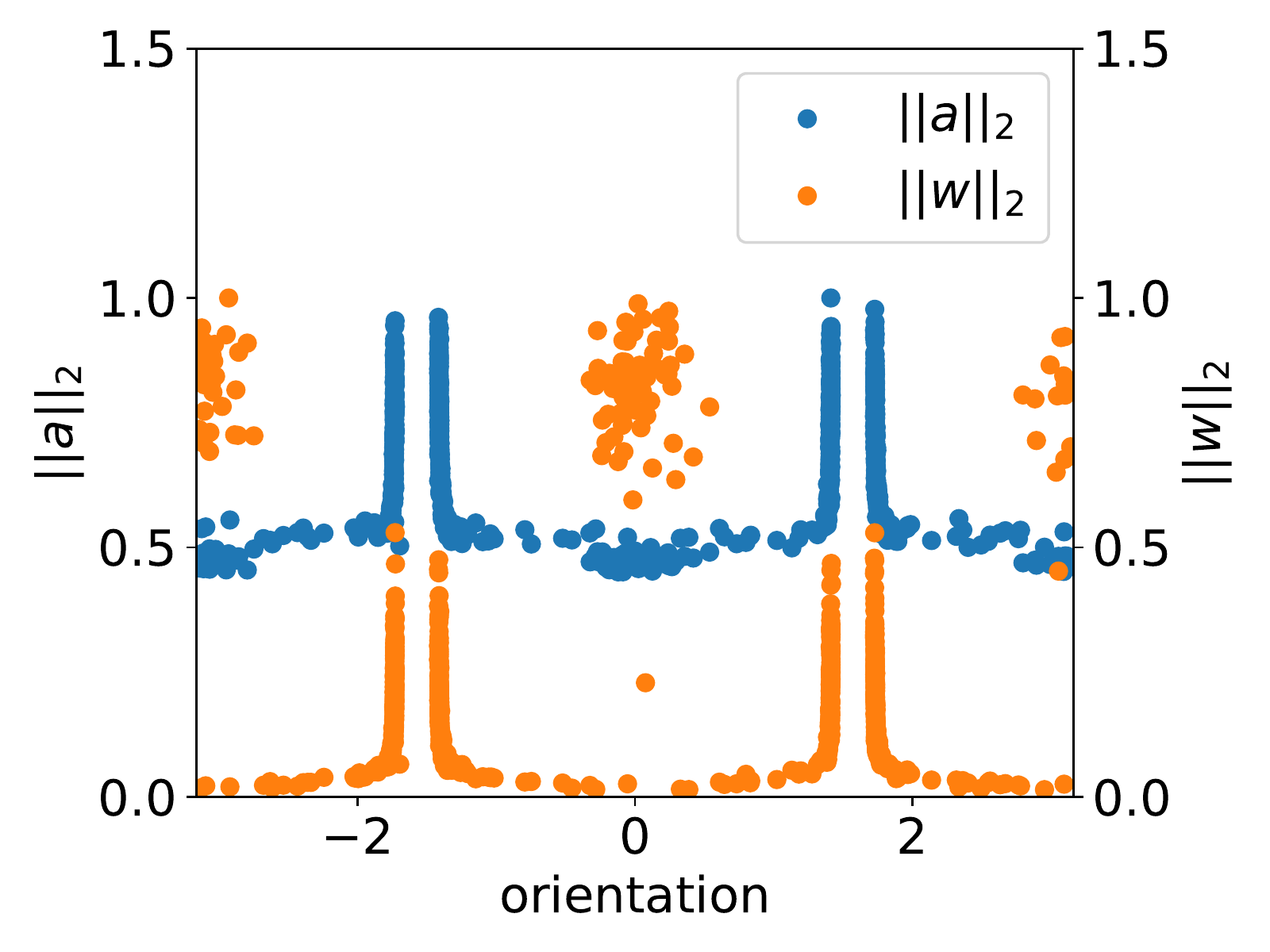}}
 \subfloat[two-layer NN, $p=1$]{\includegraphics[width=0.3\textwidth]{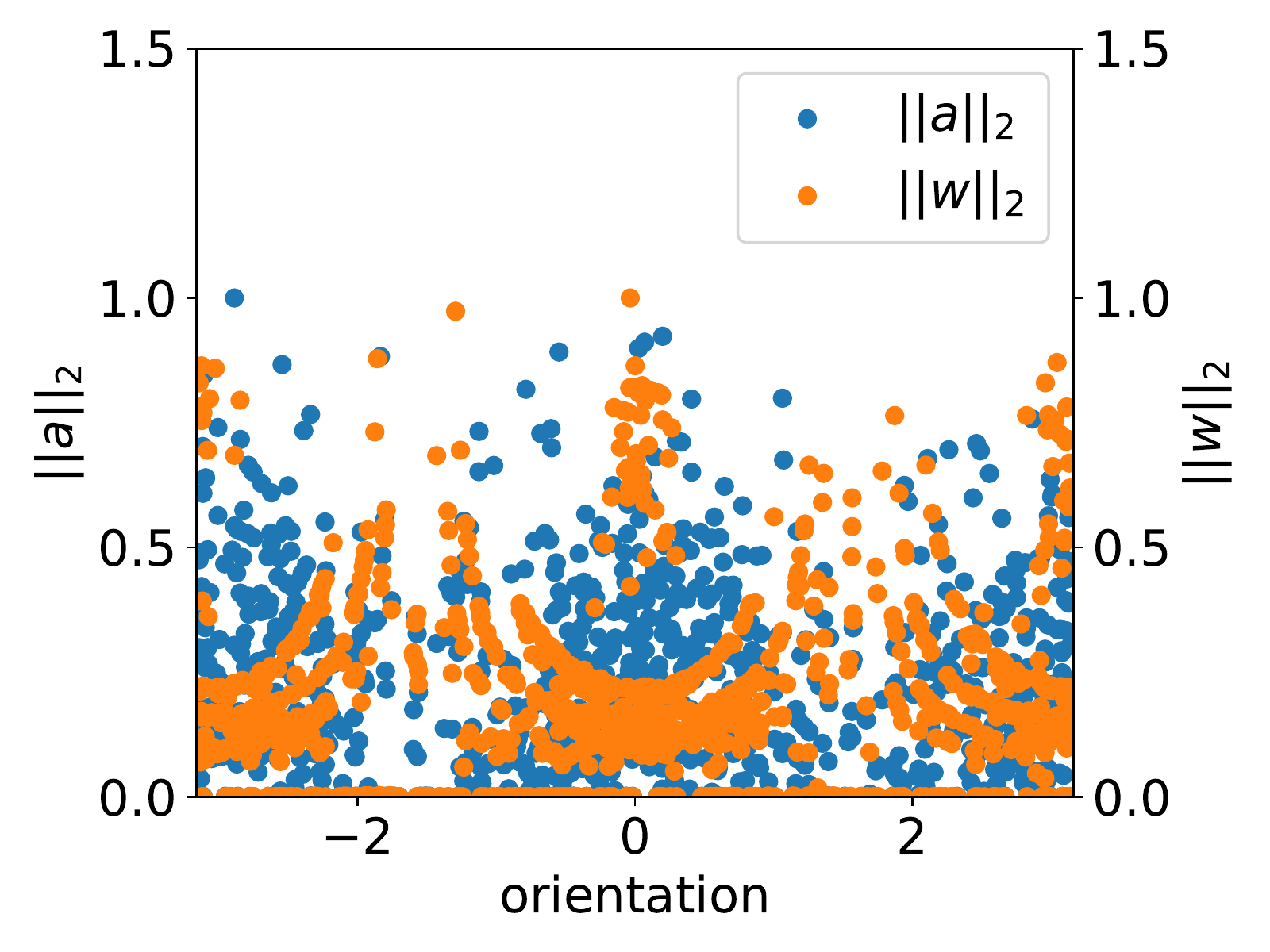}} 
  \subfloat[three-layer NN, $p=1$]{\includegraphics[width=0.3\textwidth]{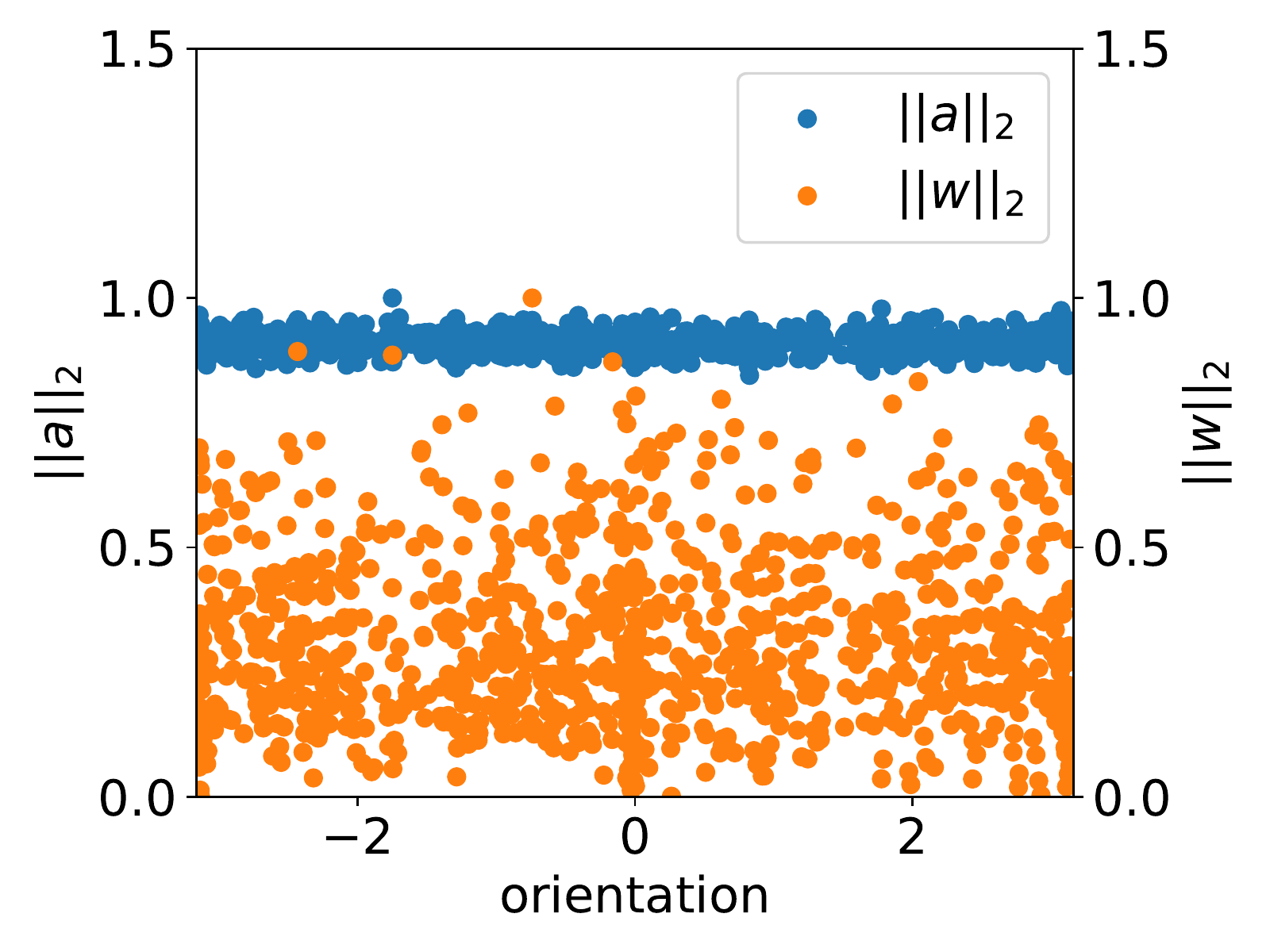}}

  \caption{The normalized scatter diagrams between $\norm{a_j}$, $\norm{\vw_j}$ and the orientation of tanh NNs for the initialization parameters and the parameters trained with and without dropout. Blue dots and orange dots are the output weight distribution and the input weight distribution, respectively. } \label{fig:condense_tanh_further}
\end{figure}

\end{document}